\documentclass[a4paper,11pt]{llncs}

\usepackage{amsmath}
\usepackage{amssymb}  
\usepackage{latexsym} 
\usepackage{multirow}

\usepackage{tikz}
\usepackage{pgfplots}
\usetikzlibrary{trees}
\usetikzlibrary{topaths}
\usetikzlibrary{decorations.pathmorphing}
\usetikzlibrary{decorations.markings}
\usetikzlibrary{matrix,backgrounds,folding}
\usetikzlibrary{chains,scopes,positioning,fit}
\usetikzlibrary{arrows,shadows}
\usetikzlibrary{calc} 
\usetikzlibrary{chains}
\usetikzlibrary{shapes,shapes.geometric,shapes.misc}

\pgfdeclarelayer{edgelayer}
\pgfdeclarelayer{nodelayer}
\pgfsetlayers{background,edgelayer,nodelayer,main}

\tikzstyle{dot}=[circle,fill=black,draw=black]
\tikzstyle{every picture}=[baseline=(current bounding box).east,scale=0.5,node distance=5mm]
\tikzstyle{none}=[inner sep=0pt]
\tikzstyle{every loop}=[]

\tikzstyle{(null)}=[]
\tikzstyle{plain}=[]

\usepackage{color} 

\newcommand{\semantics}[1]{[\![ #1 ]\!]} 
\newcommand{\ov}{\overrightarrow} 

\title{The Frobenius anatomy of word meanings II:\\ possessive relative pronouns}
\author{}  
\institute{} 
\date{}

\begin{document}
\maketitle 

{\small
\hspace{-2cm}\begin{tabular}{ccc}
 Mehrnoosh Sadrzadeh & Stephen Clark  & Bob Coecke     \\
 Queen Mary University of London & University of Cambridge &   University of Oxford     \\
 School of Electronic Eng. and Computer Science & \quad   Computer Laboratory &   Dept. of Computer Science     \\
{ \tt \small mehrs@eecs.qmul.ac.uk} &   {\tt \small stephen.clark@cl.cam.ac.uk}  & {\tt \small coecke@cs.ox.ac.uk}   
\end{tabular}
}

\begin{abstract}
 Within the categorical compositional distributional model of meaning, we provide semantic interpretations  for the subject and object roles of the possessive relative pronoun `whose'. This is done in terms of  Frobenius algebras over compact closed categories.  These  algebras  and their diagrammatic language expose  how meanings of words in relative clauses  interact with each other. We show how our interpretation is related to   Montague-style semantics and provide a  truth-theoretic interpretation. We also show how vector spaces provide a concrete interpretation and provide  preliminary corpus-based experimental evidence.  In a prequel to this paper, we used similar methods and dealt with the case of subject and object relative pronouns. 
 
\vspace{0.3cm}
 {\bf \small Dedicated to Roy Dyckhoff on the occasion of his 63rd birthday.}

 \end{abstract}

\section{Introduction}

Mathematical linguistics is a field of Computational Linguistics that formalises and reasons about properties of natural language. Through the seminal work of Lambek  \cite{Lambek0}, certain formal models of natural language  have found connections to algebraic and categorical  models of programming languages. In a nutshell, these more abstract models of both fields use  the algebraic structure of  residuated monoids. In the context of natural language,  grammatical types are elements of the monoid, their juxtapositions are modelled by the monoid multiplications and  grammatical reductions are modelled by  the partial orders of the algebra. Relational or functional words, such as verbs,  have implicative types and the  residuals  to the monoid multiplication are used to model them. Later,  Lambek simplified these algebras  and developed a structure he named a   \emph{pregroup}, wherein, there are no  binary residual operations, but each element of the algebra has a left and a right residual \cite{Lambek1}. Pregroups have been applied to  formalising and reasoning about grammatical structures of different families of natural language, for recent examples see \cite{CasadioLambek}. 

Vector spaces have been applied to formalise meanings of words in natural language, leading to models referred to as \emph{distributional} \cite{Schutze}.  In these models, the semantics of a word is a vector with coordinates  based on the frequency of the co-occurrence of that word in the context of other words \cite{Firth}. For instance, the word `queen'  often appears in the context of `reign'  and `rule', so its meaning can be guessed from the meanings of `reign' and `rule', whereas the word `carnivore'  has a different meaning since it appears in the context of, for instance,  `animal' and `meat'. This model has been successfully applied to natural language processing tasks, such as building automatic thesauri \cite{Curran}. 

Vector spaces with linear maps and pregroups  both have a compact closed categorical structure \cite{PrellerLambek,KellyLaplaza}. Based on this common structure, in previous work we provided an interpretation of pregroups in vector spaces  and  developed a vector space semantics for them \cite{Coeckeetal,Coeckeetal2}.  This semantics  extends the distributional models from words to sentences. It provides an abstract setting where  one  constructs meaning vectors  for strings of words compositionally,  based on  their  grammatical   structure and the distributional vectors of the words within them.   The theoretical predictions of a fragment of this setting  were implemented   and  evaluated on     language  tasks, such as disambiguation, phrase similarity, and term/definition classification  \cite{GrefenSadr1,GrefenSadr2,Kartetal2,Kartetal3}.

No matter how successful they have been in  certain natural language processing tasks, distributional models cannot be used to build vectors for  words whose meanings do not depend on the context. Among these words are logical words such as `and, or',articles, such as `a', `the',  and relative pronouns such as `who, that, whose'. Our focus in this and a prequel paper \cite{SadrClarkCoecke} has been  on relative pronouns.  For instance, the word `queen' can be described by the clause `a woman who rules a country', and  the word `clown' by the clause `a man who performs funny tricks'; making the relative pronoun `who'  appear in  very different contexts. Because of this context-independent property,  relative pronouns are often treated as `noise' by distributional models. As a result, they are not taken into account when  building vector representations for the clauses containing them. Even if they were (and as we will show in the last section of the paper) their vectors would be so dense that operating with them would be similar to operating with a vector consisting of 1's. Either of these solutions discards the  vital information that is encoded in the structural semantics of a relative pronoun; the information that tells us how different parts of the  clause are related to each other and which helps us derive the meaning of the clause and hence that of the word it is describing. 

In a prequel to this paper \cite{SadrClarkCoecke}, we developed a compositional distributional semantics for the subject and object relative pronouns `who, which, that, whom'.  Here we  provide a semantics  for the possessive relative pronoun `whose', in its object and subject roles.  This semantics does not depend on the context but rather on the structural roles of   the pronouns. We use  the general operations of a  Frobenius  algebra  over a vector space  \cite{CoeckeVic} to model these structural roles. The computations of the  algebra and vector space  are depicted using a   diagrammatic  language \cite{Joyal,Selinger} that  depicts  the   interactions that happen among the words of a clause and  produce the meaning of the overall compound. The diagrams  visualise the role of the relative pronoun in passing the information of the head of the clause to the rest of the clause, acting  on this information,  copying,  unifying,  and even discarding it.   Using these diagrams we show how possessive relative clauses can be decomposed to  clauses with subject and object relative pronouns that contain a possessive predicate such as `has'.  

Further, we instantiate our mathematical constructions in a truth theoretic setting  and show how the vector constructions also provide us with Montague-stye set theoretic semantics \cite{Montague}. We also instantiate our model on vector spaces built from a corpus of real data  and develop  linear algebraic  forms of the categorical morphisms. Finally, we provide a preliminary experiment  by developing a toy dataset of words and their relative-clause descriptions and show how the cosine  of the angle between the vectors of the words and the vectors of their descriptions can be used to  assign the correct description to the word.   

\section{Related Work}
As opposed to  simple models such as that of  \cite{Lapata},  a compositional distributional model that does not ignore relative pronouns is that of \cite{Baroni2}. In this line of work, the meaning vectors of phrases and sentences are computed based on their syntactic structure. Each word within a phrase or sentence has a grammatical type and based on this type a matrix or  tensor is assigned to it, representing its meaning; the composition operator is a matrix multiplication,  or its generalisation: tensor contraction. For instance, the meaning of an adjective noun phrase is  computed by multiplying the matrix of an adjective with the vector of a noun.  The meaning of a simple transitive sentence is computed by contracting the tensor of the verb with the vectors of object and subject. To some extent,  this approach is the same as ours: words that have function types live in tensor spaces and the composition operator is composition of the linear maps corresponding to these tensors. 

There are two  differences. Firstly, this model estimates the tensors of words by doing  regression on the co-occurrence vectors of their contexts. For instance, the matrix of `red' is estimated from the  adjective noun phrases such as `red car', `red carpet', `red wine', etc.  We, on the other hand,  do not bind our method to any concrete construction and the concrete constructions that we have been using are very different from the above. However, the linear-regression constructions can also be embedded in our setting and provide the same  results. There is  a second more important difference, which shows itself in the developments  of the present (and its prequel) paper on relative pronouns. The types that the models of \cite{Baroni,Baroni2} consider are purely syntactic and do  not contain semantic  information. The meaning of the  relative pronoun `which' is obtained by  taking the intersection of the meaning of its head with the  meaning of the rest of the clause. But in the absence of an intersection operator in vector spaces, the authors move to a simpler approach, where the meaning of `which' is a function that inputs a verb phrase, e.g. `eats meat' and outputs a modifies noun phrase, e.g. `which eats meat'. This bypass is not necessary in our model. As a result, we do not need to build many-dimensional tensors for the relative pronouns and the two defects of data sparsity and computational power, mentioned in \cite{Baroni2}, are automatically overcome.  These defects arise since, for example, the tensor of `which' will have four dimensions; even in a vector space model where the dimensions are reduced this will cause a problem. For instance in a vector space with 300 dimensions, the tensor of `whose' will have $300^4 = 81 \times 10^8$  (8.1. billion)  dimensions.  Compare this to  our setting, where we do not need to build any concrete tensor for relative pronouns: the Frobenius operations allow us to encode their syntactic and semantic roles using simple operations on the meanings of the other words of the clause. Finally, the approach of \cite{Baroni2} only treats the relative pronoun `which', here we also deal with the possessive pronoun `whose'.

\section{Compact Closed Categories, Diagrams,  Examples}
\label{sec:comp}

In order to ground the constructions that will provide us with meanings of relative pronouns, we need to discuss the theory of categories and in particular the definition and operations of compact closed categories. Theory of categories is a mathematical theory that abstracts away from the concrete details of structures and relates them to each other via the high-level properties that they hold. It was this theoretical tool that enabled us to relate the grammatical structures of sentences to vector semantics in a compositional  way. Frobenius algebras are operators that can be applied to certain objects within compact closed categories. These will provide our setting with  extra expressive powers and will help us  embed the features that are required to model relative pronouns. In a prequel to this paper \cite{SadrClarkCoecke}, we worked with purely formal definitions. In this paper, we explain these in an informal way. We also use a diagrammatic calculus to depict them; these should help the reader follow the explanations easier. The diagrams will also help us depict the computations necessary for providing meaning for the role of relative pronouns within linguistic compounds.

The theory of categories is the study of abstract mathematical structures referred to by \emph{categories}. Categories have \emph{objects} and \emph{morphisms}. Examples of objects are sets, elements of a set, groups or elements of a group. In the context of linguistics, they are usually taken to be grammatical types of words. Morphisms map objects to each other, they might denote a way of relating  objects to each other or transforming an object to another.   For instance, if the objects of a category are sets, the morphisms can be functions or relations between the sets. In the context of linguistics, they denote the grammatical reductions between the types. 

If we denote the objects of a category by $A, B, C, ...$, the morphisms will be  denoted by $f \colon A \to B$, $g \colon B \to C$. The morphisms can compose with each other, that is  whenever we have morphisms $f$ and $g$ defined as above, we also have a morphisms $h \colon A \to C$, and $h$ is a composition of $f$ and $g$, denoted by  $g\circ f$. There is also a special morphism called \emph{identity} that transforms an object to itself; it is denoted by $1_A \colon A \to A$. In the context of sets, this can be the morphism that maps the elements of a set to other elements of the same set.  Diagrammatically, the objects and the identity morphisms are depicted by lines. All other  morphisms are depicted by boxes.  For instance  a morphism $f \colon A \to B$ and an object $A$ and its  identity arrow  $1_A \colon A \to A$ are depicted as follows:

\vspace{-0.2cm}
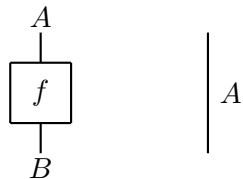
\begin{center}
  {%
\beginpgfgraphicnamed{compact-diag}
\begin{tikzpicture}[scale=0.8]
	\begin{pgfonlayer}{nodelayer}
		\node [style=none] (0) at (-9, -1) {};
		\node [style=none] (1) at (-7, -1) {};
		\node [style=none] (2) at (-7, 1) {};
		\node [style=none] (3) at (-9, 1) {};
		\node [style=none] (4) at (-8, 0) {$f$};
		\node [style=none] (5) at (-8, 1) {};
		\node [style=none] (6) at (-8, 2) {};
		\node [style=none] (7) at (-8, -1) {};
		\node [style=none] (8) at (-8, -2) {};
		\node [style=none] (9) at (-8, 2.5) {$A$};
		\node [style=none] (10) at (-8, -2.5) {$B$};
		\node [style=none] (11) at (-2.5, 2) {};
		\node [style=none] (12) at (-2.5, -2) {};
		\node [style=none] (13) at (-1.75, 0) {$A$};
	\end{pgfonlayer}
	\begin{pgfonlayer}{edgelayer}
		\draw [style=thick] (3.center) to (0.center);
		\draw [style=thick] (3.center) to (2.center);
		\draw [style=thick] (2.center) to (1.center);
		\draw [style=thick] (1.center) to (0.center);
		\draw [style=thick] (6.center) to (5.center);
		\draw [style=thick] (7.center) to (8.center);
		\draw [style=thick] (11.center) to (12.center);
	\end{pgfonlayer}
\end{tikzpicture}}
\endpgfgraphicnamed}
\end{center}

One can define operations on the objects. For example, if objects are sets, one can take their intersection, union, or Cartesian product. In the context of linguistics, one needs to juxtapose the grammatical types to obtain the  type of a juxtaposition of words. The abstract form of the juxtaposition operation is called a  monoidal tensor and is denoted by $A \otimes B$. So if $A$ is the grammatical type of the word $w_1$, e.g. `red'  and $B$ is the grammatical type of the word $w_2$, e.g. `car', then $A \otimes B$ is the type of the string of words $w_1 w_2$, that is `red car'. Such operations usually have a unit, for the case of union, the unit is the empty set, denoted by $\emptyset$, since we have $A \cup \emptyset = A$. In the case of juxtaposition, the unit is the empty type, denoted by $1$, since we have that the juxtaposition of a type with an empty type is the original type, that is $A \otimes I = A$.  A category that has a monoidal tensor with a unit (and which satisfies certain other equations, which we will not present here), is called a \emph{monoidal closed category}.  

Diagrammatically, the tensor products of the objects and morphisms are depicted by juxtaposing their diagrams side by side, whereas compositions of morphisms are depicted by putting one on top of the other, for instance the object $A \otimes B$, and the morphisms $f \otimes g$ and $f \circ h$, for $f \colon A \to B, g \colon C \to D$, and $h \colon B \to C$ are depicted as follows:

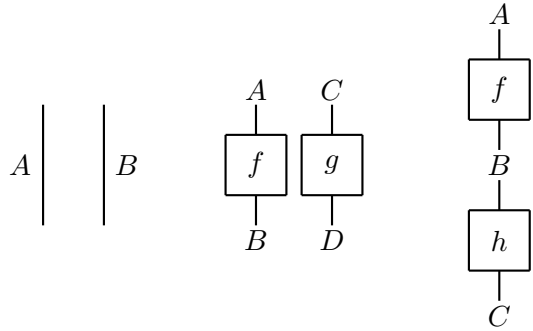
\begin{center}
  {%
\beginpgfgraphicnamed{compact-diag-tensor}
\begin{tikzpicture}[scale=0.8]
	\begin{pgfonlayer}{nodelayer}
		\node [style=none] (0) at (-4, 0) {};
		\node [style=none] (1) at (-2, 0) {};
		\node [style=none] (2) at (-2, 2) {};
		\node [style=none] (3) at (-4, 2) {};
		\node [style=none] (4) at (-3, 1) {$f$};
		\node [style=none] (5) at (-3, 2) {};
		\node [style=none] (6) at (-3, 3) {};
		\node [style=none] (7) at (-3, 0) {};
		\node [style=none] (8) at (-3, -1) {};
		\node [style=none] (9) at (-3, 3.5) {$A$};
		\node [style=none] (10) at (-3, -1.5) {$B$};
		\node [style=none] (11) at (-0.5, -1.5) {$D$};
		\node [style=none] (12) at (-0.5, 1) {$g$};
		\node [style=none] (13) at (-0.5, 2) {};
		\node [style=none] (14) at (-1.5, 0) {};
		\node [style=none] (15) at (-0.5, -1) {};
		\node [style=none] (16) at (0.5, 0) {};
		\node [style=none] (17) at (-0.5, 3.5) {$C$};
		\node [style=none] (18) at (-0.5, 3) {};
		\node [style=none] (19) at (0.5, 2) {};
		\node [style=none] (20) at (-1.5, 2) {};
		\node [style=none] (21) at (-0.5, 0) {};
		\node [style=none] (22) at (5, 2.5) {};
		\node [style=none] (23) at (5, 1.5) {};
		\node [style=none] (24) at (4, 4.5) {};
		\node [style=none] (25) at (5, 4.5) {};
		\node [style=none] (26) at (5, 5.5) {};
		\node [style=none] (27) at (5, 3.5) {$f$};
		\node [style=none] (28) at (6, 2.5) {};
		\node [style=none] (29) at (5, 6) {$A$};
		\node [style=none] (30) at (4, 2.5) {};
		\node [style=none] (31) at (5, 1) {$B$};
		\node [style=none] (32) at (6, 4.5) {};
		\node [style=none] (33) at (5, -2.5) {};
		\node [style=none] (34) at (5, -3.5) {};
		\node [style=none] (35) at (4, -0.5) {};
		\node [style=none] (36) at (5, -0.5) {};
		\node [style=none] (37) at (5, 0.5) {};
		\node [style=none] (38) at (5, -1.5) {$h$};
		\node [style=none] (39) at (6, -2.5) {};
		\node [style=none] (40) at (4, -2.5) {};
		\node [style=none] (41) at (5, -4) {$C$};
		\node [style=none] (42) at (6, -0.5) {};
		\node [style=none] (43) at (-10, -1) {};
		\node [style=none] (44) at (-10, 3) {};
		\node [style=none] (45) at (-10.75, 1) {$A$};
		\node [style=none] (46) at (-8, -1) {};
		\node [style=none] (47) at (-8, 3) {};
		\node [style=none] (48) at (-7.25, 1) {$B$};
	\end{pgfonlayer}
	\begin{pgfonlayer}{edgelayer}
		\draw  [style=thick] (3.center) to (0.center);
		\draw  [style=thick] (3.center) to (2.center);
		\draw  [style=thick] (2.center) to (1.center);
		\draw  [style=thick] (1.center) to (0.center);
		\draw  [style=thick] (6.center) to (5.center);
		\draw  [style=thick] (7.center) to (8.center);
		\draw  [style=thick] (20.center) to (14.center);
		\draw  [style=thick] (20.center) to (19.center);
		\draw  [style=thick] (19.center) to (16.center);
		\draw  [style=thick] (16.center) to (14.center);
		\draw  [style=thick] (18.center) to (13.center);
		\draw  [style=thick] (21.center) to (15.center);
		\draw  [style=thick] (24.center) to (30.center);
		\draw  [style=thick] (24.center) to (32.center);
		\draw  [style=thick] (32.center) to (28.center);
		\draw  [style=thick] (28.center) to (30.center);
		\draw  [style=thick] (26.center) to (25.center);
		\draw  [style=thick] (22.center) to (23.center);
		\draw  [style=thick] (35.center) to (40.center);
		\draw  [style=thick] (35.center) to (42.center);
		\draw  [style=thick] (42.center) to (39.center);
		\draw  [style=thick] (39.center) to (40.center);
		\draw  [style=thick] (37.center) to (36.center);
		\draw  [style=thick] (33.center) to (34.center);
		\draw [style=thick] (44.center) to (43.center);
		\draw [style=thick] (47.center) to (46.center);
	\end{pgfonlayer}
\end{tikzpicture}}
\endpgfgraphicnamed}
\end{center}

A category is called \emph{compact closed}, whenever  it is monoidal closed and moreover, its objects can cancel each other out and generate each other. Cancelation means that there is way of transforming the tensor of certain objects to the unit of the tensor. Generation means that there is a way of transforming the unit to the tensor of certain objects. To make this property formal, we assign to each  object $A$, an  object denoted by $A^r$,  referred to by  the right adjoint of $A$,  and  another object denoted by $A^l$,  referred to by the left adjoint of $A$.  The morphism that transforms $A \otimes A^r$ to $I$ is referred to by $\epsilon_A^r$ and the morphism that transforms $A^l \otimes A$ to $I$ is referred to by $\epsilon^l_A$. These are denoted as follows:
\[
\text{\bf Cancelations}: \qquad A \otimes A^r   \stackrel{\epsilon_A^r} {\longrightarrow} I
\qquad
A^l \otimes A \stackrel{\epsilon_A^l}{\longrightarrow} I
\]

\noindent
The morphism that transforms   $I$ to $A^r \otimes A$ is referred to by $\eta_A^r$ and the morphism that transforms $I$ to $A^l \otimes A$ is referred to by $\eta_A^l$. These are denoted as follows:
\[
\text{\bf Generations}: \qquad  I  \stackrel{\eta_A^r}{\longrightarrow} A^r \otimes A
\qquad
 I \stackrel{\eta_A^l}{\longrightarrow} A \otimes A^l
\]

\noindent
The $\epsilon$ maps are depicted by cups, and the $\eta$ maps by caps.  For instance, the diagrams for  $\epsilon^l \colon A^l \otimes A \to I$ and $\eta \colon I \to A\otimes A^l$ are  as follows:
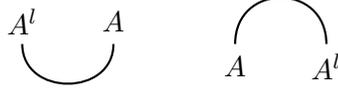
\begin{center}
  {%
\beginpgfgraphicnamed{compact-cap-cup}
\begin{tikzpicture}[scale=0.8]
	\begin{pgfonlayer}{nodelayer}
		\node [style=none] (0) at (-5, 0) {};
		\node [style=none] (1) at (-2, 0) {};
		\node [style=none] (2) at (-5, 0.75) {$A^l$};
		\node [style=none] (3) at (2, -0.75) {$A$};
		\node [style=none] (4) at (5, -0.75) {$A^l$};
		\node [style=none] (5) at (2, 0) {};
		\node [style=none] (6) at (5, 0) {};
		\node [style=none] (7) at (-2, 0.75) {$A$};
	\end{pgfonlayer}
	\begin{pgfonlayer}{edgelayer}
		\draw [thick, bend right=90, looseness=1.50] (0.center) to (1.center);
		\draw [thick, bend left=90, looseness=1.75] (5.center) to (6.center);
	\end{pgfonlayer}
\end{tikzpicture}}
\endpgfgraphicnamed}
\end{center}

The  morphisms of a compact closed category satisfy  certain formal equations. The most important one consists of four equations referred to by \emph{yanking}. One of the instances of this property is  $(\epsilon^l \otimes 1_A) \circ (1_A \otimes \eta^l)  = 1_A$,   depicted below:  

\begin{center}
    {%
\beginpgfgraphicnamed{compact-yank}
\begin{tikzpicture}[scale=0.8]
	\begin{pgfonlayer}{nodelayer}
		\node [style=none] (0) at (-5, 0) {};
		\node [style=none] (1) at (-2, 0) {};
		\node [style=none] (2) at (-5, 0.5) {$A^l$};
		\node [style=none] (3) at (-2, 0.5) {$A$};
		\node [style=none] (4) at (1, 0.5) {$A^l$};
		\node [style=none] (5) at (-2, 1) {};
		\node [style=none] (6) at (1, 1) {};
		\node [style=none] (7) at (3, 0) {$=$};
		\node [style=none] (8) at (5, 2.5) {};
		\node [style=none] (9) at (5, -1.5) {};
		\node [style=none] (10) at (5.5, 0.5) {$A$};
		\node [style=none] (11) at (-5, 1) {};
		\node [style=none] (12) at (-5, 2.5) {};
		\node [style=none] (13) at (1, 0) {};
		\node [style=none] (14) at (1, -1.5) {};
	\end{pgfonlayer}
	\begin{pgfonlayer}{edgelayer}
		\draw [thick, bend right=90, looseness=1.50] (0.center) to (1.center);
		\draw [thick, bend left=90, looseness=1.75] (5.center) to (6.center);
		\draw [style=thick] (8.center) to (9.center);
		\draw [style=thick] (12.center) to (11.center);
		\draw [style=thick] (13.center) to (14.center);
	\end{pgfonlayer}
\end{tikzpicture}}
\endpgfgraphicnamed}
\end{center}

\noindent 
Roughly speaking, yanking expresses the fact that a cancelation followed by a generation (or the other way around) is the same as doing nothing, that is, the same as having  an identity morphism. 

\subsection{A compact closed category for grammar}

Theory of categories was first applied to the analysis of grammatical structure of language in \cite{Lambek0}. The categories presented and studied there were monoidal closed.  The first application of compact closed categories to grammar was in \cite{Lambek}. The argument behind the passage from monoidal to compact was mainly simplicity. Monoidal closed categories have two other operations (other than the tensor) and all of  these three operations are needed when analysing grammar. In compact closed categories, the applications of these two other operations are modelled by the cancelation and generation morphisms.  

The compact closed category of grammar is called a \emph{pregroup}.  The   objects of this category are grammatical types. The morphisms of this category are grammatical reductions. The objects are denoted by $p,q, r, \cdots$, and the morphisms are denoted by partial orders such as $p \leq q$. The partial order morphisms compose with each other as follows: whenever we have $p \leq q$ and $q \leq r$, we also have that $p \leq r$. This is because  partial ordering is a transitive relation. The tensor product of the category denotes the juxtaposition of types. So $p \otimes q$ denotes the juxtaposition of type $p$ with type $q$. This can be the type of a two word phrase `$w_1 w_2$', where $w_1$ has type $p$ and $w_2$ has type $q$. 

Since the category is compact, each type $p$ has a right adjoint $p^r$ and a left adjoint $p^l$. This means that we have the following cancelation and generation morphisms:

\begin{align*}
\text{\bf Cancelation} \qquad&  \epsilon_p^r \colon \ p \otimes p^r \leq 1 \hspace{2cm} \epsilon_p^l \colon \  p^l \otimes p \leq 1\\
\text{\bf Generation} \qquad&  \eta_p^r \colon \  1 \leq p^r \otimes p \hspace{2cm}  \eta_p^l \colon \ 1 \leq p \otimes p^l
\end{align*}

The linguistic motivation behind the adjoint types is as follows.  The vocabulary of a language consists of two kinds of words: the ones that are atomic such as nouns and the ones that are relational such as adjectives and verbs. The relational types input the atomic types as their arguments and then modify them. The relational types  have adjoints in them and the adjoint types represent their arguments. So when such types are juxtaposed with their arguments (in the right order), they will cancel each other out and produce a new modified type. The left and right labels of adjunction denote the order with which a relational type needs to be juxtaposed with its arguments.  This order depends on the grammar and differs from language to language. For instance, in English the adjective `red' occurs before the noun `car'.  So the adjective has type $n\otimes n^l$, for $n$ the type of a noun. This means that it needs an argument of type $n$ and it has to be to the left of its argument. The grammatical reduction of an adjective-noun phrase `red car' is denoted by the  following morphism:

\begin{center}
\begin{tabular}{cccc}
red & & car &\\
$(n\otimes n^l)$ & $ \otimes$ & $n$ & $ \leq n \otimes 1 = n$
\end{tabular}
\end{center}

\noindent
Here we are transforming the  $n$ in the type of the adjective to itself and we are cancelling the $n^r$ with $n$ in the type of the noun. So the above grammatical reduction corresponds to the morphism $1_n \otimes \epsilon^l_n$.  As another example consider a simple intransitive sentence such as `men sneeze'. Here the noun `men' has type $n$ and the intransitive verb `sneeze' has type $n^r \otimes s$. This type means that, according to the grammar of English, an intransitive verb inputs a noun (for simplicity we assign the same type $n$ to nouns and noun phrases) as its subject  and has to be the right of that noun. After inputing the subject and modifying it, the verb will produce a sentence, denoted by the  type $s$. This grammatical reduction is denoted by the  morphism $\epsilon^r_n \otimes 1_s$ obtained as follows:

\begin{center}
\begin{tabular}{cccc}
men & & sneeze &\\
$n$ & $ \otimes$ & $(n^r \otimes s)$ & $ \leq 1 \otimes s = s$
\end{tabular}
\end{center}

Finally, consider a simple transitive sentence such as `men like cats'. Here, the transitive verb `like' needs two arguments of type $n$ and has to be to the right of one and the left of the other. After inputting these nouns, it will modify them and produce a sentence, thus, it has type $n^r \otimes s \otimes n^l$. The morphisms corresponding to the grammatical reduction of a transitive sentence is $\epsilon^r_n \otimes 1_s \otimes \epsilon^l_n$ and is obtained as follows:

\begin{center}
\begin{tabular}{cccccc}
men & & like & & cats\\
$n$ & $ \otimes$ & $(n^r \otimes s \otimes n^l)$ &  $\otimes$ & $n^l$& $ \leq 1 \otimes s \otimes 1 = s$
\end{tabular}
\end{center}

\subsection{A compact closed category for  meaning}

Distributional models of meaning represent meanings of words by vectors. These vectors live in a finite dimensional vector space with a fixed set of basis vectors. Such vector spaces also form a compact closed category. In this category,  vector spaces $V, W, ...$ are  objects and   linear mappings between them $f \colon V \to W$ are morphisms. The  composition of morphisms is   the composition of linear maps. The  tensor product is the tensor product of vector  spaces $V \otimes W$, whose  unit is the scalar field of the vector spaces; in our case this is the field of reals $\mathbb{R}$.   The left and right adjoints are the same, that is we have   $V^l \cong V^r \cong V^*$, where $V^*$ is  the dual  space  of  $V$. Since the basis vectors of  these vector spaces are fixed, we have an isomorphism between  $V^*$ and  $V$, that is $V^* \cong V$.

Diagrammatically speaking, vector spaces are lines and vectors within them are triangles.   Each triangle has a number of strings emanating  from it. This number denotes the tensor rank of the  vector, for instance,  $\ov{v} \in V, \ov{v'} \in V \otimes W$, and $\ov{v''} \in V \otimes W \otimes Z$ are depicted as follows:
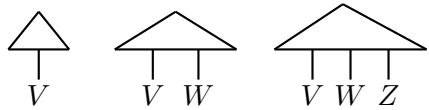
\begin{center}
  {%
\beginpgfgraphicnamed{compact-diag-triangle}
\begin{tikzpicture}[scale=0.8]
	\begin{pgfonlayer}{nodelayer}
		\node [style=none] (0) at (-5, 0) {};
		\node [style=none] (1) at (-1, 0) {};
		\node [style=none] (2) at (-3, 1.25) {};
		\node [style=none] (3) at (-3.75, 0) {};
		\node [style=none] (4) at (-3.75, -1) {};
		\node [style=none] (5) at (-3.75, -1.5) {$V$};
		\node [style=none] (6) at (-2.25, 0) {};
		\node [style=none] (7) at (-2.25, -1.5) {$W$};
		\node [style=none] (8) at (-2.25, -1) {};
		\node [style=none] (9) at (2.75, -1) {};
		\node [style=none] (10) at (2.75, -1.5) {$W$};
		\node [style=none] (11) at (2.5, 1.5) {};
		\node [style=none] (12) at (1.5, 0) {};
		\node [style=none] (13) at (5.25, 0) {};
		\node [style=none] (14) at (1.5, -1.5) {$V$};
		\node [style=none] (15) at (1.5, -1) {};
		\node [style=none] (16) at (0.25, 0) {};
		\node [style=none] (17) at (2.75, 0) {};
		\node [style=none] (18) at (4, -1.5) {$Z$};
		\node [style=none] (19) at (4, 0) {};
		\node [style=none] (20) at (4, -1) {};
		\node [style=none] (21) at (-8.5, 0) {};
		\node [style=none] (22) at (-7.5, -1.5) {$V$};
		\node [style=none] (23) at (-7.5, 0) {};
		\node [style=none] (24) at (-6.5, 0) {};
		\node [style=none] (25) at (-7.5, -1) {};
		\node [style=none] (26) at (-7.5, 1.25) {};
		\node [style=none] (29) at (2.5, 0.5) {};
	\end{pgfonlayer}
	\begin{pgfonlayer}{edgelayer}
		\draw  [style = thick](0.center) to (2.center);
		\draw  [style = thick](2.center) to (1.center);
		\draw  [style = thick](1.center) to (0.center);
		\draw  [style = thick](3.center) to (4.center);
		\draw  [style = thick](6.center) to (8.center);
		\draw  [style = thick](16.center) to (11.center);
		\draw  [style = thick](11.center) to (13.center);
		\draw  [style = thick](13.center) to (16.center);
		\draw  [style = thick](12.center) to (15.center);
		\draw  [style = thick](17.center) to (9.center);
		\draw  [style = thick](19.center) to (20.center);
		\draw  [style = thick](21.center) to (26.center);
		\draw  [style = thick](26.center) to (24.center);
		\draw  [style = thick](24.center) to (21.center);
		\draw  [style = thick](23.center) to (25.center);
	\end{pgfonlayer}
\end{tikzpicture}}
\endpgfgraphicnamed}
\end{center}

Given a basis $\{\ov{r}_i\}_i$ for a vector space $V$, the two cancelation maps become isomorphic to one, since we have $V^* \otimes V \cong  V \otimes V^* \cong V \otimes V$. This   map is as follows: 
\[
\epsilon_V\colon \    V \otimes V \to \mathbb{R}
\]
Concretely, the application of this map to  vectors $\ov{v}, \ov{w}$  from $V$ is taking their inner product, which provides us with a number in $\mathbb{R}$, defined  as follows:
\[
\epsilon_V(\ov{v} \otimes \ov{w}) =  \langle \ov{v} \mid \ov{w} \rangle
\]
Representing vectors by their linear expansions, that is $\ov{v} = \sum_i c_i \ov{r}_i, \ov{w} = \sum_j c_j \ov{r}_j$,  the above becomes equivalent to the following:
\[
c_{ij} \sum_{ij} \langle \ov{r}_i \mid \ov{w}_j \rangle
\]

For the same reason as above, the two  generation maps also become isomorphic to the following one:

\[
\eta\colon \ \mathbb{R} \to V \otimes V
\]
Concretely, the application of this map to a real number $k \in \mathbb{R}$  produces a vector in $V \otimes V$ defined as follows 
\[
\eta(k) =  k \sum_i \ov{r}_i \otimes \ov{r}_i
\]
For an example, take $V$ to be a two dimensional space with the basis $\{\ov{r}_1, \ov{r}_2\}$. An example of the generation map in this space is as follows:
\[
\eta(k) = k (\ov{r}_1 \otimes \ov{r}_1 \ + \ \ov{r}_2 \otimes \ov{r}_2)
\]
Because the basis vectors of $V$ are fixed, the above can be equivalently written in the form of a  $2 \times 2$ matrix as follows:
\[ \eta(k) \ = \ 
\left(
\begin{array}{cc}
k \quad & 0 \\ 0 \quad & k 
\end{array}
\right)
\]

\section{Frobenius Algebras, Diagrams, and Examples}
\label{sec:frob}

Again we refer the reader for a formal definition and references  to the prequel paper \cite{SadrClarkCoecke}. Informally speaking, some of the objects of a compact closed category might have special properties, referred to by \emph{copying} and \emph{uncopying}.  Note that these maps are only defined on certain objects of the category and the family of these objects need not coincide with the family of all the objects of the category. Although the general definition of  Frobenius algebras is over any compact closed category, for the purpose of this paper, we will consider the case of vector spaces.

The \emph{copying} property is an expression of the fact that  there is a linear way of transforming a certain  vector space $V$ to the vector space $V \otimes V$.  The linear map corresponding to this transformation is denoted as follows:
\[
\Delta \colon   V \to V \otimes V
\]
Concretely, it acts on a vector $\ov{v} = \sum_i c_i \ov{r}_i$ of $V$ as follows:
\[
\Delta (\ov{v}) =  \sum_i c_i \ov{r}_i \otimes \ov{r}_i
\]
So for a two dimensional space with basis vectors $\{\ov{r}_1, \ov{r}_2\}$, we have:
\[
\Delta(\ov{v}) =  \Delta (c_1 \ov{r}_1 + c_2 \ov{r}_2)  =  c_1 (\ov{r}_1 \otimes \ov{r}_1) \ + \ c_2 (\ov{r}_2 \otimes \ov{r}_2)
\]
Using the matrix notation, the above can be equivalently written as:
\[
\Delta \left(\begin{array}{c} c_1   \\ c_2\end{array} \right)  = \left ( \begin{array}{cc} c_1 & \quad 0 \\ 0 & \quad c_2 \end{array} \right)
\]
The copying map has a unit  $\iota$ of the type $V \to \mathbb{R}$, which transforms a certain  vector space to the scalar field. Concretely, it sends a vector  to the sum of its co-ordinates. So in general for $\ov{v} \in V$, we have:
\[
\iota(\ov{v}) = \iota(\sum_i c_i \ov{r}_i) = \sum_i c_i
\]
An example consider $\iota(\ov{v}) = \iota( c_1 \ov{r}_1 + c_2 \ov{r}_2) = c_1 + c_2$. 

In the context of linguistics one can think of copying as a way of being able to dispatch the information of a vector space to two vector spaces (and by analogy also for the vectors within these vector spaces). For instance, it might be needed to input the information expressed in a noun to a relative pronoun and to the verb of the main clause. In this case, and as we will see in more detail later on, we will copy the vector of the noun and pass a copy to the relative pronoun and another copy to the verb.  In other words, we dispatch the information of the noun to the relative pronoun and to the verb.  The linguistic application  of the unit $\iota$ is that sometimes one needs to discard the information of a word or part of a word, in which case $\iota$ is applied to the vector of that word. An example is again the case of relative clauses, where the relative pronoun inputs the type of sentence $s$ from the verb but has to discard it as the output of a relative clause is a noun, rather than a sentence.

The  \emph{uncopying} map expresses the fact that there is a linear way of transforming the tensor product of a certain vector space $V$ with itself to $V$. That is we have the following linear map:
\[
\mu \colon V \otimes V \to V
\]
Concretely, the application of $\mu$ on vectors in $V$ is defined as follows:
\[
\mu (\ov{v} \otimes \ov{w}) = \mu(\sum_i c_i \ov{r}_i \otimes \sum_j c_j \ov{r}_j) =  \sum_i  c_i   c_j  \delta_{ij}  \ov{r}_j
\]
The notation $\delta_{ij} \ov{r}_i$ is defined as follows:
\[
\delta_{ij} \ov{r}_j = \begin{cases} \ov{r}_i & i = j \\ 0 & i \neq j \end{cases}
\]
As an example we have 
\[
\mu (\ov{v} \otimes \ov{w}) = \mu ((c_1 \ov{r}_1 + c_2 \ov{r}_2)\otimes (c_3 \ov{r}_3  + c_4 \ov{r}_4)) = c_1 c_3 \ov{r}_1 + c_2 c_4 \ov{r}_2
\]
Using the matrix notation, the above can be written as follows
\[
\mu\left(\left(\begin{array}{c} c_1 \\ c_2 \end{array}\right) \otimes \left ( \begin{array}{c} c_3 \\ c_4 \end{array} \right) \right) = 
\mu \left( \begin{array}{cc}  c_1c_3 & \quad c_1c_4 \\ c_2c_3 & \quad c_2c_4  \end{array}\right) = \left( \begin{array}{c} c_1c_3 \\ c_2c_4\end{array} \right)
\]
The above is equal to the point wise multiplication of the two input vectors, that is $\ov{v} \odot \ov{w}$.  The uncopying map has a unit $\zeta \colon \mathbb{R} \to V$, which   transforms a real number $k$ to a vector whose all co-ordinates are $k$. That is, we have
\[
\zeta(k) = \sum_i k \ov{r}_i
\]
An example  consider $\zeta(k) = k\ov{r}_1  + k \ov{r}_2$.

The diagrammatic forms of the Froebnius morphisms are as follows:

\begin{center}
{%
\beginpgfgraphicnamed{comp-alg-coalg}
\begin{tikzpicture}[scale=0.8]
	\begin{pgfonlayer}{nodelayer}
		\node [style=none] (0) at (-3.5, 2.25) {};
		\node [draw, thick, style=none, minimum size=0.2 cm, circle, fill=white] (1) at (-4.5, 1.25) {};
		\node [style=none] (2) at (-5.5, 2.25) {};
		\node [style=none] (3) at (-4.5, 0.5) {};
		\node [style=none] (4) at (4.25, 1.5) {};
		\node [style=none] (5) at (6.25, 1.5) {};
		\node [style=none] (6) at (4.25, 1.5) {};
		\node [draw, thick, style=none, minimum size=0.2 cm, circle, fill=white] (7) at (5.25, 2.5) {};
		\node [style=none] (8) at (5.25, 3.25) {};
		\node [style=none] (9) at (6.25, 1.5) {};
		\node [style=none] (10) at (-7.5, 2) {$(\mu,\zeta)$};
		\node [draw, thick, style=none, minimum size=0.2 cm, circle, fill=white] (11) at (-2.25, 2.25) {};
		\node [style=none] (12) at (-2.25, 0.5) {};
		\node [style=none] (13) at (2.5, 2) {$(\Delta,\iota)$};
		\node [draw, thick, style=none, minimum size=0.2 cm, circle, fill=white] (14) at (7.5, 1.5) {};
		\node [style=none] (15) at (7.5, 3.25) {};
	\end{pgfonlayer}
	\begin{pgfonlayer}{edgelayer}
		\draw [style=thick] (1.center) to (3.center);
		\draw [style=thick, bend left=90, looseness=1.75] (6.center) to (5.center);
		\draw [style=thick] (8.center) to (7.center);
		\draw [thick, bend right=90, looseness=1.75] (2.center) to (0.center);
		\draw [style=thick] (11.center) to (12.center);
		\draw [style=thick] (15.center) to (14.center);
	\end{pgfonlayer}
\end{tikzpicture}}
\endpgfgraphicnamed}
\end{center}

In the context of  linguistics, the uncopying map can be thought of as a way of merging or combing information. In this case, the information of two vectors from a vector space $V$ can be merged into the information of one vector in $V$. For instance, after a relative pronoun has made a copy of the vector of a word, and these copies are dispatched to various parts of the clause, there is a  need to put the modified information together, in other words, merge them, to obtain one single vector as the output of the relative clause. We will see  examples of this feature in the proceeding sections. The $\zeta$ map is used to create a vector for a word that has been dropped from a phrase. For instance the relative pronoun `that' is usually dropped, as is the case in the clause `dogs I saw yesterday'. The original form of this phrase is `dogs that I saw yesterday'. In such cases, the $\zeta$ map enables us to generate a vector for the dropped pronoun. 

The above linear maps satisfy a number of equations (e.g. commutativity and specialty), the major of which is the  \emph{Frobenius} condition with the following formal form:

\vspace{-0.4cm}
\begin{align*}
(\mu \otimes 1_V) \circ (1_V \otimes \Delta) \ = \  \Delta \circ \mu  \ = \  (1_V \otimes \mu) \circ (\Delta \otimes 1_V) 
\end{align*}

\noindent
The diagrammatic form this property is as follows:

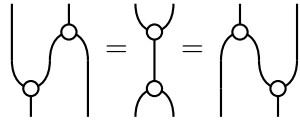
\begin{center}
{%
\beginpgfgraphicnamed{equation}
\begin{tikzpicture}
	\begin{pgfonlayer}{nodelayer}
		\node [style=none] (0) at (-1.25, 1.75) {$=$};
		\node [style=none] (1) at (0.75, 1.75) {$=$};
		\node [style=none] (2) at (0.25, 3) {};
		\node [style=none] (3) at (-0.75, 3) {};
		\node [draw, thick, style=none, minimum size=0.2 cm, circle, fill=white] (4) at (-0.25, 2.25) {};
		\node [style=none] (5) at (-0.75, 0) {};
		\node [style=none] (6) at (0.25, 0) {};
		\node [draw, thick, style=none, minimum size=0.2 cm, circle, fill=white] (7) at (-0.25, 0.75) {};
		\node [style=none] (8) at (-2.5, 3) {};
		\node [draw, thick, style=none, minimum size=0.2 cm, circle, fill=white] (9) at (-2.5, 2.25) {};
		\node [style=none] (10) at (-3, 1.5) {};
		\node [style=none] (11) at (-2, 1.5) {};
		\node [draw, thick, style=none, minimum size=0.2 cm, circle, fill=white] (12) at (-3.5, 0.75) {};
		\node [style=none] (13) at (-4, 1.5) {};
		\node [style=none] (14) at (-2, 0) {};
		\node [style=none] (15) at (-4, 3) {};
		\node [style=none] (16) at (-3.5, 0) {};
		\node [style=none] (17) at (1.5, 0) {};
		\node [style=none] (18) at (1.5, 0.75) {};
		\node [style=none] (19) at (2.5, 1.5) {};
		\node [style=none] (20) at (3.5, 1.5) {};
		\node [style=none] (21) at (1.5, 1.5) {};
		\node [draw, thick, style=none, minimum size=0.2 cm, circle, fill=white] (22) at (2, 2.25) {};
		\node [draw, thick, style=none, minimum size=0.2 cm, circle, fill=white] (23) at (3, 0.75) {};
		\node [style=none] (24) at (3, 0) {};
		\node [style=none] (25) at (2, 3) {};
		\node [style=none] (26) at (3.5, 3) {};
	\end{pgfonlayer}
	\begin{pgfonlayer}{edgelayer}
		\draw [thick, bend left=90, looseness=2.50] (2.center) to (3.center);
		\draw [style=thick, bend left=90, looseness=2.50] (5.center) to (6.center);
		\draw [style=thick] (4.center) to (7.center);
		\draw [style=thick] (8.center) to (9.center);
		\draw [style=thick, bend left=90, looseness=2.50] (10.center) to (11.center);
		\draw [style=thick, bend left=270, looseness=2.50] (13.center) to (10.center);
		\draw [style=thick] (11.center) to (14.center);
		\draw [style=thick] (15.center) to (13.center);
		\draw [style=thick] (12.center) to (16.center);
		\draw [style=thick] (17.center) to (18.center);
		\draw [style=thick, bend left=90, looseness=2.50] (21.center) to (19.center);
		\draw [style=thick, bend left=270, looseness=2.50] (19.center) to (20.center);
		\draw [style=thick] (23.center) to (24.center);
		\draw [style=thick] (25.center) to (22.center);
		\draw [style=thick] (21.center) to (18.center);
		\draw [style=thick] (26.center) to (20.center);
	\end{pgfonlayer}
\end{tikzpicture}}
\endpgfgraphicnamed}
\end{center}

Informally speaking, this  property says that if one has the tensor of a   vector space with itself, that is $V \otimes V$, then one can keep the first space and copy the second space followed by then uncopying the first $V$ with the first output of copying and keeping the second output  (or do these operations in the opposite order, that is copy the first $V$ and then uncopy the second output with $V$), the result is the same as  an uncopy followed by a copy. This property allows us to copy and uncopy in an alternate way  and many times, and at the end  being able to merge all the operations into a  big uncopying   and  copying, which performs all of the previous operations at once. Diagrammatically, we have   the following:

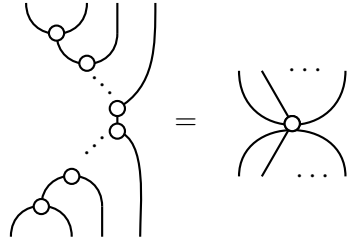
\[
{%
\beginpgfgraphicnamed{spider-normal-form}
\begin{tikzpicture}[scale=0.8]
	\begin{pgfonlayer}{nodelayer}
		\node [style=none] (0) at (-4.75, 3.25) {};
		\node [style=none] (1) at (-2.75, 3.25) {};
		\node [draw, thick, style=none, minimum size=0.2 cm, circle, fill=white] (2) at (-3.75, 2.25) {};
		\node [style=none] (3) at (-1.75, 2.25) {};
		\node [style=none] (4) at (-1.75, 3.25) {};
		\node [draw, thick, style=none, minimum size=0.2 cm, circle, fill=white] (5) at (-2.75, 1.25) {};
		\node [style=none] (6) at (-1.75, -0.25) {};
		\node [style=none] (7) at (-0.5, 3.25) {};
		\node [style=none] (8) at (0.5, -0.75) {$=$};
		\node [style=none] (9) at (2.25, 1) {};
		\node [style=none] (10) at (5.75, 1) {};
		\node [draw, thick, style=none, minimum size=0.2 cm, circle, fill=white] (11) at (4, -0.75) {};
		\node [style=none] (12) at (3, 1) {};
		\node [style=none] (13) at (4.5, 1) {$\cdots$};
		\node [style=none] (14) at (2.25, -2.5) {};
		\node [style=none] (15) at (5.75, -2.5) {};
		\node [style=none] (16) at (4.75, -2.5) {$\cdots$};
		\node [style=none] (17) at (-5.25, -4.5) {};
		\node [style=none] (18) at (-3.25, -4.5) {};
		\node [draw, thick, style=none, minimum size=0.2 cm, circle, fill=white] (19) at (-4.25, -3.5) {};
		\node [style=none] (20) at (-3.25, -2.5) {};
		\node [style=none] (21) at (-1, -4.5) {};
		\node [style=none] (22) at (-1.75, -1) {};
		\node [style=none] (23) at (-2.25, -3.5) {};
		\node [style=none] (24) at (-2.25, -4.5) {};
		\node [draw, thick, style=none, minimum size=0.2 cm, circle, fill=white] (25) at (-3.25, -2.5) {};
		\node [draw, thick, style=none, minimum size=0.2 cm, circle, fill=white] (26) at (4, -0.75) {};
		\node [style=none] (27) at (3, -2.5) {};
		\node [draw, thick, style=none, minimum size=0.2 cm, circle, fill=white] (28) at (-1.75, -0.25) {};
		\node [draw, thick, style=none, minimum size=0.2 cm, circle, fill=white] (29) at (-1.75, -1) {};
		\node [style=none] (30) at (-2.75, -1.75) {$\cdot$};
		\node [style=none] (31) at (-2.5, -1.5) {$\cdot$};
		\node [style=none] (32) at (-2.25, -1.25) {$\cdot$};
		\node [style=none] (33) at (-2.5, 0.75) {$\cdot$};
		\node [style=none] (34) at (-2.25, 0.5) {$\cdot$};
		\node [style=none] (35) at (-2, 0.25) {$\cdot$};
	\end{pgfonlayer}
	\begin{pgfonlayer}{edgelayer}
		\draw [style=thick, bend right=90, looseness=1.75] (0.center) to (1.center);
		\draw [style=thick, bend right=90, looseness=1.75] (2.center) to (3.center);
		\draw [style=thick] (4.center) to (3.center);
		\draw [style=thick, bend left=270, looseness=1.75] (9.center) to (10.center);
		\draw [style=thick] (12.center) to (11.center);
		\draw [style=thick, bend left=90, looseness=1.50] (14.center) to (15.center);
		\draw [style=thick, in=90, out=-15, looseness=0.75] (22.center) to (21.center);
		\draw [style=thick, in=30, out=-90] (7.center) to (6.center);
		\draw [thick, bend left=90, looseness=1.75] (17.center) to (18.center);
		\draw [thick, bend left=90, looseness=1.75] (19.center) to (23.center);
		\draw [style=thick](23.center) to (24.center);
		\draw [style=thick](27.center) to (26.center);
		\draw [style=thick](6.center) to (22.center);
	\end{pgfonlayer}
\end{tikzpicture}}
\endpgfgraphicnamed}
\]
This property is sometimes referred to by the \emph{spider} property. In the context of linguistics, this property expresses the fact that if we have two words (within the same vector space) and we dispatch the information of  the second word and then merge the first output with the first word, this will result in the same pair of words obtained from the operation of first merging the information of the original two words, then dispatching them into two other words.

\section{From Grammar to Meaning}

Since both pregroup grammars and vector spaces are compact closed,   there exists a structure-preserving map between the two $F \colon Preg \to FVect$. This map and its mathematical properties were developed and discussed in previous work \cite{Coeckeetal,Kartetal1,Coeckeetal2}. 
Here, we review its main properties, which include  assigning to each atomic grammatical type a vector space as follows:
\[
F(1) =  \mathbb{R}\qquad F(n)= N \qquad F(s) = S
\]
Naturally, the unit of juxtaposition is mapped to the unit of tensor product in vector spaces. We map the atomic type $n$ to a vector space $N$ and the atomic type $s$ to a vector space $S$. These vector spaces can be built in different ways. In the  sections that follow, we present two  instantiations for them. 

A juxtaposition of types is mapped to the tensor product of the vector spaces assigned to each type, that is we have:

\[
F(p\otimes q) = F(p) \otimes F(q)
\]
The two left and right adoints of each grammatical type are sent to the dual of the vector spaces assigned to the original type, that is we have: 
\[
F(p^l) = F(p^r) = F(p)^* \cong F(p)
\]

As an example,  consider the grammatical type of a transitive verb, that is $n^r \otimes s \otimes n^l$, then its vector space assignment  is computed as follows:

\[
F(n^r \otimes s \otimes n^l) = F(n^r) \otimes F(s) \otimes F(n^l) = 
 F(n) \otimes F(s) \otimes F(n) =  N\otimes S \otimes N
\]
 This assignment  means that the meaning vector of a
transitive verb is a vector in the tensor space $N \otimes S \otimes N$.

The grammatical reductions, i.e. the partial order morphisms of a pregroup,
are mapped to  linear maps. That is, a partial  $p \leq q$ in a pregroup,  is mapped to a   linear map $f_{\leq} \colon F(p) \to F(q)$ in vector spaces.  The cancelation $\epsilon$ and generation
$\eta$ maps of a pregroup are assigned to  cancelation and generation maps of a vector space.

As an example, recall the grammatical  reduction of a transitive  sentence. This reduction is interpreted as follows in vector spaces:

\[
F(\epsilon_n^r\otimes 1_s \otimes \epsilon_n^r) = F(\epsilon_n^r) \otimes F(1_s) \otimes
F(\epsilon_n^l) = 
F(\epsilon_n) \otimes F(1_s) \otimes
F(\epsilon_n)  =  \epsilon_N \otimes 1_S \otimes \epsilon_N
\]

To obtain a vector meaning for a string of words,  we apply the vector space interpretation of the grammatical reduction of the string to the vectors of the words within the string.  For instance, the vector
meaning of  the sentence `men like cats'  is as follows:
\[F(\epsilon_n^r \otimes 1_s
\otimes \epsilon_n^l)(\ov{\mbox{\em men}} \otimes \ov{\mbox{\em like}} \otimes
\ov{\mbox{\em cats}})\]
 This meaning is depictable as follows:

 \begin{center}
  {%
\beginpgfgraphicnamed{vect-sem-eg}
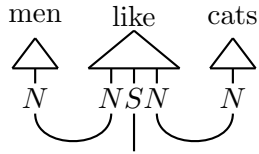
\begin{tikzpicture}[scale=0.8]
	\begin{pgfonlayer}{nodelayer}
		\node [style=none] (0) at (-3.5, 1) {};
		\node [style=none] (1) at (-2, 1) {};
		\node [style=none] (2) at (-2.75, 2) {};
		\node [style=none] (3) at (-1, 1) {};
		\node [style=none] (4) at (2, 1) {};
		\node [style=none] (5) at (0.5, 2.25) {};
		\node [style=none] (6) at (4.5, 1) {};
		\node [style=none] (7) at (3.75, 2) {};
		\node [style=none] (8) at (3, 1) {};
		\node [style=none] (9) at (-2.75, 1) {};
		\node [style=none] (10) at (-0.25, 1) {};
		\node [style=none] (11) at (0.5, 1) {};
		\node [style=none] (12) at (1.25, 1) {};
		\node [style=none] (13) at (3.75, 1) {};
		\node [style=none] (14) at (-2.75, 0.5) {};
		\node [style=none] (15) at (-0.25, 0.5) {};
		\node [style=none] (16) at (0.5, 0.5) {};
		\node [style=none] (17) at (1.25, 0.5) {};
		\node [style=none] (18) at (3.75, 0.5) {};
		\node [style=none] (19) at (-2.75, 0) {$N$};
		\node [style=none] (20) at (-0.25, 0) {$N$};
		\node [style=none] (21) at (0.5, 0) {$S$};
		\node [style=none] (22) at (1.25, 0) {$N$};
		\node [style=none] (23) at (3.75, 0) {$N$};
		\node [style=none] (24) at (-2.75, -0.5) {};
		\node [style=none] (25) at (-0.25, -0.5) {};
		\node [style=none] (26) at (0.5, -0.5) {};
		\node [style=none] (27) at (1.25, -0.5) {};
		\node [style=none] (28) at (3.75, -0.5) {};
		\node [style=none] (29) at (0.5, -1.75) {};
		\node [style=none] (30) at (-2.75, 2.75) {men};
		\node [style=none] (31) at (0.5, 2.75) {like};
		\node [style=none] (32) at (3.75, 2.75) {cats};
	\end{pgfonlayer}
	\begin{pgfonlayer}{edgelayer}
		\draw [style = thick] (0.center) to (1.center);
		\draw [style = thick] (1.center) to (2.center);
		\draw [style = thick] (2.center) to (0.center);
		\draw [style = thick]  (3.center) to (4.center);
		\draw [style = thick] (8.center) to (6.center);
		\draw [style = thick] (6.center) to (7.center);
		\draw [style = thick] (4.center) to (5.center);
		\draw [style = thick] (3.center) to (5.center);
		\draw [style = thick] (8.center) to (7.center);
		\draw [style = thick] (9.center) to (14.center);
		\draw [style = thick] (10.center) to (15.center);
		\draw [style = thick] (11.center) to (16.center);
		\draw [style = thick] (12.center) to (17.center);
		\draw [style = thick] (13.center) to (18.center);
		\draw [thick, bend right=90, looseness=1.25] (24.center) to (25.center);
		\draw [thick, bend right=90, looseness=1.25] (27.center) to (28.center);
		\draw [style = thick] (26.center) to (29.center);
	\end{pgfonlayer}
\end{tikzpicture}}
\endpgfgraphicnamed}
\end{center}

\noindent 
As you can see, depending on their types, the
distributional meanings of the words are either atomic vectors or
linear maps. For instance, the distributional meaning of `men' is a
vector in $N$, whereas the distributional meaning of `love' is in the
space $N \otimes S \otimes N$, which is equivalent to a linear map from 
$N \otimes N$ to $S$.

\section{Modelling  Possessive Relative Pronouns}  
\label{poss}

In this section,  we use the categorical constructions introduced in the previous sections and present a model for the possessive relative pronoun in its subject and object roles. We start with the pregroup representations of the grammatical types of these pronouns and show how a relative clause containing them reduces.  We then develop  a categorical semantics for these using the cancelation and generation maps of compact closed categories and the morphisms of Frobenius algebras. In a nutshell, the cancelation map $\epsilon$ models the  application of the semantics of one word to
another; the generation map $\eta$, passes information among the words by bridging the 
intermediate words; and the Frobenius operations dispatch and combine  the noun vectors and discard the sentence vectors. The end product of this model is a compositional vector representation for the meanings of possessive relative clauses.

The possessive relative pronoun  `whose' can occur in a subject or  object position, hence producing subject and object possessive clauses. The general forms of these clauses are as follows:
\begin{quote}
{\bf Poss Subj:} \quad Possessor whose Subject Verb Object\\
{\bf Poss Obj:}\quad \ \ Possessor whose Object Subject Verb
\end{quote}
An example  for each of the above cases is  as follows:
\begin{quote}
{\bf Poss Subj:} \quad  `author whose book entertained John'\\
{\bf Poss Obj:} \quad  \ \ `author whose book John read'
\end{quote}

Informally speaking, the head of such a clause  is a possessor who owns an item, which is either the subject or the object of the rest of the clause.  For instance,  in the subject example above,  the possessor  head noun `author' owns a book which has entertained John and in the object case, John has read this book.  In a manner of speaking, we are  describing the   head of the clause   through his possession. For instance, in the above clauses we are describing an `author'  through his `book'.

The syntactic role of the relative pronoun `whose' is to relate the head of the clause  to the rest of the clause by first inputting the modified subject or object of the rest of the clause, then inputting the `owner' of this noun, and finally letting the head of the clause be modified by  these owners and outputting their modified versions.  For instance, in the subject example above, `whose' first inputs the subject of  `book entertained John', that is, books which have been modified by the verb phrase `entertained John'. Then it inputs the owners of these books, and finally modifies `author' by  them, outputting the authors whose books entertained John.  Set theoretically speaking,   `whose' is to choose  from the set of authors, the ones that have books which have entertained John, in the first example, and the ones that  have  books which John read, in the second example.

The pregroup  types of these pronouns are as follows:
\[
\mbox{\bf Poss Subj:}\quad n^r n s^l n n^l
\hspace{2cm}
\mbox{\bf Poss Obj:}\quad n^r n n^{ll} s^l n^l
\]
The grammatical reduction of the subject case is as follows:
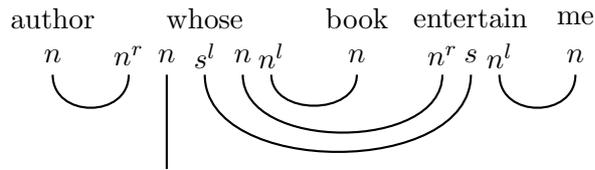
\begin{center}
 {%
\beginpgfgraphicnamed{whose-subj-preg}
\begin{tikzpicture}
	\begin{pgfonlayer}{nodelayer}
		\node [style=none] (0) at (-1, 0) {};
		\node [style=none] (1) at (6.25, 0.5) {$n^r$};
		\node [style=none] (2) at (7, 0.5) {$s$};
		\node [style=none] (3) at (7.75, 0.5) {$n^l$};
		\node [style=none] (4) at (4, 0.5) {$n$};
		\node [style=none] (5) at (-4, 0.5) {$n$};
		\node [style=none] (6) at (-4, 0) {};
		\node [style=none] (7) at (-2, 0) {};
		\node [style=none] (8) at (-2, 0.5) {$n^r$};
		\node [style=none] (9) at (-1, 0.5) {$n$};
		\node [style=none] (10) at (0, 0.5) {$s^l$};
		\node [style=none] (11) at (1, 0.5) {$n$};
		\node [style=none] (12) at (0, 0) {};
		\node [style=none] (13) at (1, 0) {};
		\node [style=none] (14) at (4, 0) {};
		\node [style=none] (15) at (-1, -2.5) {};
		\node [style=none] (16) at (6.25, 0) {};
		\node [style=none] (17) at (7.75, 0) {};
		\node [style=none] (18) at (7, 0) {};
		\node [style=none] (19) at (-4, 1.5) {author};
		\node [style=none] (20) at (0, 1.5) {whose};
		\node [style=none] (21) at (4, 1.5) {book};
		\node [style=none] (22) at (7, 1.5) {entertain};
		\node [style=none] (23) at (9.75, 1.5) {me};
		\node [style=none] (24) at (9.75, 0.5) {$n$};
		\node [style=none] (25) at (9.75, 0) {};
		\node [style=none] (26) at (1.75, 0.5) {$n^l$};
		\node [style=none] (27) at (1.75, 0) {};
	\end{pgfonlayer}
	\begin{pgfonlayer}{edgelayer}
		\draw [style=thick, bend left=270, looseness=1.50] (6.center) to (7.center);
		\draw [style=thick] (0.center) to (15.center);
		\draw [style=thick, bend right=90] (12.center) to (18.center);
		\draw [style=thick, bend left=90, looseness=1.25] (14.center) to (27.center);
		\draw [style=thick, bend right=90] (13.center) to (16.center);
		\draw [style=thick, bend right=90, looseness=1.50] (17.center) to (25.center);
	\end{pgfonlayer}
\end{tikzpicture}}
\endpgfgraphicnamed}
  \end{center}
  \noindent The grammatical reduction of the object  case is as follows:
  \begin{center}
  {%
\beginpgfgraphicnamed{whose-obj-preg}
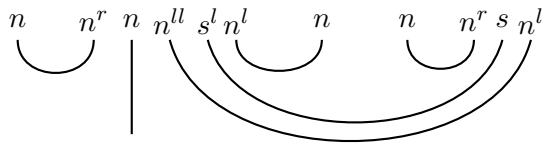
\begin{tikzpicture}
	\begin{pgfonlayer}{nodelayer}
		\node [style=none] (0) at (-1, 0) {};
		\node [style=none] (1) at (8, 0.5) {$n^r$};
		\node [style=none] (2) at (8.75, 0.5) {$s$};
		\node [style=none] (3) at (9.5, 0.5) {$n^l$};
		\node [style=none] (4) at (4, 0.5) {$n$};
		\node [style=none] (5) at (-4, 0.5) {$n$};
		\node [style=none] (6) at (-4, 0) {};
		\node [style=none] (7) at (-2, 0) {};
		\node [style=none] (8) at (-2, 0.5) {$n^r$};
		\node [style=none] (9) at (-1, 0.5) {$n$};
		\node [style=none] (10) at (0, 0.5) {$n^{ll}$};
		\node [style=none] (11) at (1, 0.5) {$s^l$};
		\node [style=none] (12) at (0, 0) {};
		\node [style=none] (13) at (1, 0) {};
		\node [style=none] (14) at (4, 0) {};
		\node [style=none] (15) at (-1, -2.5) {};
		\node [style=none] (16) at (8, 0) {};
		\node [style=none] (17) at (9.5, 0) {};
		\node [style=none] (18) at (8.75, 0) {};
		\node [style=none] (19) at (6.25, 0.5) {$n$};
		\node [style=none] (20) at (1.75, 0.5) {$n^l$};
		\node [style=none] (21) at (1.75, 0) {};
		\node [style=none] (22) at (6.25, 0) {};
	\end{pgfonlayer}
	\begin{pgfonlayer}{edgelayer}
		\draw [style=thick, bend left=270, looseness=1.50] (6.center) to (7.center);
		\draw [style=thick] (0.center) to (15.center);
		\draw [style=thick, bend left=90, looseness=1.25] (14.center) to (21.center);
		\draw [style=thick, bend right=90, looseness=1.50] (22.center) to (16.center);
		\draw [style=thick, bend right=75] (13.center) to (18.center);
		\draw [style=thick, bend right=75] (12.center) to (17.center);
	\end{pgfonlayer}
\end{tikzpicture}}
\endpgfgraphicnamed}
\end{center}

\noindent The vector  spaces in which the meaning vectors  of  `whose' live, are obtained as follows: 
\begin{eqnarray*}
F(n^r n s^l n n^l) &=& N \otimes N \otimes S \otimes N \otimes N\\
F(n^r n n^{ll} s^l n^l) &=& N \otimes N \otimes N \otimes S \otimes N
\end{eqnarray*}
Other than passing the information around, these pronouns  also  act on the subject/object of the clause to establish the ownership relationship between them and the possessor. We denote this action by a `'s'-labeled  box with the type $N \to N$; it takes the subject/object as input then outputs its owner. The remaining structure  of the pronoun duplicates the information of the  subject/object and passes a copy to the verb, then unifies  the   possessor with the owners of  the modified subject/object.   These processes are depicted below:

\bigskip
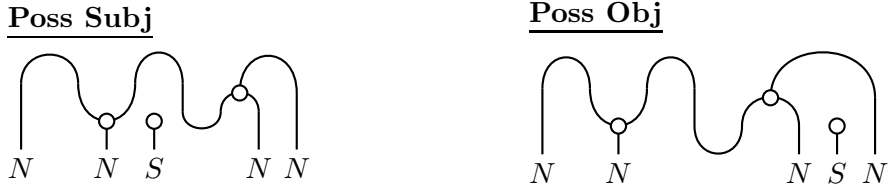
\begin{minipage}{6cm}
\underline{\bf Poss Subj}

  {%
\beginpgfgraphicnamed{whose-subj-frob}
\begin{tikzpicture}
	\begin{pgfonlayer}{nodelayer}
		\node [style=none] (0) at (-2.75, 0.25) {$N$};
		\node [style=none] (1) at (-0.5, 0.25) {$N$};
		\node [style=none] (2) at (0.75, 0.25) {$S$};
		\node [style=none] (3) at (3.5, 0.25) {$N$};
		\node [style=none] (4) at (-1.25, 2.5) {};
		\node [style=none] (5) at (-2.75, 0.75) {};
		\node [style=none] (6) at (-0.5, 0.75) {};
		\node [fill=white, draw, thick, circle, minimum size=0.2 cm, style=none] (7) at (0.75, 1.5) {};
		\node [style=none] (8) at (-2.75, 2.5) {};
		\node [style=none] (9) at (-0.5, 1.5) {};
		\node [style=none] (10) at (0.25, 2.5) {};
		\node [style=none] (11) at (3.5, 0.75) {};
		\node [style=none] (12) at (1.5, 2.5) {};
		\node [style=none] (13) at (0.75, 0.75) {};
		\node [fill=white, draw, thick, circle, minimum size=0.2 cm, style=none] (14) at (-0.5, 1.5) {};
		\node [style=none] (15) at (4.5, 0.75) {};
		\node [style=none] (16) at (4.5, 0.25) {$N$};
		\node [style=none] (17) at (4.5, 2.25) {};
		\node [style=none] (18) at (3, 2.25) {};
		\node [style=none] (19) at (2.5, 1.75) {};
		\node [style=none] (20) at (3.5, 1.75) {};
		\node [style=none] (21) at (1.5, 1.75) {};
		\node [style=none] (22) at (0.75, 1.5) {};
		\node [style=none] (23) at (-0.5, 1.5) {};
		\node [fill=white, draw, thick, circle, minimum size=0.2 cm, style=none] (24) at (3, 2.25) {};
	\end{pgfonlayer}
	\begin{pgfonlayer}{edgelayer}
		\draw [style=thick] (7.center) to (13.center);
		\draw [style=thick] (8.center) to (5.center);
		\draw [style=thick] (9.center) to (6.center);
		\draw [style=thick, bend left=90, looseness=2.25] (18.center) to (17.center);
		\draw [style=thick] (17.center) to (15.center);
		\draw [style=thick, bend left=90, looseness=1.75] (19.center) to (20.center);
		\draw [style=thick] (20.center) to (11.center);
		\draw [style=thick] (12.center) to (21.center);
		\draw [style=thick, bend left=90, looseness=1.50] (19.center) to (21.center);
		\draw [style=thick, bend left=90, looseness=1.75] (8.center) to (4.center);
		\draw [style=thick, bend right=90, looseness=2.25] (12.center) to (10.center);
		\draw [style=thick, bend right=90, looseness=2.25] (4.center) to (10.center);
	\end{pgfonlayer}
\end{tikzpicture}}
\endpgfgraphicnamed}
  \end{minipage}
\qquad
\begin{minipage}{6cm}
\underline{\bf Poss Obj}

    {%
\beginpgfgraphicnamed{whose-obj-frob}
\begin{tikzpicture}
	\begin{pgfonlayer}{nodelayer}
		\node [style=none] (0) at (-2.5, 0.25) {$N$};
		\node [style=none] (1) at (-0.5, 0.25) {$N$};
		\node [style=none] (2) at (5.25, 0.25) {$S$};
		\node [style=none] (3) at (4.25, 0.25) {$N$};
		\node [style=none] (4) at (-1.25, 2.5) {};
		\node [style=none] (5) at (-2.5, 0.75) {};
		\node [style=none] (6) at (-0.5, 0.75) {};
		\node [fill=white, draw, thick, circle, minimum size=0.2 cm, style=none] (7) at (5.25, 1.5) {};
		\node [style=none] (8) at (-2.5, 2.5) {};
		\node [style=none] (9) at (-0.5, 1.5) {};
		\node [style=none] (10) at (0.25, 2.5) {};
		\node [style=none] (11) at (4.25, 0.75) {};
		\node [style=none] (12) at (1.5, 2.5) {};
		\node [style=none] (13) at (5.25, 0.75) {};
		\node [fill=white, draw, thick, circle, minimum size=0.2 cm, style=none] (14) at (-0.5, 1.5) {};
		\node [style=none] (15) at (6.25, 0.75) {};
		\node [style=none] (16) at (6.25, 0.25) {$N$};
		\node [style=none] (17) at (6.25, 2.25) {};
		\node [style=none] (18) at (3.5, 2.25) {};
		\node [style=none] (19) at (2.75, 1.5) {};
		\node [style=none] (20) at (4.25, 1.5) {};
		\node [style=none] (21) at (1.5, 1.5) {};
		\node [style=none] (22) at (5.25, 1.5) {};
		\node [style=none] (23) at (-0.5, 1.5) {};
		\node [fill=white, draw, thick, circle, minimum size=0.2 cm, style=none] (24) at (3.5, 2.25) {};
	\end{pgfonlayer}
	\begin{pgfonlayer}{edgelayer}
		\draw [style=thick] (7.center) to (13.center);
		\draw [style=thick] (8.center) to (5.center);
		\draw [style=thick] (9.center) to (6.center);
		\draw [style=thick, bend left=90, looseness=1.50] (18.center) to (17.center);
		\draw [style=thick] (17.center) to (15.center);
		\draw [style=thick, bend left=90, looseness=1.75] (19.center) to (20.center);
		\draw [style=thick] (20.center) to (11.center);
		\draw [style=thick] (12.center) to (21.center);
		\draw [style=thick, bend left=90, looseness=2.25] (8.center) to (4.center);
		\draw [style=thick, bend right=90, looseness=2.25] (12.center) to (10.center);
		\draw [style=thick, bend right=90, looseness=2.25] (4.center) to (10.center);
		\draw [style=thick, bend right=90, looseness=2.00] (21.center) to (19.center);
	\end{pgfonlayer}
\end{tikzpicture}}
\endpgfgraphicnamed}
\end{minipage}

\bigskip
\noindent The above diagrams  correspond to the following morphisms:
\begin{eqnarray*}
&&(1_N \otimes \mu_N \otimes \zeta_S \otimes \epsilon_N \otimes \Delta_N \otimes 1_N)
\circ(1_{N \otimes N} \otimes \overline{\text{'s}} \otimes 1_{N \otimes  N \otimes N})
\circ(\eta_N \otimes \eta_N \otimes \eta_N)\\
&&\left(1_N \otimes \mu_N \otimes \epsilon_N \otimes  \Delta_N \otimes \zeta_S \otimes 1_N \right)
\circ
\left(1_{N \otimes N} \otimes \overline{\text{'s}} \otimes 1_{N \otimes N \otimes N}\right)
\circ
\left(\eta_N \otimes \eta_N \otimes \eta_N\right)
\end{eqnarray*}
\noindent  The   diagrams of the meanings of these clauses visualise the above processes in the clause. For instance,  consider the diagram of the subject case:
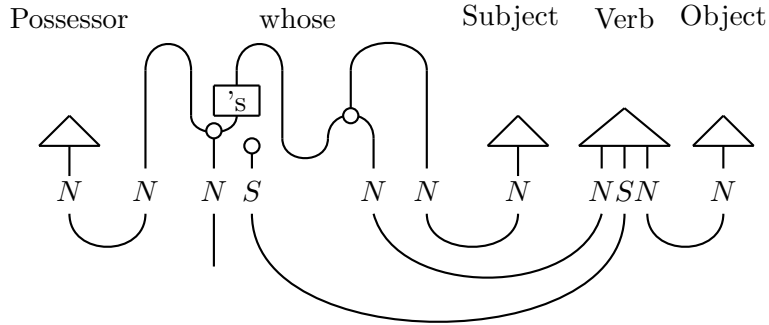
\begin{center}
  {%
\beginpgfgraphicnamed{whose-subj-sem-TRIANGLE}
\begin{tikzpicture}[scale=0.8]
	\begin{pgfonlayer}{nodelayer}
		\node [style=none] (0) at (8, -0.5) {$N$};
		\node [style=none] (1) at (8.75, -0.5) {$S$};
		\node [style=none] (2) at (9.5, -0.5) {$N$};
		\node [style=none] (3) at (-9.5, -0.5) {$N$};
		\node [style=none] (4) at (-9.5, -1.25) {};
		\node [style=none] (5) at (8, -1.25) {};
		\node [style=none] (6) at (9.5, -1.25) {};
		\node [style=none] (7) at (12, -0.5) {$N$};
		\node [style=none] (8) at (12, -1.25) {};
		\node [style=none] (9) at (-9.5, 0) {};
		\node [style=none] (10) at (-9.5, 1) {};
		\node [style=none] (11) at (-8.5, 1) {};
		\node [style=none] (12) at (-9.5, 2) {};
		\node [style=none] (13) at (-10.5, 1) {};
		\node [style=none] (14) at (12, 1) {};
		\node [style=none] (15) at (12, 0.25) {};
		\node [style=none] (16) at (12, 2) {};
		\node [style=none] (17) at (13, 1) {};
		\node [style=none] (18) at (11, 1) {};
		\node [style=none] (19) at (8.75, 2.25) {};
		\node [style=none] (20) at (8, 1) {};
		\node [style=none] (21) at (8.75, 0.25) {};
		\node [style=none] (22) at (8, 0.25) {};
		\node [style=none] (23) at (7.25, 1) {};
		\node [style=none] (24) at (10.25, 1) {};
		\node [style=none] (25) at (8.75, 1) {};
		\node [style=none] (26) at (9.5, 1) {};
		\node [style=none] (27) at (9.5, 0.25) {};
		\node [style=none] (28) at (-9.5, 5.25) {Possessor};
		\node [style=none] (29) at (8.75, 5.25) {Verb};
		\node [style=none] (30) at (12, 5.25) {Object};
		\node [style=none] (31) at (6.25, 1) {};
		\node [style=none] (32) at (5.25, 0) {};
		\node [style=none] (33) at (4.25, 1) {};
		\node [style=none] (34) at (5.25, 2) {};
		\node [style=none] (35) at (5.25, -0.5) {$N$};
		\node [style=none] (36) at (5.25, 1) {};
		\node [style=none] (37) at (5, 5.25) {Subject};
		\node [style=none] (38) at (5.25, -1.25) {};
		\node [style=none] (39) at (2.25, -1.25) {};
		\node [style=none] (40) at (-7, -1.25) {};
		\node [style=none] (41) at (-4.75, -3) {};
		\node [style=none] (42) at (-4.75, -1.25) {};
		\node [style=none] (43) at (-3.5, -1.25) {};
		\node [style=none] (44) at (0.5, -1.25) {};
		\node [style=none] (45) at (8.75, -1.25) {};
		\node [style=none] (46) at (-2.5, 3.5) {};
		\node [style=none] (47) at (-7, 0.25) {};
		\node [style=none] (48) at (-5.5, 2) {};
		\node [style=none] (49) at (-4.75, 1.5) {};
		\node [fill=white, draw, thick, circle, minimum size=0.2 cm, style=none] (50) at (-4.75, 1.5) {};
		\node [style=none] (51) at (-4.75, 1.5) {};
		\node [style=none] (52) at (-4.75, 0.25) {};
		\node [style=none] (53) at (-7, -0.5) {$N$};
		\node [fill=white, draw, thick, circle, minimum size=0.2 cm, style=none] (54) at (-0.25, 2) {};
		\node [style=none] (55) at (-7, 3.5) {};
		\node [style=none] (56) at (-2.5, 1.25) {};
		\node [fill=white, draw, thick, circle, minimum size=0.2 cm, style=none] (57) at (-3.5, 1) {};
		\node [style=none] (58) at (0.5, 0.25) {};
		\node [style=none] (59) at (0.5, -0.5) {$N$};
		\node [style=none] (60) at (-3.5, -0.5) {$S$};
		\node [style=none] (61) at (-3.25, 2) {};
		\node [style=none] (62) at (2.25, -0.5) {$N$};
		\node [style=none] (63) at (-3.25, 3) {};
		\node [style=none] (64) at (0.5, 1.25) {};
		\node [style=none] (65) at (-4, 2) {};
		\node [style=none] (66) at (-4, 2) {};
		\node [style=none] (67) at (-4.75, -0.5) {$N$};
		\node [style=none] (68) at (-1, 1.25) {};
		\node [style=none] (69) at (2.25, 0.25) {};
		\node [style=none] (70) at (-4.75, 2) {};
		\node [style=none] (71) at (-3.5, 1) {};
		\node [style=none] (72) at (2.25, 3.5) {};
		\node [style=none] (73) at (-3.5, 0.25) {};
		\node [style=none] (74) at (-4.75, 3) {};
		\node [style=none] (75) at (-4, 2.5) {'s};
		\node [style=none] (76) at (-2, 5.25) {whose};
		\node [style=none] (77) at (-4, 3.5) {};
		\node [style=none] (78) at (-5.5, 3.5) {};
		\node [style=none] (79) at (-0.25, 3.5) {};
		\node [style=none] (80) at (-4, 3) {};
	\end{pgfonlayer}
	\begin{pgfonlayer}{edgelayer}
		\draw [style=thick, bend right=90, looseness=1.50] (6.center) to (8.center);
		\draw [style=thick] (13.center) to (11.center);
		\draw [style=thick] (11.center) to (12.center);
		\draw [style=thick] (12.center) to (13.center);
		\draw [style=thick] (10.center) to (9.center);
		\draw [style=thick] (18.center) to (17.center);
		\draw [style=thick] (17.center) to (16.center);
		\draw [style=thick] (16.center) to (18.center);
		\draw [style=thick] (14.center) to (15.center);
		\draw [style=thick] (23.center) to (24.center);
		\draw [style=thick] (19.center) to (23.center);
		\draw [style=thick] (19.center) to (24.center);
		\draw [style=thick] (20.center) to (22.center);
		\draw [style=thick] (25.center) to (21.center);
		\draw [style=thick] (26.center) to (27.center);
		\draw [style=thick] (33.center) to (31.center);
		\draw [style=thick] (31.center) to (34.center);
		\draw [style=thick] (34.center) to (33.center);
		\draw [style=thick] (36.center) to (32.center);
		\draw [style=thick, bend right=90, looseness=1.25] (39.center) to (38.center);
		\draw [style=thick, bend right=90, looseness=1.50] (4.center) to (40.center);
		\draw [style=thick] (42.center) to (41.center);
		\draw [style=thick, bend right=75] (44.center) to (5.center);
		\draw [style=thick, bend right=90] (43.center) to (45.center);
		\draw [style=thick] (57.center) to (73.center);
		\draw [style=thick] (55.center) to (47.center);
		\draw [style=thick] (49.center) to (52.center);
		\draw [style=thick] (72.center) to (69.center);
		\draw [style=thick, bend left=90, looseness=1.75] (68.center) to (64.center);
		\draw [style=thick] (64.center) to (58.center);
		\draw [style=thick] (46.center) to (56.center);
		\draw [style=thick, bend right=270, looseness=1.50] (68.center) to (56.center);
		\draw [style=thick, bend right=90, looseness=1.25] (48.center) to (65.center);
		\draw [style=thick] (74.center) to (63.center);
		\draw [style=thick] (63.center) to (61.center);
		\draw [style=thick] (61.center) to (70.center);
		\draw [style=thick] (70.center) to (74.center);
		\draw [thick, bend left=90, looseness=2.00] (55.center) to (78.center);
		\draw [thick, bend left=90, looseness=2.00] (77.center) to (46.center);
		\draw [style=thick] (79.center) to (54.center);
		\draw [thick, bend left=90, looseness=1.25] (79.center) to (72.center);
		\draw [style=thick](78.center) to (48.center);
		\draw [style=thick](77.center) to (80.center);
	\end{pgfonlayer}
\end{tikzpicture}}
\endpgfgraphicnamed}
\end{center}
The pronoun `whose'  inputs the information of the subject and outputs its owners after applying the \fbox{'s} to it. This information is unified with the possessor via a $\mu$ map, then a copy of it is passed to the verb via a $\Delta$ map and outputted after the verb has acted on it. The flow of information happens via the three $\eta$ maps. A similar process takes place in the object case, as depicted below: 

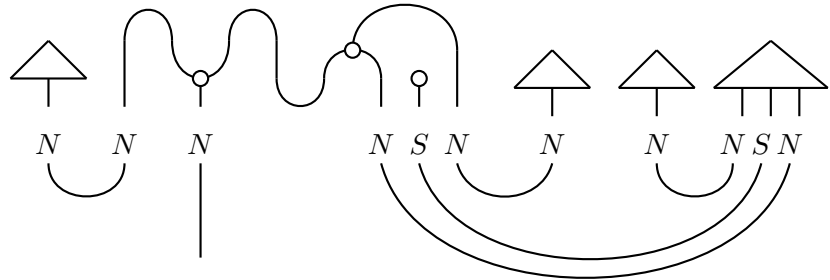
\begin{center}  
   {%
\beginpgfgraphicnamed{whose-obj-sem}
\begin{tikzpicture}
	\begin{pgfonlayer}{nodelayer}
		\node [style=none] (0) at (-5.5, 0) {};
		\node [style=none] (1) at (8.5, 0.5) {$N$};
		\node [style=none] (2) at (9.25, 0.5) {$S$};
		\node [style=none] (3) at (10, 0.5) {$N$};
		\node [style=none] (4) at (-9.5, 0.5) {$N$};
		\node [style=none] (5) at (-9.5, 0) {};
		\node [style=none] (6) at (-7.5, 0) {};
		\node [style=none] (7) at (-0.75, 0) {};
		\node [style=none] (8) at (0.25, 0) {};
		\node [style=none] (9) at (1.25, 0) {};
		\node [style=none] (10) at (-5.5, -2.5) {};
		\node [style=none] (11) at (8.5, 0) {};
		\node [style=none] (12) at (10, 0) {};
		\node [style=none] (13) at (9.25, 0) {};
		\node [style=none] (14) at (6.5, 0.5) {$N$};
		\node [style=none] (15) at (6.5, 0) {};
		\node [style=none] (16) at (-0.75, 2.25) {};
		\node [fill=white, draw, thick, circle, minimum size=0.2 cm, style=none] (17) at (0.25, 2.25) {};
		\node [style=none] (18) at (-1.5, 3) {};
		\node [style=none] (19) at (-5.5, 0.5) {$N$};
		\node [style=none] (20) at (-7.5, 3.25) {};
		\node [style=none] (21) at (-0.75, 1.5) {};
		\node [style=none] (22) at (0.25, 0.5) {$S$};
		\node [style=none] (23) at (-0.75, 0.5) {$N$};
		\node [style=none] (24) at (1.25, 3) {};
		\node [style=none] (25) at (-3.5, 2.25) {};
		\node [fill=white, draw, thick, circle, minimum size=0.2 cm, style=none] (26) at (-1.5, 3) {};
		\node [style=none] (27) at (-5.5, 2.25) {};
		\node [style=none] (28) at (1.25, 0.5) {$N$};
		\node [fill=white, draw, thick, circle, minimum size=0.2 cm, style=none] (29) at (-5.5, 2.25) {};
		\node [style=none] (30) at (-5.5, 1.5) {};
		\node [style=none] (31) at (0.25, 2.25) {};
		\node [style=none] (32) at (-5.5, 2.25) {};
		\node [style=none] (33) at (-4.75, 3.25) {};
		\node [style=none] (34) at (1.25, 1.5) {};
		\node [style=none] (35) at (-3.5, 3.25) {};
		\node [style=none] (36) at (-6.25, 3.25) {};
		\node [style=none] (37) at (-7.5, 0.5) {$N$};
		\node [style=none] (38) at (-2.25, 2.25) {};
		\node [style=none] (39) at (0.25, 1.5) {};
		\node [style=none] (40) at (-7.5, 1.5) {};
		\node [style=none] (41) at (8.75, 1.25) {};
		\node [style=none] (42) at (11, 2) {};
		\node [style=none] (43) at (8.75, 2) {};
		\node [style=none] (44) at (10.25, 1.25) {};
		\node [style=none] (45) at (9.5, 3.25) {};
		\node [style=none] (46) at (9.5, 2) {};
		\node [style=none] (47) at (-8.5, 2.25) {};
		\node [style=none] (48) at (10.25, 2) {};
		\node [style=none] (49) at (-10.5, 2.25) {};
		\node [style=none] (50) at (-9.5, 3.25) {};
		\node [style=none] (51) at (9.5, 1.25) {};
		\node [style=none] (52) at (-9.5, 1.5) {};
		\node [style=none] (53) at (8, 2) {};
		\node [style=none] (54) at (-9.5, 2.25) {};
		\node [style=none] (55) at (6.5, 2) {};
		\node [style=none] (56) at (7.5, 2) {};
		\node [style=none] (57) at (6.5, 1.25) {};
		\node [style=none] (58) at (6.5, 3) {};
		\node [style=none] (59) at (5.5, 2) {};
		\node [style=none] (60) at (3.75, 0.5) {$N$};
		\node [style=none] (61) at (4.75, 2) {};
		\node [style=none] (62) at (2.75, 2) {};
		\node [style=none] (63) at (3.75, 3) {};
		\node [style=none] (64) at (3.75, 1.25) {};
		\node [style=none] (65) at (3.75, 2) {};
		\node [style=none] (66) at (3.75, 0) {};
	\end{pgfonlayer}
	\begin{pgfonlayer}{edgelayer}
		\draw [style=thick, bend left=270, looseness=1.50] (5.center) to (6.center);
		\draw [style=thick] (0.center) to (10.center);
		\draw [style=thick, bend right=90, looseness=1.50] (15.center) to (11.center);
		\draw [style=thick, bend right=75] (8.center) to (13.center);
		\draw [style=thick, bend right=75] (7.center) to (12.center);
		\draw [style=thick] (17.center) to (39.center);
		\draw [style=thick] (20.center) to (40.center);
		\draw [style=thick] (27.center) to (30.center);
		\draw [style=thick, bend left=90, looseness=1.50] (18.center) to (24.center);
		\draw [style=thick] (24.center) to (34.center);
		\draw [style=thick, bend left=90, looseness=1.75] (38.center) to (16.center);
		\draw [style=thick] (16.center) to (21.center);
		\draw [style=thick] (35.center) to (25.center);
		\draw [style=thick, bend left=90, looseness=2.25] (20.center) to (36.center);
		\draw [style=thick, bend right=90, looseness=2.25] (35.center) to (33.center);
		\draw [style=thick, bend right=90, looseness=2.25] (36.center) to (33.center);
		\draw [style=thick, bend right=90, looseness=2.00] (25.center) to (38.center);
		\draw [style=thick] (53.center) to (42.center);
		\draw [style=thick] (45.center) to (53.center);
		\draw [style=thick] (45.center) to (42.center);
		\draw [style=thick] (43.center) to (41.center);
		\draw [style=thick] (46.center) to (51.center);
		\draw [style=thick] (48.center) to (44.center);
		\draw [style=thick] (49.center) to (47.center);
		\draw [style=thick] (47.center) to (50.center);
		\draw [style=thick] (50.center) to (49.center);
		\draw [style=thick] (54.center) to (52.center);
		\draw [style=thick] (59.center) to (56.center);
		\draw [style=thick] (56.center) to (58.center);
		\draw [style=thick] (58.center) to (59.center);
		\draw [style=thick] (55.center) to (57.center);
		\draw [style=thick] (62.center) to (61.center);
		\draw [style=thick] (61.center) to (63.center);
		\draw [style=thick] (63.center) to (62.center);
		\draw [style=thick] (65.center) to (64.center);
		\draw [style=thick, bend right=75, looseness=1.25] (9.center) to (66.center);
	\end{pgfonlayer}
\end{tikzpicture}}
\endpgfgraphicnamed}
\end{center}

\noindent The  actions of the above processes are summarised in their normal forms. Consider the case of the subject clause, normalised below: 
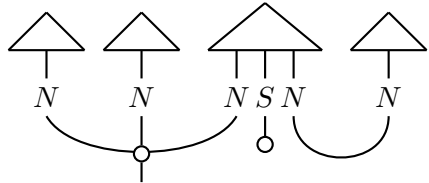
\begin{center}
\begin{minipage}{7cm}
{%
\beginpgfgraphicnamed{whose-subj-norm}
\begin{tikzpicture}
	\begin{pgfonlayer}{nodelayer}
		\node [style=none] (0) at (3.75, 0.5) {$N$};
		\node [style=none] (1) at (4.5, 0.5) {$S$};
		\node [style=none] (2) at (5.25, 0.5) {$N$};
		\node [style=none] (3) at (1.25, 0.5) {$N$};
		\node [style=none] (4) at (-1.25, 0.5) {$N$};
		\node [style=none] (5) at (-1.25, 0) {};
		\node [style=none] (6) at (1.25, 0) {};
		\node [style=none] (7) at (3.75, 0) {};
		\node [style=none] (8) at (5.25, 0) {};
		\node [style=none] (9) at (7.75, 0.5) {$N$};
		\node [style=none] (10) at (7.75, 0) {};
		\node [style=none] (11) at (4.5, -0.75) {};
		\node [style=none] (12) at (-1.25, 1) {};
		\node [style=none] (13) at (-1.25, 1.75) {};
		\node [style=none] (14) at (-0.25, 1.75) {};
		\node [style=none] (15) at (-1.25, 2.75) {};
		\node [style=none] (16) at (-2.25, 1.75) {};
		\node [style=none] (17) at (7.75, 1.75) {};
		\node [style=none] (18) at (7.75, 1) {};
		\node [style=none] (19) at (7.75, 2.75) {};
		\node [style=none] (20) at (8.75, 1.75) {};
		\node [style=none] (21) at (6.75, 1.75) {};
		\node [style=none] (22) at (4.5, 3) {};
		\node [style=none] (23) at (3.75, 1.75) {};
		\node [style=none] (24) at (4.5, 1) {};
		\node [style=none] (25) at (3.75, 1) {};
		\node [style=none] (26) at (3, 1.75) {};
		\node [style=none] (27) at (6, 1.75) {};
		\node [style=none] (28) at (4.5, 1.75) {};
		\node [style=none] (29) at (5.25, 1.75) {};
		\node [style=none] (30) at (5.25, 1) {};
		\node [style=plain] (31) at (7.75, 4) {};
		\node [style=plain] (32) at (7.75, -2.5) {};
		\node [style=none] (33) at (1.25, 1) {};
		\node [style=none] (34) at (1.25, 1.75) {};
		\node [style=none] (35) at (2.25, 1.75) {};
		\node [style=none] (36) at (0.25, 1.75) {};
		\node [style=none] (37) at (1.25, 2.75) {};
		\node [fill=white, draw, thick, circle, minimum size=0.2 cm, style=none] (38) at (1.25, -1) {};
		\node [style=none] (39) at (1.25, -1.75) {};
		\node [style=none] (40) at (4.5, 0) {};
		\node [fill=white, draw, thick, circle, minimum size=0.2 cm, style=none] (41) at (4.5, -0.75) {};
	\end{pgfonlayer}
	\begin{pgfonlayer}{edgelayer}
		\draw [style=thick, bend right=90, looseness=1.50] (8.center) to (10.center);
		\draw [style=thick] (16.center) to (14.center);
		\draw [style=thick] (14.center) to (15.center);
		\draw [style=thick] (15.center) to (16.center);
		\draw [style=thick] (13.center) to (12.center);
		\draw [style=thick] (21.center) to (20.center);
		\draw [style=thick] (20.center) to (19.center);
		\draw [style=thick] (19.center) to (21.center);
		\draw [style=thick] (17.center) to (18.center);
		\draw [style=thick] (26.center) to (27.center);
		\draw [style=thick] (22.center) to (26.center);
		\draw [style=thick] (22.center) to (27.center);
		\draw [style=thick] (23.center) to (25.center);
		\draw [style=thick] (28.center) to (24.center);
		\draw [style=thick] (29.center) to (30.center);
		\draw [style=thick] (36.center) to (35.center);
		\draw [style=thick] (35.center) to (37.center);
		\draw [style=thick] (37.center) to (36.center);
		\draw [style=thick] (34.center) to (33.center);
		\draw [style=thick, bend right=60, looseness=0.75] (5.center) to (7.center);
		\draw [style=thick] (6.center) to (38);
		\draw [style=thick] (38) to (39.center);
		\draw [style=thick] (40.center) to (11.center);
	\end{pgfonlayer}
\end{tikzpicture}}
\endpgfgraphicnamed}
\end{minipage}
\end{center}
Here, the $\eta$ and $\epsilon$ maps are yanked and the result is displayed in a more clear way: the verb acts on the subject and object, but does not return a sentence, as usual, since its sentence dimension is discarded by `whose'. Instead, the information of the subject, after the verb has acted on it, is inputted to \fbox{'s} then unified  with the information of the possessor. The meaning of the clause is the result of this unification. The above diagram corresponds to the following categorical morphisms:
\begin{eqnarray*}
&&\mu_N  \circ \left(1_N \otimes \overline{\text{'s}} \right)\circ \left(1_N \otimes \mu_N \otimes \iota_S \otimes \epsilon_N\right) \left(\ov{\text{Poss}} \otimes \ov{\text{Sub}} \otimes \ov{\text{Verb}} \otimes \ov{\text{Obj}}\right) 
\end{eqnarray*}
The normal form of the object clause, describing a similar process, is as follows:

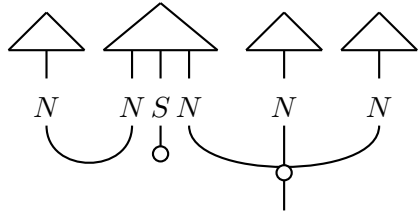
\begin{center}
\begin{minipage}{7cm}
{%
\beginpgfgraphicnamed{whose-obj-norm}
\begin{tikzpicture}
	\begin{pgfonlayer}{nodelayer}
		\node [style=none] (0) at (-7.5, 1) {$N$};
		\node [style=none] (1) at (-6.75, 1) {$S$};
		\node [style=none] (2) at (-6, 1) {$N$};
		\node [style=none] (3) at (-7.5, 0.5) {};
		\node [style=none] (4) at (-6, 0.5) {};
		\node [style=none] (5) at (-6.75, 0.5) {};
		\node [style=none] (6) at (-9.75, 1) {$N$};
		\node [style=none] (7) at (-9.75, 0.5) {};
		\node [style=none] (8) at (-7.5, 1.75) {};
		\node [style=none] (9) at (-5.25, 2.5) {};
		\node [style=none] (10) at (-7.5, 2.5) {};
		\node [style=none] (11) at (-6, 1.75) {};
		\node [style=none] (12) at (-6.75, 3.75) {};
		\node [style=none] (13) at (-6.75, 2.5) {};
		\node [style=none] (14) at (-6, 2.5) {};
		\node [style=none] (15) at (-6.75, 1.75) {};
		\node [style=none] (16) at (-8.25, 2.5) {};
		\node [style=none] (17) at (-9.75, 2.5) {};
		\node [style=none] (18) at (-8.75, 2.5) {};
		\node [style=none] (19) at (-9.75, 1.75) {};
		\node [style=none] (20) at (-9.75, 3.5) {};
		\node [style=none] (21) at (-10.75, 2.5) {};
		\node [style=none] (22) at (-3.5, 1) {$N$};
		\node [style=none] (23) at (-2.5, 2.5) {};
		\node [style=none] (24) at (-4.5, 2.5) {};
		\node [style=none] (25) at (-3.5, 3.5) {};
		\node [style=none] (26) at (-3.5, 1.75) {};
		\node [style=none] (27) at (-3.5, 2.5) {};
		\node [style=none] (28) at (-3.5, 0.5) {};
		\node [style=none] (29) at (-6.75, -0.25) {};
		\node [style=none] (30) at (-1, 2.5) {};
		\node [style=none] (31) at (-1, 1) {$N$};
		\node [style=none] (32) at (-1, 0.5) {};
		\node [style=none] (33) at (0, 2.5) {};
		\node [style=none] (34) at (-2, 2.5) {};
		\node [style=none] (35) at (-1, 3.5) {};
		\node [style=none] (36) at (-1, 1.75) {};
		\node [style=none] (37) at (-3.5, -0.75) {};
		\node [style=none] (38) at (-3.5, -1.75) {};
		\node [fill=white, draw, thick, circle, minimum size=0.2 cm, style=none] (39) at (-6.75, -0.25) {};
		\node [fill=white, draw, thick, circle, minimum size=0.2 cm, style=none]  (40) at (-3.5, -0.75) {};
	\end{pgfonlayer}
	\begin{pgfonlayer}{edgelayer}
		\draw [style=thick] (16.center) to (9.center);
		\draw [style=thick] (12.center) to (16.center);
		\draw [style=thick] (12.center) to (9.center);
		\draw [style=thick] (10.center) to (8.center);
		\draw [style=thick] (13.center) to (15.center);
		\draw [style=thick] (14.center) to (11.center);
		\draw [style=thick] (21.center) to (18.center);
		\draw [style=thick] (18.center) to (20.center);
		\draw [style=thick] (20.center) to (21.center);
		\draw [style=thick] (17.center) to (19.center);
		\draw [style=thick] (24.center) to (23.center);
		\draw [style=thick] (23.center) to (25.center);
		\draw [style=thick] (25.center) to (24.center);
		\draw [style=thick] (27.center) to (26.center);
		\draw [style=thick] (5.center) to (29.center);
		\draw [style=thick, bend right=90, looseness=1.50] (7.center) to (3.center);
		\draw [style=thick] (34.center) to (33.center);
		\draw [style=thick] (33.center) to (35.center);
		\draw [style=thick] (35.center) to (34.center);
		\draw [style=thick] (30.center) to (36.center);
		\draw [style=thick, bend right=90, looseness=0.75] (4.center) to (32.center);
		\draw [style=thick] (28.center) to (37.center);
		\draw [style=thick] (37.center) to (38.center);
	\end{pgfonlayer}
\end{tikzpicture}}
\endpgfgraphicnamed}
\end{minipage}
\end{center}
The above corresponds to the  following morphism:
\begin{eqnarray*}
&&\mu_N
\circ\left(\overline{\text{'s}} \otimes 1_N\right)
\circ \left(\epsilon_N \otimes \iota_S \otimes \mu_N \otimes 1_N\right)
 \left(\ov{\text{Sub}} \otimes \ov{\text{Verb}} \otimes \ov{\text{Obj}}  \otimes\ov{\text{Poss}}\right) 
 \end{eqnarray*}

As an example of the occurrence of a relative clause in a sentence, consider the third verse of the translation of  quatrain (XLVI) of Omar Khayyam (11th century Persian poet and mathematician) by Fitzgerald \cite{Fitz}.  The full quatrain is as follows (the choice of capital letters is by Fitzgerald):
\begin{quote}
{\it 
For in and out, above, about, below\\
Tis nothing but a magic shadow-show\\
Play'd in a box \textbf{whose} candle is the sun,\\
Round \emph{which} we Phantom Figures come and go.}
\end{quote}

In `Play'd in a box \textbf{whose} candle is the sun', we have the possessive pronoun `whose' in a subject role. It is modifying the noun `candle' which is the subject of the verb `is' by the possessor noun `a box'. This part is analysed in exactly the same way as presented in section \ref{poss}. That is, the noun phrases   `a box' ,  `candle', and `the sun' have type $n$, whose has type $n^r n s^l n^l$, and the predicate `is' has type $n^rsn^l$. The possessive clause `a box whose candle is the sun', which has type $n$, is then used by the preposition `in' to modify the verb phrase `play'd'. For this part we may analyse  `play'd'  as a verb phrase $n^rs$ and hence the  preposition `in' will have type $(n^rs)^r s n^l$.  Or one can argue that `play'd' is an abbreviated sentence of type $s$ whose original sentence was something like `it is played'. In this case, the preposition `in' will have type $s^r sn^l$. In either case,  the general grammatical reduction is the same:  `in' inputs a verb phrase or a sentence on the left and a noun (which is the possessive clause) on the right; it then modifies the former with the latter and outputs a sentence. The normal form of the  compact closed meaning of this verse can then be depicted as follows, where $X$ can be either $F(n^r)$ or $F(n^rs)$, representing either of the discussed cases:

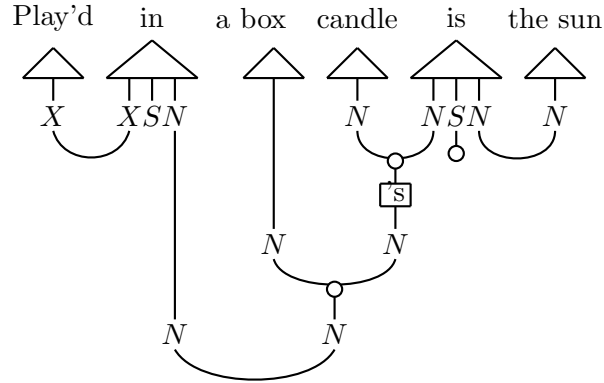
\begin{center}
   {%
\beginpgfgraphicnamed{poem-verse-2}
\begin{tikzpicture}[scale=0.8]
	\begin{pgfonlayer}{nodelayer}
		\node [style=none] (0) at (3.5, -1.75) {$N$};
		\node [style=none] (1) at (5.5, 2.5) {$S$};
		\node [style=none] (2) at (6.25, 2.5) {$N$};
		\node [style=none] (3) at (-0.5, -2.25) {};
		\node [style=none] (4) at (1.5, -4.25) {};
		\node [style=none] (5) at (3.5, -2.25) {};
		\node [style=none] (6) at (6.25, 2) {};
		\node [style=none] (7) at (8.75, 2.5) {$N$};
		\node [style=none] (8) at (8.75, 2) {};
		\node [style=none] (9) at (5.5, 1.25) {};
		\node [style=none] (10) at (2.25, 2) {};
		\node [style=none] (11) at (2.25, 3.75) {};
		\node [style=none] (12) at (0.5, 3.75) {};
		\node [style=none] (13) at (-0.5, 4.75) {};
		\node [style=none] (14) at (-1.5, 3.75) {};
		\node [style=none] (15) at (8.75, 3.75) {};
		\node [style=none] (16) at (8.75, 3) {};
		\node [style=none] (17) at (8.75, 4.75) {};
		\node [style=none] (18) at (9.75, 3.75) {};
		\node [style=none] (19) at (7.75, 3.75) {};
		\node [style=none] (20) at (5.5, 5) {};
		\node [style=none] (21) at (4.75, 3.75) {};
		\node [style=none] (22) at (5.5, 3) {};
		\node [style=none] (23) at (4.75, 2) {};
		\node [style=none] (24) at (4, 3.75) {};
		\node [style=none] (25) at (7, 3.75) {};
		\node [style=none] (26) at (5.5, 3.75) {};
		\node [style=none] (27) at (6.25, 3.75) {};
		\node [style=none] (28) at (6.25, 3) {};
		\node [style=none] (29) at (-0.5, -1.25) {};
		\node [style=none] (30) at (-0.5, 3.75) {};
		\node [style=none] (31) at (3.25, 3.75) {};
		\node [style=none] (32) at (1.25, 3.75) {};
		\node [style=none] (33) at (2.25, 4.75) {};
		\node [fill=white, draw, thick, circle, minimum size=0.2 cm, style=none] (34) at (1.5, -3.25) {};
		\node [style=none] (35) at (5.5, 2) {};
		\node [fill=white, draw, thick, circle, minimum size=0.2 cm, style=none] (36) at (5.5, 1.25) {};
		\node [style=none] (37) at (-1.25, 5.75) {a box};
		\node [style=none] (38) at (2.25, 5.75) {candle};
		\node [style=none] (39) at (5.5, 5.75) {is};
		\node [style=none] (40) at (8.75, 5.75) {the sun};
		\node [style=none] (41) at (3.5, -1.25) {};
		\node [style=none] (42) at (3.5, 0) {'s};
		\node [style=none] (43) at (3.5, 0.25) {};
		\node [style=none] (44) at (3, -0.5) {};
		\node [style=none] (45) at (4, 0.25) {};
		\node [style=none] (46) at (4, -0.5) {};
		\node [style=none] (47) at (-0.5, -1.75) {$N$};
		\node [style=none] (48) at (3.5, 0.25) {};
		\node [style=none] (49) at (3, 0.25) {};
		\node [style=none] (50) at (3.5, -0.5) {};
		\node [fill=white, draw, thick, circle, minimum size=0.2 cm, style=none] (51) at (3.5, 1) {};
		\node [style=none] (52) at (2.25, 3.75) {};
		\node [style=none] (53) at (2.25, 3) {};
		\node [style=none] (54) at (2.25, 2.5) {$N$};
		\node [style=none] (55) at (4.75, 3.75) {};
		\node [style=none] (56) at (4.75, 3) {};
		\node [style=none] (57) at (4.75, 2.5) {$N$};
		\node [style=none] (58) at (1.5, -4.75) {$N$};
		\node [style=none] (59) at (-5.25, 3) {};
		\node [style=none] (60) at (-3, 3.75) {};
		\node [style=none] (61) at (-3.75, 2.5) {$N$};
		\node [style=none] (62) at (-3.75, 3) {};
		\node [style=none] (63) at (-5.25, 3.75) {};
		\node [style=none] (64) at (-4.5, 3) {};
		\node [style=none] (65) at (-4.5, 5) {};
		\node [style=none] (66) at (-4.5, 3.75) {};
		\node [style=none] (67) at (-4.5, 2.5) {$S$};
		\node [style=none] (68) at (-3.75, 3.75) {};
		\node [style=none] (69) at (-5.25, 2.5) {$X$};
		\node [style=none] (70) at (-5.25, 3.75) {};
		\node [style=none] (71) at (-6, 3.75) {};
		\node [style=none] (72) at (-4.5, 5.75) {in};
		\node [style=none] (73) at (-7.75, 4.75) {};
		\node [style=none] (74) at (-8.75, 3.75) {};
		\node [style=none] (75) at (-7.75, 3.75) {};
		\node [style=none] (76) at (-6.75, 3.75) {};
		\node [style=none] (77) at (-7.75, 3) {};
		\node [style=none] (78) at (-7.75, 5.75) {Play'd};
		\node [style=none] (79) at (-7.75, 2.5) {$X$};
		\node [style=none] (80) at (-7.75, 2) {};
		\node [style=none] (81) at (-5.25, 2) {};
		\node [style=none] (82) at (-3.75, 2) {};
		\node [style=none] (83) at (-3.75, -4.25) {};
		\node [style=none] (84) at (-3.75, -4.75) {$N$};
		\node [style=none] (85) at (-3.75, -5.25) {};
		\node [style=none] (86) at (1.5, -5.25) {};
		\node [style=none] (87) at (1.5, -5.25) {};
	\end{pgfonlayer}
	\begin{pgfonlayer}{edgelayer}
		\draw [style=thick, bend right=90, looseness=1.25] (6.center) to (8.center);
		\draw [style=thick] (14.center) to (12.center);
		\draw [style=thick] (12.center) to (13.center);
		\draw [style=thick] (13.center) to (14.center);
		\draw [style=thick] (19.center) to (18.center);
		\draw [style=thick] (18.center) to (17.center);
		\draw [style=thick] (17.center) to (19.center);
		\draw [style=thick] (15.center) to (16.center);
		\draw [style=thick] (24.center) to (25.center);
		\draw [style=thick] (20.center) to (24.center);
		\draw [style=thick] (20.center) to (25.center);
		\draw [style=thick] (26.center) to (22.center);
		\draw [style=thick] (27.center) to (28.center);
		\draw [style=thick] (32.center) to (31.center);
		\draw [style=thick] (31.center) to (33.center);
		\draw [style=thick] (33.center) to (32.center);
		\draw [style=thick] (30.center) to (29.center);
		\draw [style=thick, bend right=75, looseness=0.75] (3.center) to (5.center);
		\draw [style=thick] (4.center) to (34.center);
		\draw [style=thick] (35.center) to (9.center);
		\draw [style=thick] (49.center) to (44.center);
		\draw [style=thick] (44.center) to (46.center);
		\draw [style=thick] (46.center) to (45.center);
		\draw [style=thick] (45.center) to (49.center);
		\draw [style=thick] (50.center) to (41.center);
		\draw [thick, bend right=90, looseness=1.25] (10.center) to (23.center);
		\draw [style=thick] (51.center) to (48.center);
		\draw [style=thick] (52.center) to (53.center);
		\draw [style=thick] (55.center) to (56.center);
		\draw [style=thick] (71.center) to (60.center);
		\draw [style=thick] (65.center) to (71.center);
		\draw [style=thick] (65.center) to (60.center);
		\draw [style=thick] (66.center) to (64.center);
		\draw [style=thick] (68.center) to (62.center);
		\draw [style=thick] (70.center) to (59.center);
		\draw [style=thick] (74.center) to (76.center);
		\draw [style=thick] (76.center) to (73.center);
		\draw [style=thick] (73.center) to (74.center);
		\draw [style=thick] (75.center) to (77.center);
		\draw [style=thick, bend right=75, looseness=1.25] (80.center) to (81.center);
		\draw [style=thick] (82.center) to (83.center);
		\draw [style=thick, bend right=60, looseness=0.75] (85.center) to (86.center);
	\end{pgfonlayer}
\end{tikzpicture}}
\endpgfgraphicnamed}
\end{center}

\noindent
It is apparent that  the process is very compositional. One can plug in the meanings of different parts of the phrases together to obtain a meaning for the full sentence.

\section{Decomposing  Whose}
\label{decomp}

In \cite{Lambek}, Lambek suggests that the type of  the compounds `whose Subject' and `whose Object' in the subject and object possessive clauses  should be the same as the type of the relative pronoun `that' in its subject and object roles, respectively.  These types are as follows:

\[
\mbox{\bf Subj:}\quad n^r n s^l n
\hspace{2cm}
\mbox{\bf Obj:}\quad n^r n n^{ll} s^l
\] 

We observe that this reduction is indeed the case in our setting, as shown in the corresponding  syntactic computations, depicted as follows:

\begin{center}
 {%
\beginpgfgraphicnamed{whose-decomp-Sbj}
\begin{tikzpicture}
	\begin{pgfonlayer}{nodelayer}
		\node [style=none] (0) at (-7, 1) {$n^r$};
		\node [style=none] (1) at (-6, 0.75) {$n$};
		\node [style=none] (2) at (-5, 1) {$s^l$};
		\node [style=none] (3) at (-4, 0.75) {$n$};
		\node [style=none] (4) at (-3, 1) {$n^l$};
		\node [style=none] (5) at (-1, 0.75) {$n$};
		\node [style=none] (6) at (-1, 0) {};
		\node [style=none] (7) at (-3, 0) {};
		\node [style=none] (8) at (-6, 0) {};
		\node [style=none] (9) at (-7, 0) {};
		\node [style=none] (10) at (-5, 0) {};
		\node [style=none] (11) at (-4, 0) {};
		\node [style=none] (12) at (-7, -0.75) {};
		\node [style=none] (13) at (-6, -0.75) {};
		\node [style=none] (14) at (-5, -0.75) {};
		\node [style=none] (15) at (-4, -0.75) {};
		\node [style=none] (16) at (-5.25, 2) {whose};
		\node [style=none] (17) at (-1, 2) {Subject};
		\node [style=none] (18) at (5, 0) {$(n^r n s^l nn^l) n \leq n^r n s^l n$};
	\end{pgfonlayer}
	\begin{pgfonlayer}{edgelayer}
		\draw [style=thick, bend right=90, looseness=1.25] (7.center) to (6.center);
		\draw [style=thick](9.center) to (12.center);
		\draw  [style=thick](8.center) to (13.center);
		\draw  [style=thick](10.center) to (14.center);
		\draw  [style=thick](11.center) to (15.center);
	\end{pgfonlayer}
\end{tikzpicture}}
\endpgfgraphicnamed}
 
 \medskip
 {%
\beginpgfgraphicnamed{whose-decomp-Obj}
\begin{tikzpicture}
	\begin{pgfonlayer}{nodelayer}
		\node [style=none] (0) at (-7, 1) {$n^r$};
		\node [style=none] (1) at (-6, 1) {$n$};
		\node [style=none] (2) at (-5, 1) {$n^{ll}$};
		\node [style=none] (3) at (-4, 1) {$s^l$};
		\node [style=none] (4) at (-3, 1) {$n^l$};
		\node [style=none] (5) at (-1, 0.75) {$n$};
		\node [style=none] (6) at (-1, 0) {};
		\node [style=none] (7) at (-3, 0) {};
		\node [style=none] (8) at (-6, 0) {};
		\node [style=none] (9) at (-7, 0) {};
		\node [style=none] (10) at (-5, 0) {};
		\node [style=none] (11) at (-4, 0) {};
		\node [style=none] (12) at (-7, -0.75) {};
		\node [style=none] (13) at (-6, -0.75) {};
		\node [style=none] (14) at (-5, -0.75) {};
		\node [style=none] (15) at (-4, -0.75) {};
		\node [style=none] (16) at (-5.25, 2) {whose};
		\node [style=none] (17) at (-1, 2) {Object};
		\node [style=none] (18) at (5, 0) {$(n^r nn^{ll} s^l n^l) n \leq n^r n n^{ll} s^l$};
	\end{pgfonlayer}
	\begin{pgfonlayer}{edgelayer}
		\draw [style=thick, bend right=90, looseness=1.25] (7.center) to (6.center);
		\draw  [style=thick](9.center) to (12.center);
		\draw  [style=thick](8.center) to (13.center);
		\draw  [style=thick](10.center) to (14.center);
		\draw  [style=thick](11.center) to (15.center);
	\end{pgfonlayer}
\end{tikzpicture}}
\endpgfgraphicnamed} 
\end{center}

\noindent
This way of looking at the type of `whose' suggests that any possessive relative clause can be seen as  a relative clause without the actual possessive pronoun `whose'. This is possible by a combination of  two relative pronouns and the predicate `has', as follows:

\medskip
\begin{tabular}{cc}
{\bf Poss Subj:} & \quad Possessor  that has  Subject that Verb Object.\\
{\bf Poss Obj:}& \quad Possessor that has  Object that Subject Verb.
\end{tabular}

\medskip
\noindent
For instance, for our above examples we would have:

\begin{quote}
`author that has a book that entertained John'\\
`author that has a book that John read'
\end{quote}

We verify that above suggestion is correct by showing that the possessive relative clauses and their non-whose version have the same meanings. 

\begin{proposition}
The clause `Possessor that has  Subject that Verb Object' has the same vector space meaning as the clause `Possessor whose Subject Verb Object'. 
\end{proposition}

\begin{proof}
The meaning of the subject  relative pronoun, as developed in previous work \cite{SadrClarkCoecke},   is depicted as follows:

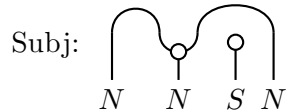
\begin{center}
\mbox{Subj:}  {%
\beginpgfgraphicnamed{subj-rel-sem}
\begin{tikzpicture}
	\begin{pgfonlayer}{nodelayer}
		\node [style=none] (0) at (-3.25, 0.75) {};
		\node [style=none] (1) at (-1.5, 0.75) {};
		\node [style=none] (2) at (1, 2) {};
		\node [style=none] (3) at (1, 0.75) {};
		\node [style=none] (4) at (0, 0.75) {};
		\node [style=none] (5) at (-3.25, 0.25) {$N$};
		\node [style=none] (6) at (-1.5, 0.25) {$N$};
		\node [style=none] (7) at (0, 0.25) {$S$};
		\node [style=none] (8) at (1, 0.25) {$N$};
		\node [fill=white, draw, thick, circle, minimum size=0.2 cm, style=none] (9) at (0, 1.75) {};
		\node [style=none] (10) at (-3.25, 2) {};
		\node [style=none] (11) at (-2, 2) {};
		\node [style=none] (12) at (-1, 2) {};
		\node [style=none] (13) at (-1.5, 1.5) {};
		\node [fill=white, draw, thick, circle, minimum size=0.2 cm, style=none] (14) at (-1.5, 1.5) {};
	\end{pgfonlayer}
	\begin{pgfonlayer}{edgelayer}
		\draw [style=thick] (2.center) to (3.center);
		\draw [style=thick, bend left=90, looseness=1.75] (10.center) to (11.center);
		\draw [style=thick] (9.center) to (4.center);
		\draw [style=thick, bend left=90, looseness=1.25] (12.center) to (2.center);
		\draw [style=thick, bend right=90, looseness=1.75] (11.center) to (12.center);
		\draw [style=thick] (10.center) to (0.center);
		\draw [style=thick] (13.center) to (1.center);
	\end{pgfonlayer}
\end{tikzpicture}}
\endpgfgraphicnamed}
\end{center}

\noindent
Hence, the  meaning of the clause `Possessor that has  Subject that Verb Object'  is computed  as follows:

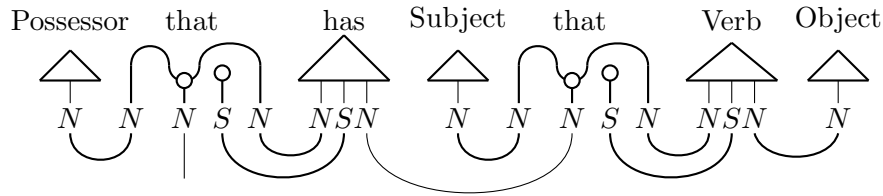
\begin{center}
   {%
\beginpgfgraphicnamed{whose-nowhose-Sbj}
\begin{tikzpicture}[scale=0.8]
	\begin{pgfonlayer}{nodelayer}
		\node [style=none] (0) at (0.75, 0.5) {};
		\node [style=none] (1) at (2.75, 0.5) {};
		\node [style=none] (2) at (1.75, 1.5) {};
		\node [style=none] (3) at (9.25, 0.5) {};
		\node [style=none] (4) at (12.25, 0.5) {};
		\node [style=none] (5) at (10.75, 1.75) {};
		\node [style=none] (6) at (14.25, 1.5) {};
		\node [style=none] (7) at (15.25, 0.5) {};
		\node [style=none] (8) at (13.25, 0.5) {};
		\node [style=none] (9) at (5.5, -1.25) {};
		\node [style=none] (10) at (10, 0.5) {};
		\node [style=none] (11) at (10.75, 0.5) {};
		\node [style=none] (12) at (11.5, 0.5) {};
		\node [style=none] (13) at (10, -0.75) {$N$};
		\node [style=none] (14) at (10.75, -0.75) {$S$};
		\node [style=none] (15) at (11.5, -0.75) {$N$};
		\node [style=none] (16) at (14.25, 0.5) {};
		\node [style=none] (17) at (14.25, -0.75) {$N$};
		\node [style=none] (18) at (1.75, 0.5) {};
		\node [style=none] (19) at (1.75, -0.75) {$N$};
		\node [style=none] (20) at (1.75, -1.25) {};
		\node [style=none] (21) at (3.75, -1.25) {};
		\node [style=none] (22) at (6.75, -1.25) {};
		\node [style=none] (23) at (8, -1.25) {};
		\node [style=none] (24) at (10, -1.25) {};
		\node [style=none] (25) at (10.75, -1.25) {};
		\node [style=none] (26) at (11.5, -1.25) {};
		\node [style=none] (27) at (14.25, -1.25) {};
		\node [style=none] (28) at (1.75, -0.25) {};
		\node [style=none] (29) at (10, -0.25) {};
		\node [style=none] (30) at (10.75, -0.25) {};
		\node [style=none] (31) at (11.5, -0.25) {};
		\node [style=none] (32) at (14.25, -0.25) {};
		\node [style=none] (33) at (8, 1) {};
		\node [style=none] (34) at (6, 1) {};
		\node [style=none] (35) at (5.5, -0.25) {};
		\node [fill=white, draw, thick, circle, minimum size=0.2 cm, style=none] (36) at (6.75, 0.75) {};
		\node [style=none] (37) at (3.75, -0.75) {$N$};
		\node [style=none] (38) at (8, -0.25) {};
		\node [style=none] (39) at (8, -0.75) {$N$};
		\node [style=none] (40) at (5.5, -0.75) {$N$};
		\node [style=none] (41) at (6.75, -0.75) {$S$};
		\node [style=none] (42) at (6.75, -0.25) {};
		\node [style=none] (43) at (3.75, 1) {};
		\node [style=none] (44) at (5.5, 0.5) {};
		\node [fill=white, draw, thick, circle, minimum size=0.2 cm, style=none] (45) at (5.5, 0.5) {};
		\node [style=none] (46) at (3.75, -0.25) {};
		\node [style=none] (47) at (5, 1) {};
		\node [style=none] (48) at (1.75, 2.5) {Subject};
		\node [style=none] (49) at (5.75, 2.5) {that};
		\node [style=none] (50) at (10.75, 2.5) {Verb};
		\node [style=none] (51) at (14.25, 2.5) {Object};
		\node [style=none] (52) at (-7.5, -6) {};
		\node [style=none] (53) at (-2, 0.5) {};
		\node [style=none] (54) at (-9, -0.25) {};
		\node [style=none] (55) at (-9, -1.25) {};
		\node [style=none] (56) at (-2, 2.5) {has};
		\node [style=none] (57) at (-11, -0.25) {};
		\node [style=none] (58) at (-1.25, -0.75) {$N$};
		\node [style=none] (59) at (-7.25, -0.75) {$N$};
		\node [style=none] (60) at (-11, 2.5) {Possessor};
		\node [style=none] (61) at (-7.25, 0.5) {};
		\node [style=none] (62) at (-7.75, 1) {};
		\node [style=none] (63) at (-6, -0.25) {};
		\node [style=none] (64) at (-4.75, -0.25) {};
		\node [style=none] (65) at (-2, 2) {};
		\node [style=none] (66) at (-3.5, 0.5) {};
		\node [style=none] (67) at (-1.25, -1.25) {};
		\node [fill=white, draw, thick, circle, minimum size=0.2 cm, style=none] (68) at (-7.25, 0.5) {};
		\node [style=none] (69) at (-7.25, -1.25) {};
		\node [style=none] (70) at (-2.75, -0.75) {$N$};
		\node [style=none] (71) at (-9, 1) {};
		\node [style=none] (72) at (-4.75, -1.25) {};
		\node [style=none] (73) at (-2, -0.75) {$S$};
		\node [style=none] (74) at (-9, -0.75) {$N$};
		\node [style=none] (75) at (-11, -1.25) {};
		\node [style=none] (76) at (-11, 0.5) {};
		\node [fill=white, draw, thick, circle, minimum size=0.2 cm, style=none] (77) at (-6, 0.75) {};
		\node [style=none] (78) at (-2, -1.25) {};
		\node [style=none] (79) at (-1.25, -0.25) {};
		\node [style=none] (80) at (-4.75, 1) {};
		\node [style=none] (81) at (-0.5, 0.5) {};
		\node [style=none] (82) at (-2.75, -1.25) {};
		\node [style=none] (83) at (-2.75, -0.25) {};
		\node [style=none] (84) at (-11, -0.75) {$N$};
		\node [style=none] (85) at (-1.25, 0.5) {};
		\node [style=none] (86) at (-7, 2.5) {that};
		\node [style=none] (87) at (-6, -1.25) {};
		\node [style=none] (88) at (-7.25, -0.25) {};
		\node [style=none] (89) at (-10, 0.5) {};
		\node [style=none] (90) at (-4.75, -0.75) {$N$};
		\node [style=none] (91) at (-2.75, 0.5) {};
		\node [style=none] (92) at (-2, -0.25) {};
		\node [style=none] (93) at (-6, -0.75) {$S$};
		\node [style=none] (94) at (-6.75, 1) {};
		\node [style=none] (95) at (-12, 0.5) {};
		\node [style=none] (96) at (-11, 1.5) {};
		\node [style=none] (97) at (-7.25, -2.75) {};
	\end{pgfonlayer}
	\begin{pgfonlayer}{edgelayer}
		\draw [style=thick] (0.center) to (1.center);
		\draw [style=thick] (1.center) to (2.center);
		\draw [style=thick] (2.center) to (0.center);
		\draw [style=thick] (3.center) to (4.center);
		\draw [style=thick] (5.center) to (3.center);
		\draw [style=thick] (5.center) to (4.center);
		\draw [style=thick] (8.center) to (7.center);
		\draw [style=thick] (7.center) to (6.center);
		\draw [style=thick] (6.center) to (8.center);
		\draw [style=thick, bend left=270, looseness=1.50] (20.center) to (21.center);
		\draw [style=thick, bend right=75] (26.center) to (27.center);
		\draw [style=thick, bend left=90, looseness=1.25] (24.center) to (23.center);
		\draw [style=thick, bend right=90, looseness=1.25] (22.center) to (25.center);
		\draw [style=plain] (18.center) to (28.center);
		\draw [style=plain] (10.center) to (29.center);
		\draw [style=plain] (11.center) to (30.center);
		\draw [style=plain] (12.center) to (31.center);
		\draw [style=plain] (16.center) to (32.center);
		\draw [style=thick] (33.center) to (38.center);
		\draw [style=thick, bend left=90, looseness=1.75] (43.center) to (47.center);
		\draw [style=thick] (36.center) to (42.center);
		\draw [style=thick, bend left=90, looseness=1.25] (34.center) to (33.center);
		\draw [style=thick, bend right=90, looseness=1.75] (47.center) to (34.center);
		\draw [style=thick] (43.center) to (46.center);
		\draw [style=thick] (44.center) to (35.center);
		\draw [style=thick] (95.center) to (89.center);
		\draw [style=thick] (89.center) to (96.center);
		\draw [style=thick] (96.center) to (95.center);
		\draw [style=thick] (66.center) to (81.center);
		\draw [style=thick] (65.center) to (66.center);
		\draw [style=thick] (65.center) to (81.center);
		\draw [style=thick, bend left=270, looseness=1.50] (75.center) to (55.center);
		\draw [style=thick, bend left=90, looseness=1.25] (82.center) to (72.center);
		\draw [style=thick, bend right=90, looseness=1.25] (87.center) to (78.center);
		\draw [style=plain] (76.center) to (57.center);
		\draw [style=plain] (91.center) to (83.center);
		\draw [style=plain] (53.center) to (92.center);
		\draw [style=plain] (85.center) to (79.center);
		\draw [style=thick] (80.center) to (64.center);
		\draw [style=thick, bend left=90, looseness=1.75] (71.center) to (62.center);
		\draw [style=thick] (77.center) to (63.center);
		\draw [style=thick, bend left=90, looseness=1.25] (94.center) to (80.center);
		\draw [style=thick, bend right=90, looseness=1.75] (62.center) to (94.center);
		\draw [style=thick] (71.center) to (54.center);
		\draw [style=thick] (61.center) to (88.center);
		\draw (69.center) to (97.center);
		\draw [bend right=90] (67.center) to (9.center);
	\end{pgfonlayer}
\end{tikzpicture}}
\endpgfgraphicnamed}
\end{center}

\noindent
This normalises to the following:

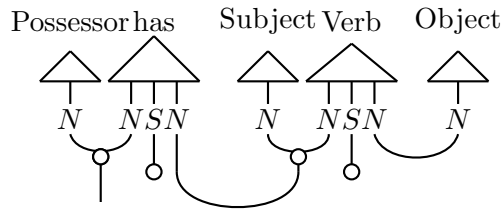
\begin{center}
   {%
\beginpgfgraphicnamed{whose-nowhose-Sbj-norm1}
\begin{tikzpicture}[scale=0.8]
	\begin{pgfonlayer}{nodelayer}
		\node [style=none] (0) at (-8.5, 0.25) {};
		\node [style=none] (1) at (-6.5, 0.25) {};
		\node [style=none] (2) at (-7.5, 1.25) {};
		\node [style=none] (3) at (-6.25, 0.25) {};
		\node [style=none] (4) at (-3.25, 0.25) {};
		\node [style=none] (5) at (-4.75, 1.5) {};
		\node [style=none] (6) at (-1.25, 1.25) {};
		\node [style=none] (7) at (-0.25, 0.25) {};
		\node [style=none] (8) at (-2.25, 0.25) {};
		\node [style=none] (9) at (-6.5, -2.75) {};
		\node [style=none] (10) at (-5.5, 0.25) {};
		\node [style=none] (11) at (-4.75, 0.25) {};
		\node [style=none] (12) at (-4, 0.25) {};
		\node [style=none] (13) at (-5.5, -1) {$N$};
		\node [style=none] (14) at (-4.75, -1) {$S$};
		\node [style=none] (15) at (-4, -1) {$N$};
		\node [style=none] (16) at (-1.25, 0.25) {};
		\node [style=none] (17) at (-1.25, -1) {$N$};
		\node [style=none] (18) at (-7.5, 0.25) {};
		\node [style=none] (19) at (-7.5, -1) {$N$};
		\node [style=none] (20) at (-7.5, -1.5) {};
		\node [style=none] (21) at (-5.5, -1.5) {};
		\node [style=none] (22) at (-4, -1.5) {};
		\node [style=none] (23) at (-1.25, -1.5) {};
		\node [style=none] (24) at (-7.5, -0.5) {};
		\node [style=none] (25) at (-5.5, -0.5) {};
		\node [style=none] (26) at (-4.75, -0.5) {};
		\node [style=none] (27) at (-4, -0.5) {};
		\node [style=none] (28) at (-1.25, -0.5) {};
		\node [fill=white, draw, thick, circle, minimum size=0.2 cm, style=none] (29) at (-6.5, -2.25) {};
		\node [style=none] (30) at (-7.5, 2.25) {Subject};
		\node [style=none] (31) at (-4.75, 2.25) {Verb};
		\node [style=none] (32) at (-1.25, 2.25) {Object};
		\node [style=none] (33) at (-7.5, -6) {};
		\node [style=none] (34) at (-11.25, 0.25) {};
		\node [style=none] (35) at (-12, -1.5) {};
		\node [style=none] (36) at (-11.25, 2.25) {has};
		\node [style=none] (37) at (-14, -0.5) {};
		\node [style=none] (38) at (-10.5, -1) {$N$};
		\node [style=none] (39) at (-14, 2.25) {Possessor};
		\node [style=none] (40) at (-11.25, -1.5) {};
		\node [style=none] (41) at (-11.25, 1.75) {};
		\node [style=none] (42) at (-12.75, 0.25) {};
		\node [style=none] (43) at (-10.5, -2.75) {};
		\node [fill=white, draw, thick, circle, minimum size=0.2 cm, style=none] (44) at (-13, -2.25) {};
		\node [style=none] (45) at (-13, -2.25) {};
		\node [style=none] (46) at (-12, -1) {$N$};
		\node [style=none] (47) at (-11.25, -1) {$S$};
		\node [style=none] (48) at (-14, -1.5) {};
		\node [style=none] (49) at (-14, 0.25) {};
		\node [fill=white, draw, thick, circle, minimum size=0.2 cm, style=none] (50) at (-11.25, -2.75) {};
		\node [style=none] (51) at (-10.5, -0.5) {};
		\node [style=none] (52) at (-9.75, 0.25) {};
		\node [style=none] (53) at (-12, -0.5) {};
		\node [style=none] (54) at (-14, -1) {$N$};
		\node [style=none] (55) at (-10.5, 0.25) {};
		\node [style=none] (56) at (-13, 0.25) {};
		\node [style=none] (57) at (-12, 0.25) {};
		\node [style=none] (58) at (-11.25, -0.5) {};
		\node [style=none] (59) at (-15, 0.25) {};
		\node [style=none] (60) at (-14, 1.25) {};
		\node [style=none] (61) at (-13, -3.75) {};
		\node [fill=white, draw, thick, circle, minimum size=0.2 cm, style=none] (62) at (-13, -2.25) {};
		\node [style=none] (63) at (-10.5, -1.5) {};
		\node [fill=white, draw, thick, circle, minimum size=0.2 cm, style=none] (64) at (-4.75, -2.75) {};
		\node [style=none] (65) at (-4.75, -1.5) {};
		\node [style=none] (66) at (-6.5, -2.75) {};
	\end{pgfonlayer}
	\begin{pgfonlayer}{edgelayer}
		\draw [style=thick] (0.center) to (1.center);
		\draw [style=thick] (1.center) to (2.center);
		\draw [style=thick] (2.center) to (0.center);
		\draw [style=thick] (3.center) to (4.center);
		\draw [style=thick] (5.center) to (3.center);
		\draw [style=thick] (5.center) to (4.center);
		\draw [style=thick] (8.center) to (7.center);
		\draw [style=thick] (7.center) to (6.center);
		\draw [style=thick] (6.center) to (8.center);
		\draw [style=thick, bend right=90] (20.center) to (21.center);
		\draw [style=thick, bend right=75] (22.center) to (23.center);
		\draw [style=thick] (18.center) to (24.center);
		\draw [style=thick] (10.center) to (25.center);
		\draw [style=thick] (11.center) to (26.center);
		\draw [style=thick] (12.center) to (27.center);
		\draw [style=thick] (16.center) to (28.center);
		\draw [style=thick] (59.center) to (56.center);
		\draw [style=thick] (56.center) to (60.center);
		\draw [style=thick] (60.center) to (59.center);
		\draw [style=thick] (42.center) to (52.center);
		\draw [style=thick] (41.center) to (42.center);
		\draw [style=thick] (41.center) to (52.center);
		\draw [style=thick, bend right=75] (48.center) to (35.center);
		\draw [style=thick] (49.center) to (37.center);
		\draw [style=thick] (57.center) to (53.center);
		\draw [style=thick] (34.center) to (58.center);
		\draw [style=thick] (55.center) to (51.center);
		\draw [style=thick] (50.center) to (40.center);
		\draw [style=thick] (45.center) to (61.center);
		\draw [style=thick, bend right=90] (43.center) to (9.center);
		\draw [style=thick] (63.center) to (43.center);
		\draw [style=thick] (64.center) to (65.center);
		\draw [style=thick] (29.center) to (66.center);
	\end{pgfonlayer}
\end{tikzpicture}}
\endpgfgraphicnamed}
\end{center}

\noindent
The vector space meaning of the application of the  unit of the Frobenius algebra, that is the $\iota$ map, on the  predicate `has'  is as computed as follows:
\[
(1_N \otimes \iota_S \otimes 1_N) (\overline{\text{has}}) := \sum_{ijk} C_{ijk} \ov{n}_i \otimes \iota(\ov{s}_j) \otimes \ov{n}_k = 
\sum_{ijk} C_{ijk} \ov{n}_i \otimes  \ov{n}_k = \sum_{ik} C_{ik} \ov{n}_i  \otimes \ov{n}_k
\]
This is an element of the space $N \otimes N$, which in our vector space setting is isomorphic to the set of   linear maps from $N$ to $N$. Pictorially, we have:

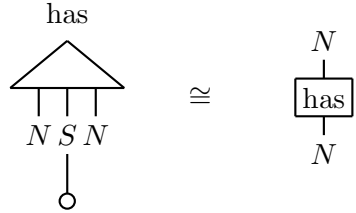
\begin{center}
   {%
\beginpgfgraphicnamed{Has}
\begin{tikzpicture}
	\begin{pgfonlayer}{nodelayer}
		\node [style=none] (0) at (-13.5, 2.75) {};
		\node [style=none] (1) at (-10.5, 2.75) {};
		\node [style=none] (2) at (-12, 4) {};
		\node [style=none] (3) at (-12.75, 2.75) {};
		\node [style=none] (4) at (-12, 2.75) {};
		\node [style=none] (5) at (-11.25, 2.75) {};
		\node [style=none] (6) at (-12.75, 1.5) {$N$};
		\node [style=none] (7) at (-12, 1.5) {$S$};
		\node [style=none] (8) at (-11.25, 1.5) {$N$};
		\node [style=none] (9) at (-12.75, 1) {};
		\node [style=none] (10) at (-11.25, 1) {};
		\node [style=none] (11) at (-12.75, 2) {};
		\node [style=none] (12) at (-12, 2) {};
		\node [style=none] (13) at (-11.25, 2) {};
		\node [style=none] (14) at (-12, 4.75) {has};
		\node [style=none] (15) at (-7.5, -6) {};
		\node [style=none] (16) at (-5.25, 1.5) {};
		\node [style=none] (17) at (-5.25, 4) {$N$};
		\node [style=none] (18) at (-7.25, -6.25) {};
		\node [fill=white, draw, thick, circle, minimum size=0.2 cm, style=none] (19) at (-12, -0.25) {};
		\node [style=none] (20) at (-12, 1) {};
		\node [style=none] (21) at (-6, 3) {};
		\node [style=none] (22) at (-4.5, 3) {};
		\node [style=none] (23) at (-4.5, 2) {};
		\node [style=none] (24) at (-6, 2) {};
		\node [style=none] (25) at (-5.25, 2) {};
		\node [style=none] (26) at (-5.25, 1.5) {};
		\node [style=none] (27) at (-5.25, 2.5) {has};
		\node [style=none] (28) at (-5.25, 3) {};
		\node [style=none] (29) at (-5.25, 3.5) {};
		\node [style=none] (30) at (-5.25, 1) {$N$};
		\node [style=none] (31) at (-8.5, 2.5) {$\cong$};
		\node [style=none] (32) at (-8.5, 2.5) {};
	\end{pgfonlayer}
	\begin{pgfonlayer}{edgelayer}
		\draw [style=thick] (0.center) to (1.center);
		\draw [style=thick] (2.center) to (0.center);
		\draw [style=thick] (2.center) to (1.center);
		\draw [style=thick] (3.center) to (11.center);
		\draw [style=thick] (4.center) to (12.center);
		\draw [style=thick] (5.center) to (13.center);
		\draw [style=thick] (19.center) to (20.center);
		\draw [style=thick] (21.center) to (24.center);
		\draw [style=thick](24.center) to (23.center);
		\draw [style=thick](23.center) to (22.center);
		\draw [style=thick](22.center) to (21.center);
		\draw [style=thick](25.center) to (26.center);
		\draw [style=thick](29.center) to (28.center);
	\end{pgfonlayer}
\end{tikzpicture}}
\endpgfgraphicnamed}
\end{center}

\noindent
As a result, the above simplifies further to the following:

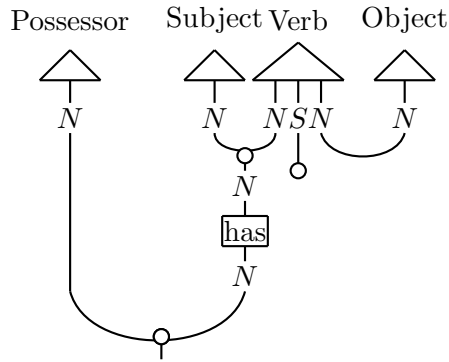
\begin{center}
   {%
\beginpgfgraphicnamed{whose-nowhose-Sbj-norm2}
\begin{tikzpicture}[scale=0.8]
	\begin{pgfonlayer}{nodelayer}
		\node [style=none] (0) at (-6.5, 3) {};
		\node [style=none] (1) at (-4.5, 3) {};
		\node [style=none] (2) at (-5.5, 4) {};
		\node [style=none] (3) at (-4.25, 3) {};
		\node [style=none] (4) at (-1.25, 3) {};
		\node [style=none] (5) at (-2.75, 4.25) {};
		\node [style=none] (6) at (0.75, 4) {};
		\node [style=none] (7) at (1.75, 3) {};
		\node [style=none] (8) at (-0.25, 3) {};
		\node [style=none] (9) at (-3.5, 3) {};
		\node [style=none] (10) at (-2.75, 3) {};
		\node [style=none] (11) at (-2, 3) {};
		\node [style=none] (12) at (-3.5, 1.75) {$N$};
		\node [style=none] (13) at (-2.75, 1.75) {$S$};
		\node [style=none] (14) at (-2, 1.75) {$N$};
		\node [style=none] (15) at (0.75, 3) {};
		\node [style=none] (16) at (0.75, 1.75) {$N$};
		\node [style=none] (17) at (-5.5, 3) {};
		\node [style=none] (18) at (-5.5, 1.75) {$N$};
		\node [style=none] (19) at (-5.5, 1.25) {};
		\node [style=none] (20) at (-3.5, 1.25) {};
		\node [style=none] (21) at (-2, 1.25) {};
		\node [style=none] (22) at (0.75, 1.25) {};
		\node [style=none] (23) at (-5.5, 2.25) {};
		\node [style=none] (24) at (-3.5, 2.25) {};
		\node [style=none] (25) at (-2.75, 2.25) {};
		\node [style=none] (26) at (-2, 2.25) {};
		\node [style=none] (27) at (0.75, 2.25) {};
		\node [fill=white, draw, thick, circle, minimum size=0.2 cm, style=none] (28) at (-4.5, 0.5) {};
		\node [style=none] (29) at (-5.5, 5) {Subject};
		\node [style=none] (30) at (-2.75, 5) {Verb};
		\node [style=none] (31) at (0.75, 5) {Object};
		\node [style=none] (32) at (-7.5, -6) {};
		\node [style=none] (33) at (-4.5, -4) {};
		\node [style=none] (34) at (-10.25, 2.25) {};
		\node [style=none] (35) at (-10.25, 5) {Possessor};
		\node [fill=white, draw, thick, circle, minimum size=0.2 cm, style=none] (36) at (-7.25, -5.5) {};
		\node [style=none] (37) at (-7.25, -5.5) {};
		\node [style=none] (38) at (-4.5, -0.5) {$N$};
		\node [style=none] (39) at (-10.25, -4) {};
		\node [style=none] (40) at (-10.25, 3) {};
		\node [style=none] (41) at (-10.25, 1.75) {$N$};
		\node [style=none] (42) at (-9.25, 3) {};
		\node [style=none] (43) at (-11.25, 3) {};
		\node [style=none] (44) at (-10.25, 4) {};
		\node [style=none] (45) at (-7.25, -6.25) {};
		\node [fill=white, draw, thick, circle, minimum size=0.2 cm, style=none] (46) at (-7.25, -5.5) {};
		\node [fill=white, draw, thick, circle, minimum size=0.2 cm, style=none] (47) at (-2.75, 0) {};
		\node [style=none] (48) at (-2.75, 1.25) {};
		\node [style=none] (49) at (-4.5, 0) {};
		\node [style=none] (50) at (-5.25, -1.5) {};
		\node [style=none] (51) at (-3.75, -1.5) {};
		\node [style=none] (52) at (-3.75, -2.5) {};
		\node [style=none] (53) at (-5.25, -2.5) {};
		\node [style=none] (54) at (-4.5, -2.5) {};
		\node [style=none] (55) at (-4.5, -3) {};
		\node [style=none] (56) at (-4.5, -2) {has};
		\node [style=none] (57) at (-4.5, -1.5) {};
		\node [style=none] (58) at (-4.5, -1) {};
		\node [style=none] (59) at (-10.25, 1.25) {};
		\node [style=none] (60) at (-10.25, -4) {};
		\node [style=none] (61) at (-4.5, -3.5) {$N$};
	\end{pgfonlayer}
	\begin{pgfonlayer}{edgelayer}
		\draw [style=thick] (0.center) to (1.center);
		\draw [style=thick] (1.center) to (2.center);
		\draw [style=thick] (2.center) to (0.center);
		\draw [style=thick] (3.center) to (4.center);
		\draw [style=thick] (5.center) to (3.center);
		\draw [style=thick] (5.center) to (4.center);
		\draw [style=thick] (8.center) to (7.center);
		\draw [style=thick] (7.center) to (6.center);
		\draw [style=thick] (6.center) to (8.center);
		\draw [style=thick, bend right=90] (19.center) to (20.center);
		\draw [style=thick, bend right=75] (21.center) to (22.center);
		\draw [style=thick] (17.center) to (23.center);
		\draw [style=thick] (9.center) to (24.center);
		\draw [style=thick] (10.center) to (25.center);
		\draw [style=thick] (11.center) to (26.center);
		\draw [style=thick] (15.center) to (27.center);
		\draw [style=thick] (43.center) to (42.center);
		\draw [style=thick] (42.center) to (44.center);
		\draw [style=thick] (44.center) to (43.center);
		\draw [style=thick, bend right=75] (39.center) to (33.center);
		\draw [style=thick] (40.center) to (34.center);
		\draw [style=thick](37.center) to (45.center);
		\draw [style=thick] (47.center) to (48.center);
		\draw [style=thick](28.center) to (49.center);
		\draw [style=thick](50.center) to (53.center);
		\draw [style=thick](53.center) to (52.center);
		\draw [style=thick](52.center) to (51.center);
		\draw [style=thick](51.center) to (50.center);
		\draw [style=thick](54.center) to (55.center);
		\draw [style=thick](58.center) to (57.center);
		\draw [style=thick](59.center) to (60.center);
	\end{pgfonlayer}
\end{tikzpicture}}
\endpgfgraphicnamed}
\end{center}

\noindent
This is the same as the meaning of the phrase `Possessor whose Subject Verb  Object'. 

\end{proof}

\begin{proposition}
The clause `Possessor that has  Object that Subject Verb' has the same vector space meaning as the clause `Possessor whose Object Subject Verb'. 
\end{proposition}

\begin{proof}
The meaning of the object relative pronoun as developed in previous work \cite{SadrClarkCoecke},   is depicted as follows:

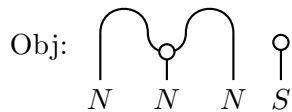
\begin{center}
\mbox{Obj:}  {%
\beginpgfgraphicnamed{obj-rel-sem}
\begin{tikzpicture}
	\begin{pgfonlayer}{nodelayer}
		\node [style=none] (0) at (-2, 0.5) {$N$};
		\node [style=none] (1) at (-0.25, 0.5) {$N$};
		\node [style=none] (2) at (2.75, 0.5) {$S$};
		\node [style=none] (3) at (1.5, 0.5) {$N$};
		\node [style=none] (4) at (-0.75, 2.25) {};
		\node [style=none] (5) at (-2, 1) {};
		\node [style=none] (6) at (-0.25, 1) {};
		\node [fill=white, draw, thick, circle, minimum size=0.2 cm, style=none] (7) at (2.75, 2) {};
		\node [style=none] (8) at (-2, 2.25) {};
		\node [style=none] (9) at (-0.25, 1.75) {};
		\node [style=none] (10) at (0.25, 2.25) {};
		\node [style=none] (11) at (1.5, 1) {};
		\node [style=none] (12) at (1.5, 2.25) {};
		\node [style=none] (13) at (2.75, 1) {};
		\node [fill=white, draw, thick, circle, minimum size=0.2 cm, style=none] (14) at (-0.25, 1.75) {};
	\end{pgfonlayer}
	\begin{pgfonlayer}{edgelayer}
		\draw [style=thick] (12.center) to (11.center);
		\draw [style=thick, bend left=90, looseness=1.75] (8.center) to (4.center);
		\draw [style=thick] (7.center) to (13.center);
		\draw [style=thick, bend left=90, looseness=1.75] (10.center) to (12.center);
		\draw [style=thick, bend right=90, looseness=1.75] (4.center) to (10.center);
		\draw [style=thick] (8.center) to (5.center);
		\draw [style=thick] (9.center) to (6.center);
	\end{pgfonlayer}
\end{tikzpicture}}
\endpgfgraphicnamed}
\end{center}

\noindent
Hence, the  meaning of the clause `Possessor that has  Object that Subject Verb'  is computed  as follows:

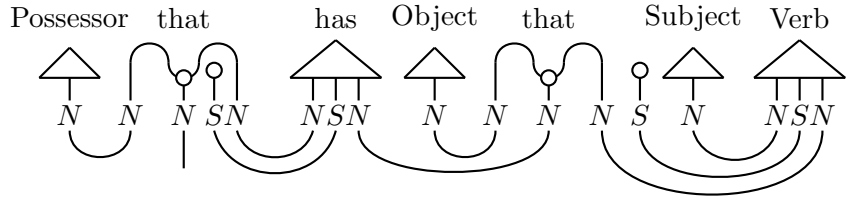
\begin{center}
   {%
\beginpgfgraphicnamed{whose-nowhose-Obj}
\begin{tikzpicture}[scale=0.8]
	\begin{pgfonlayer}{nodelayer}
		\node [style=none] (0) at (-1.5, 1.75) {};
		\node [style=none] (1) at (0.5, 1.75) {};
		\node [style=none] (2) at (-0.5, 2.75) {};
		\node [style=none] (3) at (10, 1.75) {};
		\node [style=none] (4) at (13, 1.75) {};
		\node [style=none] (5) at (11.5, 3) {};
		\node [style=none] (6) at (8, 2.75) {};
		\node [style=none] (7) at (9, 1.75) {};
		\node [style=none] (8) at (7, 1.75) {};
		\node [style=none] (9) at (3.25, 0) {};
		\node [style=none] (10) at (10.75, 1.75) {};
		\node [style=none] (11) at (11.5, 1.75) {};
		\node [style=none] (12) at (12.25, 1.75) {};
		\node [style=none] (13) at (10.75, 0.5) {$N$};
		\node [style=none] (14) at (11.5, 0.5) {$S$};
		\node [style=none] (15) at (12.25, 0.5) {$N$};
		\node [style=none] (16) at (8, 1.75) {};
		\node [style=none] (17) at (8, 0.5) {$N$};
		\node [style=none] (18) at (-0.5, 1.75) {};
		\node [style=none] (19) at (-0.5, 0.5) {$N$};
		\node [style=none] (20) at (-0.5, 0) {};
		\node [style=none] (21) at (1.5, 0) {};
		\node [style=none] (22) at (5, 0) {};
		\node [style=none] (23) at (6.25, 0) {};
		\node [style=none] (24) at (8, 0) {};
		\node [style=none] (25) at (-0.5, 1) {};
		\node [style=none] (26) at (10.75, 1) {};
		\node [style=none] (27) at (11.5, 1) {};
		\node [style=none] (28) at (12.25, 1) {};
		\node [style=none] (29) at (8, 1) {};
		\node [style=none] (30) at (10.75, 0) {};
		\node [style=none] (31) at (11.5, 0) {};
		\node [style=none] (32) at (12.25, 0) {};
		\node [fill=white, draw, thick, circle, minimum size=0.2 cm, style=none] (33) at (6.25, 2) {};
		\node [style=none] (34) at (1.5, 0.5) {$N$};
		\node [style=none] (35) at (3.75, 2.25) {};
		\node [style=none] (36) at (5, 2.25) {};
		\node [style=none] (37) at (5, 0.5) {$N$};
		\node [fill=white, draw, thick, circle, minimum size=0.2 cm, style=none] (38) at (3.25, 1.75) {};
		\node [style=none] (39) at (5, 1) {};
		\node [style=none] (40) at (3.25, 1.75) {};
		\node [style=none] (41) at (3.25, 0.5) {$N$};
		\node [style=none] (42) at (1.5, 1) {};
		\node [style=none] (43) at (6.25, 0.5) {$S$};
		\node [style=none] (44) at (1.5, 2.25) {};
		\node [style=none] (45) at (3.25, 1) {};
		\node [style=none] (46) at (2.75, 2.25) {};
		\node [style=none] (47) at (6.25, 1) {};
		\node [style=none] (48) at (-0.5, 3.75) {Object};
		\node [style=none] (49) at (3.25, 3.75) {that};
		\node [style=none] (50) at (8, 3.75) {Subject};
		\node [style=none] (51) at (11.5, 3.75) {Verb};
		\node [fill=white, draw, thick, circle, minimum size=0.2 cm, style=none] (52) at (-8.75, 1.75) {};
		\node [style=none] (53) at (-8.75, 1) {};
		\node [style=none] (54) at (-8.25, 2.25) {};
		\node [style=none] (55) at (-7, 2.25) {};
		\node [style=none] (56) at (-7.75, 1) {};
		\node [style=none] (57) at (-9.25, 2.25) {};
		\node [style=none] (58) at (-7.75, 0.5) {$S$};
		\node [style=none] (59) at (-8.75, 0.5) {$N$};
		\node [style=none] (60) at (-7, 0.5) {$N$};
		\node [style=none] (61) at (-10.5, 1) {};
		\node [style=none] (62) at (-10.5, 0.5) {$N$};
		\node [style=none] (63) at (-7, 1) {};
		\node [fill=white, draw, thick, circle, minimum size=0.2 cm, style=none] (64) at (-7.75, 2) {};
		\node [style=none] (65) at (-10.5, 2.25) {};
		\node [style=none] (66) at (-8.75, 1.75) {};
		\node [style=none] (67) at (-5.25, 1.75) {};
		\node [style=none] (68) at (-3.75, 1) {};
		\node [style=none] (69) at (-3.75, 3.75) {has};
		\node [style=none] (70) at (-3.75, 3) {};
		\node [style=none] (71) at (-4.5, 1) {};
		\node [style=none] (72) at (-3, 1) {};
		\node [style=none] (73) at (-2.25, 1.75) {};
		\node [style=none] (74) at (-4.5, 1.75) {};
		\node [style=none] (75) at (-3, 1.75) {};
		\node [style=none] (76) at (-3.75, 1.75) {};
		\node [style=none] (77) at (-12.5, 1) {};
		\node [style=none] (78) at (-11.5, 1.75) {};
		\node [style=none] (79) at (-12.5, 1.75) {};
		\node [style=none] (80) at (-13.5, 1.75) {};
		\node [style=none] (81) at (-12.5, 3.75) {Possessor};
		\node [style=none] (82) at (-12.5, 2.75) {};
		\node [style=none] (83) at (-12.5, 0.5) {$N$};
		\node [style=none] (84) at (-8.75, 3.75) {that};
		\node [style=none] (85) at (-3, 0.5) {$N$};
		\node [style=none] (86) at (-4.5, 0.5) {$N$};
		\node [style=none] (87) at (-3.75, 0.5) {$S$};
		\node [style=none] (88) at (-7, 0) {};
		\node [style=none] (89) at (-4.5, 0) {};
		\node [style=none] (90) at (-12.5, 0) {};
		\node [style=none] (91) at (-10.5, 0) {};
		\node [style=none] (92) at (-3, 0) {};
		\node [style=none] (93) at (-8.75, 0) {};
		\node [style=none] (94) at (-8.75, -1.25) {};
		\node [style=none] (95) at (-7.75, 0) {};
		\node [style=none] (96) at (-3.75, 0) {};
	\end{pgfonlayer}
	\begin{pgfonlayer}{edgelayer}
		\draw [style=thick] (0.center) to (1.center);
		\draw [style=thick] (1.center) to (2.center);
		\draw [style=thick] (2.center) to (0.center);
		\draw [style=thick] (3.center) to (4.center);
		\draw [style=thick] (5.center) to (3.center);
		\draw [style=thick] (5.center) to (4.center);
		\draw [style=thick] (8.center) to (7.center);
		\draw [style=thick] (7.center) to (6.center);
		\draw [style=thick] (6.center) to (8.center);
		\draw [style=thick, bend left=270, looseness=1.50] (20.center) to (21.center);
		\draw [style=thick] (18.center) to (25.center);
		\draw [style=thick] (10.center) to (26.center);
		\draw [style=thick] (11.center) to (27.center);
		\draw [style=thick] (12.center) to (28.center);
		\draw [style=thick] (16.center) to (29.center);
		\draw [style=thick, bend right=75, looseness=1.25] (24.center) to (30.center);
		\draw [style=thick, bend right=90] (23.center) to (31.center);
		\draw [style=thick, bend right=90] (22.center) to (32.center);
		\draw [style=thick] (36.center) to (39.center);
		\draw [style=thick, bend left=90, looseness=1.75] (44.center) to (46.center);
		\draw [style=thick] (33.center) to (47.center);
		\draw [style=thick, bend left=90, looseness=1.75] (35.center) to (36.center);
		\draw [style=thick, bend right=90, looseness=1.75] (46.center) to (35.center);
		\draw [style=thick] (44.center) to (42.center);
		\draw [style=thick] (40.center) to (45.center);
		\draw [style=thick] (55.center) to (63.center);
		\draw [style=thick, bend left=90, looseness=1.75] (65.center) to (57.center);
		\draw [style=thick] (64.center) to (56.center);
		\draw [style=thick, bend left=90, looseness=1.75] (54.center) to (55.center);
		\draw [style=thick, bend right=90, looseness=1.75] (57.center) to (54.center);
		\draw [style=thick] (65.center) to (61.center);
		\draw [style=thick] (66.center) to (53.center);
		\draw [style=thick] (67.center) to (73.center);
		\draw [style=thick] (70.center) to (67.center);
		\draw [style=thick] (70.center) to (73.center);
		\draw [style=thick] (74.center) to (71.center);
		\draw [style=thick] (76.center) to (68.center);
		\draw [style=thick] (75.center) to (72.center);
		\draw [style=thick] (80.center) to (78.center);
		\draw [style=thick] (78.center) to (82.center);
		\draw [style=thick] (82.center) to (80.center);
		\draw [style=thick] (79.center) to (77.center);
		\draw [style=thick, bend right=75, looseness=1.25] (88.center) to (89.center);
		\draw [style=thick, bend left=270, looseness=1.50] (90.center) to (91.center);
		\draw [style=thick, bend right=90, looseness=0.75] (92.center) to (9.center);
		\draw  [style=thick](93.center) to (94.center);
		\draw [style=thick, bend right=75, looseness=1.25] (95.center) to (96.center);
	\end{pgfonlayer}
\end{tikzpicture}}
\endpgfgraphicnamed}
\end{center}

\noindent
The above  normalises to the following:

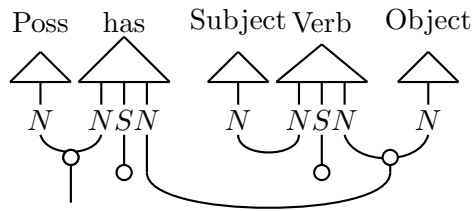
\begin{center}
   {%
\beginpgfgraphicnamed{whose-nowhose-Obj-norm1}
\begin{tikzpicture}[scale=0.8]
	\begin{pgfonlayer}{nodelayer}
		\node [style=none] (0) at (-8.5, 0.25) {};
		\node [style=none] (1) at (-6.5, 0.25) {};
		\node [style=none] (2) at (-7.5, 1.25) {};
		\node [style=none] (3) at (-6.25, 0.25) {};
		\node [style=none] (4) at (-3.25, 0.25) {};
		\node [style=none] (5) at (-4.75, 1.5) {};
		\node [style=none] (6) at (-1.25, 1.25) {};
		\node [style=none] (7) at (-0.25, 0.25) {};
		\node [style=none] (8) at (-2.25, 0.25) {};
		\node [style=none] (9) at (-2.5, -2.75) {};
		\node [style=none] (10) at (-5.5, 0.25) {};
		\node [style=none] (11) at (-4.75, 0.25) {};
		\node [style=none] (12) at (-4, 0.25) {};
		\node [style=none] (13) at (-5.5, -1) {$N$};
		\node [style=none] (14) at (-4.75, -1) {$S$};
		\node [style=none] (15) at (-4, -1) {$N$};
		\node [style=none] (16) at (-1.25, 0.25) {};
		\node [style=none] (17) at (-1.25, -1) {$N$};
		\node [style=none] (18) at (-7.5, 0.25) {};
		\node [style=none] (19) at (-7.5, -1) {$N$};
		\node [style=none] (20) at (-7.5, -1.5) {};
		\node [style=none] (21) at (-5.5, -1.5) {};
		\node [style=none] (22) at (-4, -1.5) {};
		\node [style=none] (23) at (-1.25, -1.5) {};
		\node [style=none] (24) at (-7.5, -0.5) {};
		\node [style=none] (25) at (-5.5, -0.5) {};
		\node [style=none] (26) at (-4.75, -0.5) {};
		\node [style=none] (27) at (-4, -0.5) {};
		\node [style=none] (28) at (-1.25, -0.5) {};
		\node [fill=white, draw, thick, circle, minimum size=0.2 cm, style=none] (29) at (-2.5, -2.25) {};
		\node [style=none] (30) at (-7.5, 2.25) {Subject};
		\node [style=none] (31) at (-4.75, 2.25) {Verb};
		\node [style=none] (32) at (-1.25, 2.25) {Object};
		\node [style=none] (33) at (-7.5, -6) {};
		\node [style=none] (34) at (-11.25, 0.25) {};
		\node [style=none] (35) at (-12, -1.5) {};
		\node [style=none] (36) at (-11.25, 2.25) {has};
		\node [style=none] (37) at (-14, -0.5) {};
		\node [style=none] (38) at (-10.5, -1) {$N$};
		\node [style=none] (39) at (-14, 2.25) {Poss};
		\node [style=none] (40) at (-11.25, -1.5) {};
		\node [style=none] (41) at (-11.25, 1.75) {};
		\node [style=none] (42) at (-12.75, 0.25) {};
		\node [style=none] (43) at (-10.5, -2.75) {};
		\node [fill=white, draw, thick, circle, minimum size=0.2 cm, style=none] (44) at (-13, -2.25) {};
		\node [style=none] (45) at (-13, -2.25) {};
		\node [style=none] (46) at (-12, -1) {$N$};
		\node [style=none] (47) at (-11.25, -1) {$S$};
		\node [style=none] (48) at (-14, -1.5) {};
		\node [style=none] (49) at (-14, 0.25) {};
		\node [fill=white, draw, thick, circle, minimum size=0.2 cm, style=none] (50) at (-11.25, -2.75) {};
		\node [style=none] (51) at (-10.5, -0.5) {};
		\node [style=none] (52) at (-9.75, 0.25) {};
		\node [style=none] (53) at (-12, -0.5) {};
		\node [style=none] (54) at (-14, -1) {$N$};
		\node [style=none] (55) at (-10.5, 0.25) {};
		\node [style=none] (56) at (-13, 0.25) {};
		\node [style=none] (57) at (-12, 0.25) {};
		\node [style=none] (58) at (-11.25, -0.5) {};
		\node [style=none] (59) at (-15, 0.25) {};
		\node [style=none] (60) at (-14, 1.25) {};
		\node [style=none] (61) at (-13, -3.75) {};
		\node [fill=white, draw, thick, circle, minimum size=0.2 cm, style=none] (62) at (-13, -2.25) {};
		\node [style=none] (63) at (-10.5, -1.5) {};
		\node [fill=white, draw, thick, circle, minimum size=0.2 cm, style=none] (64) at (-4.75, -2.75) {};
		\node [style=none] (65) at (-4.75, -1.5) {};
		\node [style=none] (66) at (-2.5, -2.75) {};
	\end{pgfonlayer}
	\begin{pgfonlayer}{edgelayer}
		\draw [style=thick] (0.center) to (1.center);
		\draw [style=thick] (1.center) to (2.center);
		\draw [style=thick] (2.center) to (0.center);
		\draw [style=thick] (3.center) to (4.center);
		\draw [style=thick] (5.center) to (3.center);
		\draw [style=thick] (5.center) to (4.center);
		\draw [style=thick] (8.center) to (7.center);
		\draw [style=thick] (7.center) to (6.center);
		\draw [style=thick] (6.center) to (8.center);
		\draw [style=thick, bend right=90] (20.center) to (21.center);
		\draw [style=thick, bend right=75] (22.center) to (23.center);
		\draw [style=thick] (18.center) to (24.center);
		\draw [style=thick] (10.center) to (25.center);
		\draw [style=thick] (11.center) to (26.center);
		\draw [style=thick] (12.center) to (27.center);
		\draw [style=thick] (16.center) to (28.center);
		\draw [style=thick] (59.center) to (56.center);
		\draw [style=thick] (56.center) to (60.center);
		\draw [style=thick] (60.center) to (59.center);
		\draw [style=thick] (42.center) to (52.center);
		\draw [style=thick] (41.center) to (42.center);
		\draw [style=thick] (41.center) to (52.center);
		\draw [style=thick, bend right=75] (48.center) to (35.center);
		\draw [style=thick] (49.center) to (37.center);
		\draw [style=thick] (57.center) to (53.center);
		\draw [style=thick] (34.center) to (58.center);
		\draw [style=thick] (55.center) to (51.center);
		\draw [style=thick] (50.center) to (40.center);
		\draw  [style=thick] (45.center) to (61.center);
		\draw  [style=thick] (63.center) to (43.center);
		\draw [style=thick] (64.center) to (65.center);
		\draw  [style=thick] (29.center) to (66.center);
		\draw [style=thick, bend right=90, looseness=0.50] (43.center) to (9.center);
	\end{pgfonlayer}
\end{tikzpicture}}
\endpgfgraphicnamed}
\end{center}

\noindent
Using the equivalence result of the previous proposition on the `has' predicate, the above further normalises to the following:

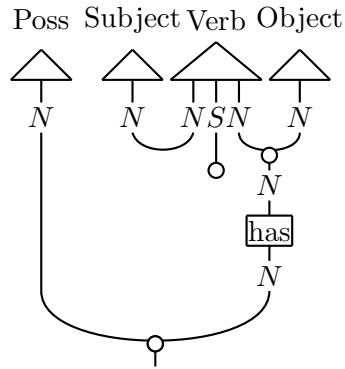
\begin{center}
   {%
\beginpgfgraphicnamed{whose-nowhose-Obj-norm2}
\begin{tikzpicture}[scale=0.8]
	\begin{pgfonlayer}{nodelayer}
		\node [style=none] (0) at (-8.25, 3) {};
		\node [style=none] (1) at (-6.25, 3) {};
		\node [style=none] (2) at (-7.25, 4) {};
		\node [style=none] (3) at (-6, 3) {};
		\node [style=none] (4) at (-3, 3) {};
		\node [style=none] (5) at (-4.5, 4.25) {};
		\node [style=none] (6) at (-1.75, 4) {};
		\node [style=none] (7) at (-0.75, 3) {};
		\node [style=none] (8) at (-2.75, 3) {};
		\node [style=none] (9) at (-5.25, 3) {};
		\node [style=none] (10) at (-4.5, 3) {};
		\node [style=none] (11) at (-3.75, 3) {};
		\node [style=none] (12) at (-5.25, 1.75) {$N$};
		\node [style=none] (13) at (-4.5, 1.75) {$S$};
		\node [style=none] (14) at (-3.75, 1.75) {$N$};
		\node [style=none] (15) at (-1.75, 3) {};
		\node [style=none] (16) at (-1.75, 1.75) {$N$};
		\node [style=none] (17) at (-7.25, 3) {};
		\node [style=none] (18) at (-7.25, 1.75) {$N$};
		\node [style=none] (19) at (-3.75, 1.25) {};
		\node [style=none] (20) at (-1.75, 1.25) {};
		\node [style=none] (21) at (-7.25, 1.25) {};
		\node [style=none] (22) at (-5.25, 1.25) {};
		\node [style=none] (23) at (-7.25, 2.25) {};
		\node [style=none] (24) at (-5.25, 2.25) {};
		\node [style=none] (25) at (-4.5, 2.25) {};
		\node [style=none] (26) at (-3.75, 2.25) {};
		\node [style=none] (27) at (-1.75, 2.25) {};
		\node [fill=white, draw, thick, circle, minimum size=0.2 cm, style=none] (28) at (-2.75, 0.5) {};
		\node [style=none] (29) at (-7.25, 5) {Subject};
		\node [style=none] (30) at (-4.5, 5) {Verb};
		\node [style=none] (31) at (-1.75, 5) {Object};
		\node [style=none] (32) at (-2.75, -4) {};
		\node [style=none] (33) at (-10.25, 2.25) {};
		\node [style=none] (34) at (-10.25, 5) {Poss};
		\node [fill=white, draw, thick, circle, minimum size=0.2 cm, style=none] (35) at (-6.5, -5.75) {};
		\node [style=none] (36) at (-6.5, -5.75) {};
		\node [style=none] (37) at (-2.75, -0.5) {$N$};
		\node [style=none] (38) at (-10.25, -4) {};
		\node [style=none] (39) at (-10.25, 3) {};
		\node [style=none] (40) at (-10.25, 1.75) {$N$};
		\node [style=none] (41) at (-9.25, 3) {};
		\node [style=none] (42) at (-11.25, 3) {};
		\node [style=none] (43) at (-10.25, 4) {};
		\node [style=none] (44) at (-6.5, -6.5) {};
		\node [fill=white, draw, thick, circle, minimum size=0.2 cm, style=none] (45) at (-6.5, -5.75) {};
		\node [fill=white, draw, thick, circle, minimum size=0.2 cm, style=none] (46) at (-4.5, 0) {};
		\node [style=none] (47) at (-4.5, 1.25) {};
		\node [style=none] (48) at (-2.75, 0) {};
		\node [style=none] (49) at (-3.5, -1.5) {};
		\node [style=none] (50) at (-2, -1.5) {};
		\node [style=none] (51) at (-2, -2.5) {};
		\node [style=none] (52) at (-3.5, -2.5) {};
		\node [style=none] (53) at (-2.75, -2.5) {};
		\node [style=none] (54) at (-2.75, -3) {};
		\node [style=none] (55) at (-2.75, -2) {has};
		\node [style=none] (56) at (-2.75, -1.5) {};
		\node [style=none] (57) at (-2.75, -1) {};
		\node [style=none] (58) at (-10.25, 1.25) {};
		\node [style=none] (59) at (-10.25, -4) {};
		\node [style=none] (60) at (-2.75, -3.5) {$N$};
	\end{pgfonlayer}
	\begin{pgfonlayer}{edgelayer}
		\draw [style=thick] (0.center) to (1.center);
		\draw [style=thick] (1.center) to (2.center);
		\draw [style=thick] (2.center) to (0.center);
		\draw [style=thick] (3.center) to (4.center);
		\draw [style=thick] (5.center) to (3.center);
		\draw [style=thick] (5.center) to (4.center);
		\draw [style=thick] (8.center) to (7.center);
		\draw [style=thick] (7.center) to (6.center);
		\draw [style=thick] (6.center) to (8.center);
		\draw [style=thick, bend right=90] (19.center) to (20.center);
		\draw [style=thick, bend right=75] (21.center) to (22.center);
		\draw [style=thick] (17.center) to (23.center);
		\draw [style=thick] (9.center) to (24.center);
		\draw [style=thick] (10.center) to (25.center);
		\draw [style=thick] (11.center) to (26.center);
		\draw [style=thick] (15.center) to (27.center);
		\draw [style=thick] (42.center) to (41.center);
		\draw [style=thick] (41.center) to (43.center);
		\draw [style=thick] (43.center) to (42.center);
		\draw [style=thick] (39.center) to (33.center);
		\draw [style=thick] (36.center) to (44.center);
		\draw [style=thick] (46.center) to (47.center);
		\draw  [style=thick] (28.center) to (48.center);
		\draw  [style=thick] (49.center) to (52.center);
		\draw  [style=thick] (52.center) to (51.center);
		\draw  [style=thick] (51.center) to (50.center);
		\draw  [style=thick] (50.center) to (49.center);
		\draw  [style=thick] (53.center) to (54.center);
		\draw  [style=thick] (57.center) to (56.center);
		\draw  [style=thick] (58.center) to (59.center);
		\draw [style=thick, bend right=90, looseness=0.75] (38.center) to (32.center);
	\end{pgfonlayer}
\end{tikzpicture}}
\endpgfgraphicnamed}
\end{center}

\noindent
This is the same as the meaning of the phrase `Possessor whose Object Subject Verb'. 
\end{proof}

\section{Truth Theoretic Instantiations}
\label{truth-vect}

In this section, we provide a  truth-theoretic instantiation for our model.   This instantiation  is designed
only as a theoretical example and to show that vector spaces can be used to   recast set theoretic Montague-style semantics \cite{Montague}.

Take $N$ to be the vector space spanned by a set of individuals
$\{\ov{n}_i\}_i$  that are mutually orthogonal. For example, $\ov{n}_1$ represents the individual
Mary, $\ov{n}_{25}$ represents Roger the dog, $\ov{n}_{10}$ represents
John, and so on.  A sum of basis vectors in this space represents a
common noun; e.g. $\ov{man} = \sum_i \ov{n}_i$, where $i$ ranges over
the basis vectors denoting men. We take
$S$ to be the one dimensional space spanned by the single vector $\ov{1}$.   We set the 
unit vector spanning $S$ to represent the truth value 1 and  the zero vector
to represent the truth value 0. 

Since the sentence space is the real line, we also have  access to an infinite set of numbers (now represented by a vector, e.g. number 2 by vector $\ov{2}$, number 3 by vector $\ov{3}$ and so on). We interpret  these numbers  as \emph{truth weights}. In a very loose sense, the real line can be seen as a fuzzy set   in the style of fuzzy logic \cite{Zadeh}.  The intuitive reading of the truth of the predicates over the real  interval is that they apply to their arguments not in an absolute sense, but with a certain weight or degree. In other words, whereas in the world of sets,   a relation (representing a predicate) is either true or false about an element of a set, in our world, we have  weighted relations and a predicate may hold about an element with a certain weight or degree.

A transitive verb $w$, which is a vector in the space $N \otimes S
\otimes N$, is represented as follows:
\[
  \overline{w} := \sum_{ij} \ov{n}_i \otimes (\alpha_{ij} \ov{1})
  \otimes \ov{n}_j \quad \text{if} \  \ov{n}_i \ w\mbox{'s} \ \ov{n}_j \ \mbox{with degree}
\ \alpha_{ij}, \ \mbox{for all} \ i,j
  \]
This means that a verb may apply to the subject and object not in an absolute way, that is to say, not only that either they are related via the relation represented by the verb or not, but  that they are related with a certain degree. For instance,  the subject John and the object Mary might not have an absolute love relationship with each other, but might love each other to a certain extent, that is to say, with a certain degree. 

Further, since $S$ is one-dimensional with its only basis vector being
$\ov{1}$, the transitive verb can be represented by the following element of $N
\otimes N$:
\[
\sum_{ij}  \alpha_{ij}  \ov{n}_i \otimes\ov{n}_j \quad \text{if} \quad
\ov{n}_i \ w\mbox{'s} \ \ov{n}_j \ \mbox{with degree}
\ \alpha_{ij}
\]
\noindent
Restricting to either   $\alpha_{ij} = 1$ or $\alpha_{ij} = 0$ provides a 0/1 meaning, i.e.  either
$\ov{n}_i$ \ w's \ $\ov{n}_j$ or not. Letting $\alpha_{ij}$ range over  the interval 
$[0,1]$ enables us to represent degrees as well as limiting cases of
truth and falsity.  For example, the verb ``love'', denoted by $\overline{love}$, is represented by $
 \sum_{ij}\alpha_{ij}  \ov{n}_i \otimes 
\ov{n}_j \quad \text{if} \quad \ov{n}_i \ \text{loves} \ \ov{n}_j \mbox{with degree} \ 
\alpha_{ij}$. 
If we take $\alpha_{ij}$ to be $1$ or $0$, from the above we obtain $
\sum_{(i,j) \in  R_{love}} \ov{n}_i
\otimes \ov{n}_j$, 
  where  $R_{love}$  is the set of all pairs $(i,j)$ such that $\ov{n}_i$ loves $\ov{n}_j$.  Note that, with this definition, the sentence space has already been
discarded, and so for this instantiation the $\iota$ map, which is the
part of the relative pronoun interpretation designed to discard the
relative clause after it has acted on the head noun, is not required.

The ownership morphism which is the map $\overline{\text{'s}} \colon N \to N$ is represented as follows:
\[
\overline{\text{'s}}(\ov{n}_{h'}) =  \sum_{h'' \in O} \ov{n}_{h''}  \quad 
\mbox{whenever $\ov{n}_{h''} $\  is  an owner of \ $\ov{n}_{h'}$} 
\]

For common nouns $\ov{\text{Subject}} = \sum_{k\in K} \ov{n}_k$, 
$\ov{\text{Object}} = \sum_{l\in L} \ov{n}_l$, and $\ov{\text{Possessor}} =  \sum_{h \in P} \ov{n}_h$,  where $k, l$, and $h$ range
over the sets of basis vectors representing the respective common
nouns, the truth-theoretic meaning of a noun phrase modified by a
possessive subject relative clause is computed as follows:

{\small
\begin{align*}
&\ov{\mbox{Possessor whose Subject Verb Object}} :=\\
& \mu_N  \circ \left(1_N \otimes \overline{\text{'s}} \right)\circ \left(1_N \otimes \mu_N \otimes \epsilon_N\right) \left(\ov{\text{Poss}} \otimes \ov{\text{Sub}} \otimes \ov{\text{Verb}}  \otimes  \ov{\text{Obj}}\right)=\\
&\mu_N  \circ \left(1_N \otimes \overline{\text{'s}} \right)\circ \left(1_N \otimes \mu_N \otimes \epsilon_N\right) \left(\sum_{h \in P} \ov{n}_h  \otimes \sum_{k\in K} \ov{n}_k  \otimes   (\sum_{ij} \alpha_{ij} \ov{n}_i \otimes\! \ov{n}_j)  \otimes  \sum_{l\in L} \ov{n}_l \right)=
\\
&\mu_N  \circ \left(1_N \otimes \overline{\text{'s}} \right)
\left(
\sum_{h \in P} \ov{n}_h   \otimes  \sum_{ij,k\in K, l \in L}  \alpha_{ij} \mu_N(\ov{n}_k \otimes\ov{n}_i)
\otimes   \epsilon_N(\ov{n}_j \otimes \ov{n}_l)
\right)=
\\
&\mu_N  \circ \left(1_N \otimes \overline{\text{'s}} \right)
\left(
\sum_{h \in P} \ov{n}_h   \otimes \sum_{ij,k\in K, l \in L} \alpha_{ij}  \delta_{ki}\ov{n}_i 
 \delta_{jl} \right)=\\
 & \mu_N  \circ \left(1_N \otimes \overline{\text{'s}} \right)
\left(
\sum_{h \in P} \ov{n}_h   \otimes \sum_{k\in K, l \in L}  \alpha_{kl} \ov{n}_k\right) =
\\
&\mu_N \left(\sum_{h \in P} \ov{n}_h   \otimes \sum_{k\in K, l \in L}  \alpha_{kl} \ \overline{\text{'s}}(\ov{n}_k)\right)=  \mu_N \left(\sum_{h \in P} \ov{n}_h   \otimes \sum_{k'\in O, l \in L}  \alpha_{k'l} \ \ov{n}_{k'}\right)\\
&\\
&=\sum_{h \in P, k' \in O, l \in L}  \alpha_{k'l} \ \mu_N(\ov{n}_h \otimes \ov{n_{k'}})= \sum_{h \in P, k' \in O, l \in L}  \alpha_{k'l}  \delta_{hk'}\ov{n}_{k'}= \sum_{h \in P,l \in L}  \alpha_{hl}  \ov{n}_{h}
\end{align*}
}

 The result is as it should be:  the sum of the possessor individuals who own the subject individuals  weighted by the degree with which the subjects have acted on the object
individuals via the verb.  A similar computation, with the difference
that the $\mu$ and $\epsilon$ maps are swapped and the ownership morphism $\overline{\text{'s}}$ acts on the object, provides the
truth-theoretic semantics of the possessive object relative clause, which will end up being $\sum_{h \in P, k\in K}  \alpha_{hk} \ov{n}_h$. 

As  an example consider
the possessive subject relative clause ``authors whose books entertained John''.  Take $N$ to be the vector space spanned by the set of all people, including authors,  and objects, and books.  The authors form a subspace  on $N$ whose basis vectors are  denoted by $\ov{a}_{h}$, where $h$ ranges over the set of authors denoted by $P$.  The books  form another subspace  whose basis vectors are denoted by $\ov{b}_k$, where $k$ ranges over the set of books denoted by $K$. Hence, ``authors'' and ``books'' are common nouns $\sum_{h \in P} \ov{a}_h$ and $\sum_{k \in K} \ov{b}_k$.  The transitive verb ``entertain'' is defined by $\sum_{(i,j) \in  R_{entertain}} \ov{b}_i \otimes \ov{n}_j$\,.
Now set ``John'' to
be the individual $\ov{n}_1$ and assume that $\ov{n}_8$ is  another individual which is not necessarily an author or a book.  The above common nouns, verb, and ownership relation are instantiated  as follows:
\begin{eqnarray*}
\ov{\text{authors}} &=& \ov{a}_5 + \ov{a}_6 + \ov{a}_7\\
\ov{\text{books}} &=& \ov{b}_2+ \ov{b}_3 + \ov{b}_4\\
\overline{\text{entertain}} &=&
(\ov{b}_2 \otimes \ov{n}_1) + (\ov{b}_3 \otimes \ov{n}_1) + (\ov{b}_4 \otimes \ov{n}_2) + (\ov{n}_5 \otimes \ov{n}_2)\\
\overline{\text{'s}} (\ov{b}_2) &=& \ov{a}_5 + \ov{a}_6\quad
\overline{\text{'s}} (\ov{b}_3) = \ov{n}_2
\quad
\overline{\text{'s}} (\ov{b}_b) = \ov{a}_8
\end{eqnarray*}
The vector corresponding to the meaning
of ``authors whose books entertained John'' is computed as follows:

{\small
\begin{align*} 
&\ov{\mbox{authors whose books entertained John}} := \\
&  \mu_N  \circ \left(1_N \otimes \overline{\text{'s}} \right)\circ \left(1_N \otimes \mu_N \otimes \epsilon_N\right) \left(\sum_{h \in P} \ov{a}_h \otimes \sum_{k \in K} \ov{b}_k \otimes  (\sum_{(i,j) \in  R_{entertain}}
\ov{b}_i \otimes \ov{n}_j) \otimes \ov{n}_1\right)=\\
&  \mu_N  \circ \left(1_N \otimes \overline{\text{'s}} \right)
 \left(\sum_{h \in P} \ov{a}_h \otimes
 \sum_{k \in K, (i,j) \in R_{entertain}} \mu_N(\ov{b}_k \otimes \ov{b}_i) \epsilon_N(\ov{n}_j \otimes \ov{n}_1)
 \right)=\\
 & \mu_N  \circ \left(1_N \otimes \overline{\text{'s}} \right)
 \left(\sum_{h \in P} \ov{a}_h \otimes
\sum_{(i,j) \in R_{entertain}} \delta_{ki} \ov{b}_i \delta_{j1}
 \right)=\\
 &\mu_N  \circ \left(1_N \otimes \overline{\text{'s}} \right)
 \left(\sum_{h \in P} \ov{a}_h \otimes (\ov{b}_2 + \ov{b}_3)\right)=\\
 &\mu_N \left(\sum_{h \in P} \ov{a}_h \otimes (\overline{\text{'s}}(\ov{b}_2) + \overline{\text{'s}}(\ov{b}_3))\right) = \mu_N \left(\sum_{h \in P} \ov{a}_h \otimes(\ov{a}_5 + \ov{a}_6 + \ov{n}_2)\right)=\\
 &\\
 & \sum_{h \in P} \mu_N(\ov{a}_h \otimes (\ov{a}_5 + \ov{a}_6 + \ov{n}_2)) = \ov{a}_5 + \ov{a}_6
\end{align*}
}
As expected,  the result is the sum of the author basis vectors who wrote books that entertained John.

The verb `entertain` can also have degrees of truth, for example instantiated as follows:
\[
1/6(\ov{b}_2 \otimes \ov{n}_1) + 2/6(\ov{b}_3 \otimes \ov{n}_1) + 2/6(\ov{b}_4 \otimes \ov{n}_2) + 1/6(\ov{n}_5 \otimes \ov{n}_2)
\]
In this case the result of the above phrase will be as follows:
\[
1/6(\ov{a}_5 + \ov{a}_6)
\]
Intuitively, this means that we are not just considering the set of authors whose books entertained John, but those elements of this set, that is those authors,  whose books  have entertained John with a certain weight, namely 1/6. For this example, these are authors $a_5$ and $a_6$. 

\section{Embedding Predicate Semantics}
\label{embed}

In this section, we provide a set-theoretic interpretation for relative  clauses according to the constructions discussed in \cite{Sag}.  In the prequel to this paper, we presented the case of subject and object clauses, here we extend that setting to possessive clauses.  We start by fixing   a  universe of elements ${\cal U}$. A proper noun is an individual (i.e. an element) in  this set;  a common noun is  the set of the individuals that have the property expressed by the common noun;  hence common nouns are unary predicates over ${\cal U}$.  An intransitive verb is the set of all individuals that are acted upon by the relationship expressed by the verb, hence these are  unary predicates over  ${\cal U}$. A transitive verb is the set of pairs of individuals that are related by the relationship expressed by the verb; hence   these  are binary predicates over ${\cal U} \times {\cal U}$.  
Here is an  example for each case:
\begin{align*}
\semantics{Mary} &= \{u_1\}\\
\semantics{Author} &= \{x \in {\cal U} \mid "x \  \mbox{is an author}"\}\\
\semantics{Sleep} &= \{x \in {\cal U}  \mid "x  \  \mbox{sleeps}"\}\\
\semantics{Entertain} &= \{(x,y) \in {\cal U} \times {\cal U} \mid "x \  \mbox{entertains} \ y"\}
\end{align*}
The subject and object relative clauses are interpreted as follows:
\begin{eqnarray*}
&&\left\{x \in {\cal U}   \mid x \in \semantics{Subj},  (x, y) \in \semantics{Verb},  \  \mbox{for all} \ y \in\semantics{Obj}  \right\}\\
&&
\left\{y  \in {\cal U} \mid y \in \semantics{Obj},   (x, y) \in \semantics{Verb},   \ \mbox{for all} \ x \in \semantics{Subj} \right\}
\end{eqnarray*}
These clauses pick out individuals which belong to the intepretation of  subject/object and which  are in the relationship expressed by the verb with all elements of the interpretation of object/subject.  Examples are  `men who love Mary' and `men whom cats love', interpreted as follows, for $x \in {\cal U}$:
\begin{eqnarray*}
&&\left\{x \in {\cal U} \mid x \in \semantics{men},   (x, y) \in \semantics{Love},  \ \mbox{for all} \ y \in \semantics{Mary}\right\} \\
&&\left\{y \in {\cal U}  \mid y \in \semantics{men},   (x, y) \in \semantics{Love},  \ \mbox{for all} \ x \in \semantics{cat}\right\} 
\end{eqnarray*}

\noindent
The  possessive  relative clauses are interpreted as follows:
\begin{align*}
&\left\{x \mid  x \in \semantics{Poss},  (x, y) \in \semantics{Has},  \ \mbox{for all} \  y \in \semantics{Subj},(y, z) \in \semantics{Verb},  \ \mbox{for all} \ z \in \semantics{Obj}\right\}\\
&\left\{x \mid x \in \semantics{Poss}, (x, y) \in \semantics{Has},  \ \mbox{for all} \   y \in \semantics{Obj}, 
(z, y) \in \semantics{Verb}, \ \mbox{for all} \ z \in \semantics{Subj}\right\}
\end{align*}
In the above,  $Has$ is the predicate which represents the ownership relation, defined as follows:
\[
\semantics{Has} = \{(x,y) \mid "x \  \mbox{has} \ y"\}
\]
 For example, the clause  `author whose book entertained John' is interpreted as follows, where since $\semantics{John}$ is a singleton we are able to abbreviate the interpretation:
\begin{align*}
\{x \mid  x \in \semantics{Author},  (x, y) \in \semantics{Has}, \ \mbox{for all} \  y \in \semantics{Book}, (y, \semantics{John}) \in \semantics{Entertain}\}&
\end{align*}

We obtain  the vector space forms of the above predicate interpretations  by developing a map from the category of sets and relations to the category of finite dimensional vector spaces  with a fixed orthogonal basis and linear maps over them, that is a map with types as follows:
\[
e\colon Rel \hookrightarrow FVect
\]
This  map  turns a set $T$  into a vector space $V_T$ spanned by the elements of $T$ and an element $t \in T$   into a basis vector $\ov{t}$ of $V_T$. A subset $W$ of $T$ is turned into  the sum vector of its elements, i.e. $\sum_i \ov{w}_i$, where $w_i$'s enumerate over  elements of $W$.  A relation $R \subseteq T \times T'$ is turned into  a linear map  from $T$ to $T' $, or equivalently  as an element of the space   $T \otimes T'$;  we represent this element by the sum of its basis vectors $\sum_{ij} \ov{t}_i \otimes \ov{t'}_j$, here $t_i$'s enumerate  elements of  $T$ and $t'_j$'s enumerate  elements of $T'$. 

The application of a relation  $R \subseteq T \times T'$  to its arguments  is turned   into   the inner product of the vector space forms of the relation and the arguments. In our sum representation, $R(a)$ when $a \in T$ and $R^{-1}(b)$ when $b \in T'$ are turned into  the following: 
\[
\sum_{ij} \langle a \mid \ov{t}_i \rangle \ov{t'}_j 
\hspace{2cm}
\sum_{ij}  \ov{t}_i \langle  \ov{t'}_j \mid b \rangle
\]
To check wether an element $t$ is in the set $T$, we  use the Frobenius  operation $\mu$ on $T$, that is $\mu \colon T \otimes T \to T$,  as follows:
\[
t \in T \quad \mbox{whenever} \quad \mu(\ov{t}, \sum_i \ov{t}_i) = \ov{t}
\]
The intersection of two sets $T, T'$ is computed by generalising the above  via  the Frobenius $\mu$ operation on  the whole universe, that is $\mu \colon {\cal U} \otimes {\cal U} \to {\cal U}$, or on a set  containing both of these sets, e.g. $T \cup T'$, that is  $\mu \colon (T \cup T') \otimes (T \cup T') \to (T \cup T')$.  In either case, for $\sum_i \ov{t}_i$ and $\sum_j \ov{t'}_j$ representations of $T$ and $T'$ we obtain the following for the intersection:
\[
T \cap T'  \ := \ \mu(\sum_i \ov{t}_i, \sum_j \ov{t'}_j)
\]
We are now ready to show that the vector space forms of the predicate interpretations of relative clauses, developed along side the constructions described above,  provide us with the same semantics as the truth-theoretic vector space semantics developed in Section \ref{truth-vect}. 
\begin{proposition}
The map $e$, described above,  provides a 0/1 truth-theoretic interpretation for possessive relative clauses. 
\end{proposition}
\begin{proof}
The first step is to instantiate the proper and common nouns, the intransitive and transitive verbs. For this, we set the universe $\cal U$ to be the  individuals representing  the nouns of the language, then the  map $e$  instantiates as follows:
\begin{itemize}
\item The universe ${\cal U}$ becomes the vector space $V_{\cal U}$, spanned by its elements. So an element $u_i \in {\cal U}$ becomes a basis vector $\ov{u}_i$ of $V_{\cal U}$.  
\item A proper noun $a \in {\cal U}$ becomes a basis vector $\ov{u}_a$ of $V_{\cal U}$. 
\item A common noun $c \subseteq {\cal U}$ becomes a  subspace of $V_{\cal U}$, represented by the sum of  the representations of its elements $\ov{u}_i$, that is
\[
c \hookrightarrow \ov{c} \quad :: \quad \sum_i \ov{u}_i \quad \text{for \ all} \ u_i \in c
\]
\item A transitive verb $w \subseteq {\cal U} \times {\cal U}$ becomes a linear map $\overline{w}$ in $V_{\cal U} \otimes V_{\cal U}$, represented by the sum of its basis vectors  $\ov{u}_i, \ov{u}_j$, which in this case are pairs of its elements $(u_i, u_j)$, that is
\[
w \hookrightarrow \overline{w} \quad :: \quad \sum_{ij} \ov{u}_i \otimes \ov{u}_j \quad \text{for \ all} \ (u_i, u_j) \in w
\]
\end{itemize}
The next step is to use the above definitions and develop the vector space forms of  the predicate interpretations of the relative clauses. To do so, we turn these interpretations into intersections of subsets and check the membership relation  via the $\mu$ map. As for  the  application of the predicates to the interpretations of their arguments, we use their relational images on subsets. That is,  for $R \subseteq {\cal U} \times {\cal U}$ and $T \subseteq {\cal U}$, we work with $R[T]$, defined by $R[T] = \{t' \in {\cal U}\mid \  (t, t') \in R \ \text{for}  \ t \in T\}$. This form of application has an implicit universal quantification: it consists of all the elements of ${\cal U}$ that are related to some element of $T$. Hence, the intersection forms  take care of  the quantification present in the predicate semantics of the relative clauses without having to explicitly use it.  

The predicate interpretation of the possessive subject clause is equivalent to the following intersection form:
\[
\semantics{Poss} \cap \semantics{Has}^{-1} \left[\semantics{Verb}^{-1}[\semantics{Obj}] \cap \semantics{Subj}\right]
\]
The vector space interpretation of the above is as follows:
\[
V_{\semantics{Poss} \cap \semantics{Has}^{-1} \left[\semantics{Verb}^{-1}[\semantics{Obj}] \cap \semantics{Subj}\right]}=
 \mu(V_{\semantics{Poss}}, V_{\semantics{Has}^{-1} \left[\semantics{Verb}^{-1}[\semantics{Obj}] \cap \semantics{Subj}\right]})
\]

\noindent
For the right hand side, we first compute the following:
{\small
\begin{eqnarray*}
V_{\semantics{Verb}^{-1}[\semantics{Obj}] \cap \semantics{Sbj}} &=&
\mu(V_{\semantics{Verb}^{-1}[\semantics{Obj}] }, V_{ \semantics{Sbj}})\\
= \mu(\sum_{ij} \ov{n}_i \langle \ov{n}_j \mid V_{\semantics{Obj}}\rangle, V_{\semantics{Sbj}}) &=&
\mu(\sum_{ij} \ov{n}_i \langle \ov{n}_j \mid \sum_{l \in L} \ov{n}_l \rangle, \sum_{k \in K} \ov{n}_k)\\
&=& \sum_{ij,k \in K, l \in L} \delta_{ki} \ov{n}_i \delta_{jl} 
\end{eqnarray*} }
This term is a sum of $\ov{n}_i$'s  whenever $k=i$ and $j=l$ and  provides us with subjects that have been modified by the verb phrase Verb-Obj. We denote these by $\sum_a \ov{n}_a$. Next, we insert this term in $V_{\semantics{Has}^{-1} \left[\semantics{Verb}^{-1}[\semantics{Obj}] \cap \semantics{Subj}\right]}$  and proceed the computation as follows:

\[
V_{\semantics{Has}^{-1} \left[\semantics{Verb}^{-1}[\semantics{Obj}] \cap \semantics{Subj}\right]} = 
\sum_{h'h''} \ov{n}_{h'} \langle \ov{n}_{h''} \mid \sum_a \ov{n}_a\rangle = 
\sum_{h'h''} \ov{n}_{h'} \delta_{h''a}
\]
This term is a sum of $\ov{n}_{h'}$'s whenever $h''=a$, that is whenever the individuals that are possessed have been modified by the Verb-Obj phrase. It provides us with the owners of such individuals. We denote them by $\sum_b \ov{n}_b$.  We insert this in $ \mu(V_{\semantics{Poss}}, V_{\semantics{Has}^{-1} \left[\semantics{Verb}^{-1}[\semantics{Obj}] \cap \semantics{Subj}\right]})$ and finish the computation as follows:
\[
 \mu(V_{\semantics{Poss}}, V_{\semantics{Has}^{-1} \left[\semantics{Verb}^{-1}[\semantics{Obj}] \cap \semantics{Subj}\right]}) =
 \mu(\sum_{h \in P} \ov{n}_h, \sum_b \ov{n}_b) = 
 \sum_{h \in P} \delta_{hb} \ov{n}_h
\]
This is a sum of $\ov{n}_h$'s whenever $h=v$, it provides us with the possessors that own individuals that have been modified by Verb-Obj. This is the same term  as the one obtained in Section \ref{truth-vect} for the case when all $\alpha$ terms are equal to 1.  It provides us with the 0/1 truth-theoretic semantics of the subject possessive clause. 

The case of the object possessive  clause is computed similarly. Here, we need to compute the vector space interpretation of the following intersection form:
\[
\semantics{Poss} \cap \semantics{Has}^{-1} \left[\semantics{Verb}[\semantics{Subj}] \cap \semantics{Obj}\right]
\]
In this case, the above  provides us with the possessors that own individuals that have been modified by Subj-Verb phrases, equal to the term obtained in Section \ref{truth-vect} for the case when $\alpha$ terms are equal to 1, hence providing us with the 0/1 truth-theoretic semantics of the clause. 

\end{proof}

Finally, it is easy to see that the decomposion of `whose', as done in Section \ref{decomp}, provides us with the same intersection forms. The intersection form of   `Possessor  that has  Subject that Verb Object' is computed as follows. First we compute the interpretation  of `Subject that Verb Object', as demonstrated in previous work\cite{SadrClarkCoecke}, this is as follows:
\[
\semantics{Sbj} \cap \semantics{Verb}^{-1}[\semantics{Obj}]
\]
Next, we compute   the interpretation of  `Possessor that has X' and obtain:
\[
\semantics{Pos} \cap \semantics{Has}^{-1}[\semantics{X}]
\]
Finally, we substitute $\semantics{Sbj} \cap \semantics{Verb}^{-1}[\semantics{Obj}]$ for $X$ in the above and obtain:
\[
\semantics{Pos} \cap \semantics{Has}^{-1}[\semantics{Sbj} \cap \semantics{Verb}^{_1}[\semantics{Obj}]]
\]
This is the same as the intersection form of  the possessive subject form. A similar computation shows that the intersection form of 
`Possessor that has  Object that Subject Verb' is the same as that of possessive object form.

\section{Concrete Instantiations}
\label{conc}

In the model of  Grefenstette and
Sadrzadeh (2011a) \cite{GrefenSadr1}, the meaning of a verb is taken
to be ``the degree to which the verb relates properties of its
subjects to properties of its objects''. This degree is
computed by forming the sum of the tensor products of the subjects and
objects of the verb across a corpus, where $w$ ranges over instances
of the verb:
\[
\overline{\text{verb}} = \sum_w (\ov{\text{sbj}} \otimes \ov{\text{obj}})_w
\]
The above is a matrix in $N\otimes N$. Since the verbs of this model do not have a sentence dimension,   no $\iota$ map will be needed in computing the meanings of the classes containing them.  

It remains to find a concrete matrix interpretation for $\overline{\text{'s}}$.  For instance, to compute the meaning vector of  "woman whose husband died", we need to know the  "owners" (in a manner of speaking) of the husbands who have died, to be able to  modify the vector of the possessor "woman" with it. Hence,  $\overline{\text{'s}}$  should be  the linear map that inputs a noun phrase and returns their owners. One  way to construct this map  for a noun phrase $X$  is to  sum  over the  nouns that have occurred  in   "noun's $X$". That is, for ${X}$ the subject or object of the relative clause,  we have
\[
\overline{\text{'s}}(X) := \sum_i (\ov{\text{noun}}) _i \qquad \mbox{for each  "noun's $X$" in the corpus}
\]

The abstract vectors corresponding to these diagrams are similar to the
truth-theoretic case, with the difference that the vectors are
populated from corpora and the scalar weights for noun vectors are not
necessarily 1 or 0. For possessor, subject,  and object noun context
vectors computed from a corpus as follows:
\[
\ov{\text{Possessor}} = \sum_h C_h \ov{n}_h\qquad
\ov{\text{Subject}} = \sum_k C_k \ov{n}_k \qquad
\ov{\text{Object}} = \sum_l C_l \ov{n}_l
\]
and the verb and $\overline{\text{'s}}$  linear maps as follows:
\[
\overline{\text{Verb}} = \sum_{ij} C_{ij} \ov{n}_i \otimes \ov{n}_j
\qquad
\overline{\text{'s}} (X) = \sum_{h'} C_{h'} \ov{n}_{h'}
\]
The concrete meaning of a noun phrase modified by a possessive subject relative
clause is as follows:
{\small
\begin{align*}
&\mu \circ \left(1_N \otimes \overline{\text{'s}}\right)\left(\sum_h C_h \ov{n}_h \otimes \sum_{kl} C_{kl} C_k C_l \ov{n}_k\right)=\mu\left(\sum_h C_h \ov{n}_h \otimes \overline{\text{'s}}\left(\sum_{kl}C_{kl} C_k C_l \ov{n}_k\right)\right)\\
&=\mu\left(\sum_h C_h \ov{n}_h \otimes  \sum_{h'} C_{h'}\ov{n}_{h'}\right) =  \sum_{hh'} C_h C_{h'} \delta_{hh'} \ov{n}_h
\end{align*}}
The difference with the truth-theoretic case is that  the degrees of the truth of the verb and the coordinates  of the subject, object, and possessor vectors are now obtained from a corpus and result  in  $C_h C_{h'}$. 

To see how the above vector represents the
meaning of the modified noun phrase, recall from Section \ref{sec:frob} that the application of the $\mu$ map to the tensor product of two vectors is their component-wise multiplication, that is the above is equivalent  to the following:
\[
\sum_h C_h \ov{n}_h \odot  \sum_{h'} C_{h'}\ov{n}_{h'} 
\]
Note that the second term of the above is the  subject modified by the verb-object phrase. Hence the above vector modifies the possessor with the subjects that have been themselves modified by the verb-object phrase  via point-wise multiplication. A similar result holds for the possessive object relative clause case.

As an example, suppose that $N$ has two dimensions with basis vectors
$\ov{n}_1$ and $\ov{n}_2$, and consider the noun phrase ``men whose dog 
bites cats''.  Define the vectors of ``dog'', ``men'', and ``cats" as follows:
\[
\ov{\text{dog}} = d_1 \ov{n}_1 + d_2 \ov{n}_2\qquad
\ov{\text{men}} = m_1 \ov{n}_1 + m_2 \ov{n}_2\qquad
\ov{\text{cats}} = c_1 \ov{n}_1 + c_2 \ov{n}_2
\]
and the matrix of ``bites'' by:
\[ 
 b_{11} \ov{n}_1 \otimes \ov{n}_2 + b_{12} \ov{n}_1 \otimes \ov{n}_2 + b_{21} \ov{n}_2 \otimes \ov{n}_1 + b_{22} \ov{n}_2 \otimes \ov{n}_2
\]
 Then the meaning of the clause becomes:
\[ 
\ov{\text{men}} \odot \overline{\text{'s}} \left(\ov{dogs} \quad  \odot \quad \left(\begin{array}{cc}
b_{11} & b_{12}\\
b_{21} & b_{22}
\end{array}\right)  \ \times \ \left(\begin{array}{c}c_1\\c_2\end{array}\right)\right)
\]
Assuming that the basis vectors of the noun space represent properties of nouns, the meaning of ``men whose dog bites cats'' is a vector representing the properties of men, which
have been modified (via multiplication) by the properties of owners of 
dogs which bite cats. Put another way, those properties of men
which overlap with properties of owner's of dogs that bite cats  get accentuated. Similarly, for the clause ``men whose dogs cats bite" we obtain the following linear algebraic form:
\[ 
\ov{\text{men}} \odot \overline{\text{'s}} \left(\ov{dogs} \quad \odot \quad \left(\begin{array}{cc}c_1&c_2\end{array}\right) \ \times \  \left(\begin{array}{cc}
b_{11} & b_{12}\\
b_{21} & b_{22}
\end{array}\right)\right)
\]
Here we have to transpose the column vector of "cats" into  a row vector (in other words  transposing it) to be able to matrix-multiply it with the matrix of ``bites". The resulting vector contains properties of men which overlap with properties of dogs that cats bite. 


\section{A Toy  Experiment}
\label{exp}

For demonstration purposes and  in order to get  an idea on how the data from a large scale corpus responds to the abstract computational methods developed here, we implement some of our  constructions on a corpus and  do a small-scale experiment with a toy natural language processing task. The corpus we use is the British National Corpus, from which we create  a vector space whose basis vectors are its   10,000 most occurring lemmas.  Our context is a window of 5 words. For each noun we built a vector, whose coordinates are computed  using  the  ratio of the probability of the  occurrence of the word in the context of  the basis  word to the probability of the  occurrence of the word overall. These parameters are chosen based on the success of our  previous experiments \cite{GrefenSadr1,GrefenSadr2,Kartetal1,Kartetal2,Kartetal3}. 

Our task consists of   building vectors for  nouns, verbs, and relative pronouns.  For  nouns, we used their context vectors built as described above. For  verbs, we used the model of \cite{GrefenSadr1}, based on the context vectors of their subjects and objects built as above.  For relative pronouns, we apply two methods, described further on below.

The task we experiment with is an imaginary term/description classification task. It is based on the fact that relative clauses are often used to describe or provide extra information  about words. For instance,     the word  `football' may by described by the clause `game that boys like' and the word  `doll'  by the clause  `toy that girls prefer'.  For our  experiment,  we chose a set of words and manually described each of them by an appropriate relative clause. This data set is presented in the following table:

\begin{center}
\begin{tabular}{c|c|c}
&Term  & Description\\
\hline
	1& emperor  & person who rule empire\\
Ê Ê Ê Ê 2& queen  & woman who reign country\\
Ê Ê Ê Ê 3& mammal  & animal which give birth\\
Ê Ê Ê Ê 4& plug  & plastic which stop water\\
Ê Ê Ê Ê 5& carnivore  & animal who eat meat\\
Ê Ê Ê Ê 6& vegetarian  & person who prefer vegetable\\
Ê Ê Ê Ê 7& doll  & toy that girl prefer\\
Ê Ê Ê Ê 8& football  & game that boy like\\
Ê Ê Ê Ê 9& skirt  & garment that woman wear\\
Ê Ê Ê Ê 10& widow  & woman whose husband die\\
Ê Ê Ê Ê 11& orphan & child whose parent die\\
Ê Ê Ê Ê 12& teacher  & person whose job educate children\\
Ê Ê Ê Ê 13& comedian  & artist whose joke entertain people\\
Ê Ê Ê Ê 14& priest  & clergy whose sermon people follow\\
Ê Ê Ê Ê 15& commander  & military whose order marine obey\\
Ê Ê Ê Ê 16& clown  & entertainer whose trick people enjoy
\end{tabular}
\end{center}

A preliminary goal would be to   check how close the vector of the description is to the vector of its term.  
For instance, we obtained that the cosine of the angle between the vector of   `football' and the vector of its description was 0.61, and the cosine between the vector of `doll' and its description was 0.50. As for the possessive cases, when we took $\overline{\text{'s}}$ to be the identity the cosine between `priest' and its description was  0.51 and the cosine between `commander' and its description was 0.40. These cosines decreased  to 0.51 and 0.21 for the case when $\overline{\text{'s}}$ was computed by summing the possessors of the nouns `sermon` and `order'; this may be  due to the rarity of the occurrence of the corresponding  phrases, i.e. `noun's sermon' and `noun's order' in the corpus.   

These cosines on their own may seem low and customary in the distributional models  is not to consider them in isolation, but in relation to all the other cosines obtained in the experiment. So we follow previous work \cite{Kartetal1,Kartetal2}  and set the goal of our task to be classification; that is we compute for what percentage of the words,   their  vectors are closest to the vectors of their descriptions  (a measure referred to by Mean Reciprocal Rank or MRR) and for what percentage of the descriptions, their vector are the closest to the vectors of their  terms (a measure referred to by Accuracy).  

For the MRR,  we compared the vector of each term to the vectors of  all the descriptions. From our 16 terms, the cosines of 12 of them were the closest to their own description; this was  for the possessive model where $\overline{\text{'s}}$ was taken to be identity. The terms whose descriptions were not the closest to them were `plug, carnivore, vegetarian', and `clown'. For the second possessive model, this number decreased to 9, with the inclusion of `orphan, comedian', and `teacher' in the set of bad terms. However, when the datasets for the possessive and non-possessive clauses were considered separately, that is, terms with possessive-clause descriptions were only compared with each other and terms with non-possessive descriptions were also only compared   with each other, this number increased to 6 out of 9 for the non-possessive cases and 6 out of 7 for the possessive cases. In the latter category, the only word which was not the closest to its own description was `clown'. In the former category, the words whose descriptions were not the closest to their own were `plug, carnivore', and `vegetarian'. In all of these cases, however, the correct description was the second closest to the word. Below are  some of our term/description cosines (in the full dataset) for four of the good terms and two of the bad terms; this was  in the first possessive model (the numbers for the second possessive model are  slightly lower). 

\medskip
\begin{center}\begin{tabular}{|ll||c|}
\hline
Term & \quad Description & Cosine \\
\hline
football&game that boys like &0.62\\
   &woman who reigns country& 0.24\\
   &toy that girls prefer& 0.18\\
   &woman whose husband died& 0.18\\
\hline
priest&clergy whose sermon people follow& 0.53\\
   &woman who reigns country& 0.45\\
   &woman whose husband died& 0.37\\
   &person who rules empire& 0.35\\ 
\hline
plug&toy that girls prefer& 0.24\\
   &plastic which stops water& 0.22\\
   &woman who reigns country& 0.17\\
   &game that boys like& 0.17\\
\hline
clown&woman who reigns country& 0.28\\
   &toy that girls prefer& 0.24\\
   &woman whose husband died& 0.24\\
   &game that boys like& 0.24\\
\hline
   \end{tabular}
   \end{center}

\bigskip
To have a comparison baseline, we also experimented with a multiplicative and an additive model. In these models,  we built   vectors   for the relative clauses by multiplying/summing the vectors of the words in them  in two ways:  (1) treating  the relative pronoun as noise and   not considering it in the computation, (2) treating the relative pronoun as any other word and considering its  context vector  in the computation. As expected, the results turned out not to  be  significantly different (as the vectors of  relative pronouns were dense  and computing  by them was similar to  computing with the  vector consisting of all 1's).   For instance, the cosine for `football' was 0.39 when the pronoun was not considered and 0.37 when it was considered; for `clown', these numbers were 0.11 and 0.10.   These  models did bad on  some of  the good results of the Frobenius model, for example in the multiplicative model, the closest description to the term `mammal' was `animal that eats meat'. At the same time,  they  performed better on some of the bad terms of the Frobenius model, for example the description of `plug' in the multiplicative model became  the closest  to it.  The  Mean Reciprocal Rank  and Accuracy of the cosines between all the clauses (possessive and non-possessive) are presented in the following table. Overall, the Frobenius model with $\overline{\text{'s}} = Id$ did the best.  All the models did better than the baseline, which was the vector of the head noun of the description, e.g. `woman' for `queen', `animal' for `mammal', `plastic' for `plug' and so on.

\bigskip
\hspace{-0.8cm}
\begin{tabular}{|c||c|c|c|c|c|c|c|}
\hline
&{\bf Base}& {\bf Frob} & {\bf Frob} & {\bf Mult} & {\bf Mult} & {\bf Add}\\
&& $\overline{\text{'s}} = Id$&   $\overline{\text{'s}} = \sum_i (\ov{\text{noun}_1}) _i$ &  wo Rel. Pr. &  w Rel. Pr.  & w/wo Rel. Pr. \\
\hline
\hline
&&&&&&\\
{\bf MRR} & 0.70 & 0.82 & 0.71 & 0.78 & 0.76 & 0.75 \\
&&&&&&\\
\hline
{\bf Acc}& 0.56& 0.75 & 0.56 & 0.62 & 0.62 & 0.62 \\
\hline
\end{tabular}


\bigskip
\noindent
The above numbers are not real representatives of the performance of the models; as the dataset is small and hand-made. Extending this task to a  real one on a large automatically built  dataset and doing more involved statistical analysis on the results requires a different venue; it constitutes work-in-progress.

\section{Conclusion and Future Directions}
In this paper, we reviewed the constructions of compact closed categories and Frobenius algebras and their diagrammatic calculi in an informal way. We also reviewed how compact closed categories can be applied to reason about the grammatical structure of  language and the distributional meanings of the words and sentences of    language. We then used the constructions of a Frobenius algebra over a vector space to extend the existing model to  one that  can also reason about possessive relative clauses.  The  Frobenius algebraic  structure of the possessive relative pronouns show how the information of the head of the clause  flows through the relative pronoun to the rest of the clause and how it interacts with the other words to produce a meaning for the whole clause. We  instantiated these abstract constructions  in a truth-theoretic model and also in a concrete vector space model  built from a real corpus. For the former, we showed how this semantics coincides with the predicate semantics of these clauses and for the latter, we have presented a toy experiment and discussed a possible application to a natural language processing task. In a prequel to this paper, we used similar methods to reason about the subject and object relative pronouns. 

For future work we aim to extend the toy example to a full blown example with a data set built from a corpus and investigate the role of our constructions. In this paper, our sentence space was the real line and we loosely interpreted the intermediate values of the unit interval as degrees or weights of truth. These values are reminiscent of the truth values in fuzzy logics. Developing a formal connections with fuzzy logic constitutes a possible future direction.

\section{Acknowledgements}
We would like to thank Dimitri Kartsaklis  for his invaluable help in implementation and statistical analysis of  the toy experiment. We would also like to thank  Laura Rimell for her very helpful comments on an earlier draft of this paper. Stephen Clark was supported by ERC Starting Grant DisCoTex(30692). Bob Coecke and Stephen Clark are supported by EPSRC Grant EP/I037512/1. 
Mehrnoosh Sadrzadeh is supported by an EPSRC CAF grant EP/J002607/1.

\end{document}